\documentclass{article}

\PassOptionsToPackage{numbers,sort&compress}{natbib}
\usepackage[preprint]{neurips_2026}
\usepackage[utf8]{inputenc} 
\usepackage[T1]{fontenc}    
\usepackage{url}            
\usepackage{nicefrac}       
\usepackage{microtype}      

\usepackage{subcaption}    
\usepackage{tikz}          
\usepackage{multirow}      
\usepackage{multicol}      
\usepackage{enumitem}      
\usepackage{comment}
\usepackage{amsmath}       
\usepackage{amssymb}       
\usepackage{mathtools}     
\usepackage{amsthm}        
\usepackage{xspace}        

\usepackage[linesnumbered,ruled,vlined]{algorithm2e} 
\usepackage[textsize=tiny]{todonotes} 

\usepackage{booktabs}
\usepackage{hyperref}       
\numberwithin{equation}{section} 

\theoremstyle{plain}
\newtheorem{theorem}{Theorem}
\newtheorem{lemma}[theorem]{Lemma}

\theoremstyle{definition}
\newtheorem{definition}{Definition}

\newtheorem{assumption}{Assumption}

\theoremstyle{remark}
\newtheorem{remark}{Remark}

\newcommand{\ibb}{\mathbb{I}}
\newcommand{\rbb}{\mathbb{R}}
\newcommand{\pbb}{\mathbb{P}}
\newcommand{\ebb}{\mathbb{E}}

\newcommand{\bz}{{z}}
\newcommand{\bw}{\mathbf{w}}

\newcommand{\bu}{\mathbf{u}}

\newcommand{\bv}{\mathbf{v}}

\newcommand{\wcal}{\mathcal{W}}
\newcommand{\vcal}{\mathcal{V}}

\newcommand{\zcal}{\mathcal{Z}}

\newcommand{\clip}{\mathrm{clip}}

\title{Generalization Analysis of Biased Stochastic Gradient Methods for Minimax Problems\thanks{The work of Yunwen Lei was supported by the Hong Kong Research Grants Council under the GRF projects 17305425 and 17302624.  The work of Yiming Ying was supported by Australian Research Council Discovery Project (DP250101359).}}

\author{%
  Shuang Zeng$^{1}$\quad
  Yunwen Lei$^{1}$\quad Yiming Ying$^{2}$\\[1.2pt]
  $^1$Department of Mathematics, The University of Hong Kong, Hong Kong, China\\[1.2pt]
  $^2$ School of Mathematics and Statistics, The University of Sydney, Sydney, Australia\\[1.2pt]
  \texttt{zengsh9@connect.hku.hk}\quad \texttt{leiyw@hku.hk}\quad \texttt{yiming.ying@sydney.edu.au}
}

\begin{document}

\maketitle


\begin{abstract}
Minimax optimization refers to a class of optimization problems with some variables to minimize and other variables to maximize. Biased stochastic gradient methods have shown practical success in solving minimax problems to either improve robustness, enhance communication efficiency, or decrease computational costs. These successes motivate a lot of theoretical works to study the convergence of biased stochastic gradient methods, while their generalization analysis remains untouched. In this paper, we present the first framework to study the stability and generalization of biased stochastic gradient methods for minimax problems. We establish a connection between stability and generalization for minimax problems by relaxing the existing bounded gradient assumption to a bounded second moment condition. We then introduce a generalized Lipschitz-type condition on bias and gradient estimators, and derive a general stability bound to clarify the connection among bias, gradient estimators, and stability. We apply our general analysis to Zeroth-order and Clipped stochastic gradient descent ascent (SGDA), and derive stability bounds that match those of SGDA under appropriate smoothing/clipping parameters. We combine stability and convergence analyses together, and derive optimal excess risk bounds of order $1/\sqrt{n}$, where $n$ is the sample size.
\end{abstract}

\section{Introduction}

Many machine learning problems can be formulated as minimax optimization problems, where the decision variables are categorized into a group of variables to minimize and another group of variables to maximize. Typical problems include adversarial training with minimization over model parameters and maximization over perturbations~\citep{madry2017towards}, and GANs with minimization over generators and maximization over discriminator's ability to distinguish fake from real data~\citep{goodfellow2014generative}. Beyond adversarial scenarios, minimax problems also appear in robust optimization~\citep{chen2017robust}, and AUC maximization~\citep{ying2016stochastic}. 


A popular class of optimization methods to solve minimax problems is the gradient-based method. A representative method is the gradient descent ascent (GDA), where we conduct a gradient descent over primal variables and a gradient ascent over dual variables. To further relieve the computational cost of the full gradient, stochastic gradient methods have been proposed by approximating the true gradient with gradient estimates based on random sampling~\citep{palaniappan2016stochastic}, leading to various efficient stochastic methods such as SGDA~\citep{nemirovski2009robust}, Variance-Reduced SGDA~\citep{palaniappan2016stochastic}, Optimistic Gradient Descent Ascent (OGDA)~\citep{daskalakis2018limit}, and Stochastic Recursive gradiEnt Descent Ascent (SREDA)~\citep{luo2020stochastic}. Oftentimes, these methods build unbiased gradient estimators whose conditional expectations recover the true gradient. This unbiasedness simplifies both the convergence analysis~\citep{nemirovski2009robust,lin2020gradient,luo2020stochastic,yang2020global,liu2021first} and generalization analysis~\citep{farnia2021train,zhang2021generalization,lei2021stability,zhu2024towards,zhang2024generalization}. However, in practical applications, people often build biased gradient estimators to either improve robustness, enhance communication efficiency, or decrease computational costs. These biased gradient estimators lead to the development of biased stochastic gradient methods for minimax problems. Below we present two representative examples:\\*[1pt]
\textbf{Zeroth-order SGDA} is a computationally efficient method for applications where explicit gradients are computationally infeasible, expensive, or impossible to compute. This is the case when generating adversarial examples with black-box neural networks to test the adversarial robustness~\citep{chen2017zoo}, and when solving automated machine learning tasks where the gradients w.r.t pipeline configuration parameters is infeasible to compute~\citep{wang2021flaml}.\\*[1pt]
\textbf{Clipped-SGDA} is effective for deep learning tasks in recent years since training deep learning models often encounters problems with heavy-tailed noises. Due to the adaptivity in preventing the algorithm from using arbitrarily large gradients~\cite{tucat2025regularized}, it has found very successful applications in natural language processing tasks, like GPT models \cite{brown2020language}. Additionally, gradient clipping helps preserve privacy in differentially private machine learning~\citep{abadi2016deep}.  

Due to the broad applications of biased stochastic gradient methods, efforts devoted to their theoretical analysis have gradually increased. Early work focused on stochastic gradient methods for minimization problems, e.g., the convergence analysis from the optimization perspective~\citep{driggs2022biased,hu2020biased} and the generalization analysis from the learning perspective~\citep{nikolakakis2022black}.
 Recently, some studies have extended these discussions on minimization problems to minimax problems, e.g., the convergence of Zeroth-order SGDA~\citep{beznosikov2020gradient, xu2023zeroth} and Clipped-SGDA~\citep{gorbunov2022clipped}. However, these discussions focus on the behavior of these biased methods on training examples, ignoring an important problem in machine learning--generalization. This gap motivates our core question:

{

\textit{Can we develop a general framework to study the generalization of biased stochastic gradient methods for minimax problems?}

}

In this paper, we answer the above question by developing a general generalization analysis of biased stochastic gradient methods for convex-concave and smooth minimax problems by leveraging the algorithmic stability concept. Our contributions are summarized as follows.

\noindent$\bullet$ We develop a quantitative connection between stability and generalization for minimax problems. We also introduce generalized Lipschitz-type conditions on the bias and gradient estimators, based on which we develop rather general stability bounds of biased stochastic gradient methods for minimax problems, showing how the bias, gradient estimators and step size affect the generalization performance.

\noindent$\bullet$ We show the generalized Lipschitz-type conditions hold for both Zeroth-order SGDA and Clipped-SGDA. Then, we derive the first stability and generalization bounds for Zeroth-order SGDA and Clipped-SGDA, both of which match bounds of SGDA under mild conditions on the smoothing/clipping parameter.

\noindent$\bullet$ Our stability analysis relaxes the bounded gradient assumption~\citep{lei2021stability} to a bounded second moment condition. Furthermore, for both Zeroth-order SGDA and Clipped-SGDA, we combine the stability and convergence analysis to give optimal excess risk bounds of order $O(1 / \sqrt{n})$, where $n$ is the sample size.



\section{Related Work}\label{sec:related_work}

\textbf{Stability and Generalization}. 
Algorithmic stability measures how robust an algorithm is to small perturbations in training data. \citet{bousquet2002stability}  introduced the concept of uniform stability, which is widely used to develop high-probability generalization bounds~\citep{feldman2019high,bousquet2020sharper}.
Additionally, uniform stability was weakened to study generalization bounds in expectation, including the on-average stability~\citep{shalev2010learnability}. Another pivotal advance came from~\citet{hardt2016train}, which pioneered the stability analysis of SGD. This work motivates many stability analyses of stochastic optimization algorithms~\citep{kuzborskij2018data,lei2020fine,bassily2020stability,zhang2024generalization,lei2026learning}. Some work studied the stability and generalization of biased stochastic gradient methods such as zeroth-order SGD and clipped-SGD~\citep{zeng2026stochastic,nikolakakis2022black,liu2023general}. 

For minimax problems,  \citet{zhang2021generalization} used argument stability to study generalization of empirical saddle points for strongly-convex-strongly-concave problems. \citet{farnia2021train} used uniform stability to study generalization of several minimax optimization algorithms such as GDA, GDmax, and proximal point method. \citet{lei2021stability} gave high-probability bounds of SGDA for convex-concave problems. \citet{zhang2024generalization} introduced uniform primal (argument) stability to study the generalization behavior of SGDA for nonconvex-concave problems.



\textbf{Biased Stochastic Gradient Methods for Minimax Problems}.
Now we turn to analyses of biased stochastic gradient methods for minimax problems.
First, biased stochastic gradient methods for minimization have been well studied. The works~\citep{ajalloeian2020convergence,driggs2022biased,hu2020biased,demidovich2023guide,surendran2024non} developed a general framework to study biased stochastic gradient methods for minimization problems, which covers many popular methods such as  Zeroth-order SGD, distributed optimization with biased compression operators, and SGD with delayed gradients.
Second, (unbiased) stochastic gradient methods for minimax problems have also been well studied. For example, the convergence analysis of SGDA and its variants has been widely studied for convex-concave~\citep{nemirovski2009robust},
strongly-convex-strongly-concave~\citep{palaniappan2016stochastic}, nonconvex-strongly-concave~\citep{lin2020gradient,luo2020stochastic},
nonconvex-concave~\citep{yang2020global} and nonconvex-nonconcave problems~\citep{liu2021first,yang2020global}.
These convergence analyses were extended to minimax gradient methods with biased gradient estimates, including Zeroth-order SGDA~\citep{beznosikov2020gradient,xu2023zeroth} and Clipped-SGDA~\citep{gorbunov2022clipped}. To our knowledge, there are no existing generalization analyses of these biased gradient methods for minimax problems.


\section{Problem Formulation\label{sec:problem}}

Let $\pbb$ be a probability measure defined on a sample space $\zcal$. For minimax learning, we aim to  learn two parameters $\bw\in\wcal$ and $\bv\in\vcal$, where $\wcal$ and $\vcal$ are parameter spaces in $\rbb^d$. The performance of a pair $(\bw,\bv)$ on an example $\bz$ can be quantified by $f(\bw,\bv;\bz)$, where
$f:\wcal\times\vcal\times\zcal\to\rbb$ is a loss function. We aim to find the best model $(\bw^*,\bv^*)$ by solving the minimax optimization problem:
\begin{equation}\label{eq:minimax-problem}
\min_{\bw\in \wcal}\max_{\bv\in\vcal} F(\bw,\bv) := \ebb_{z\sim\pbb} [f(\bw,\bv;z)],
\end{equation}
where $\ebb_{\bz}[\cdot]$ means the expectation with respect to (w.r.t.) $\bz$.
In practice, we do not know $\pbb$ but instead have access to a dataset $S=\{z_1,\ldots,z_n\}$ independently drawn from $\pbb$.  Then, we approximate $F$ by an empirical risk
$
F_S(\bw,\bv):=\frac{1}{n}\sum_{i=1}^nf(\bw,\bv;z_i).$
We apply a (randomized) algorithm $A$ to the dataset $S$ and get a model $A(S):=(A_{\bw}(S),A_{\bv}(S))\in\wcal\times \vcal$ as an approximate solution  of the problem \eqref{eq:minimax-problem}.
While an algorithm $A$ often finds a model $A(S)$ with a good performance on training, it may not generalize well to a test dataset due to the limited number of training examples. This leads to an important concept called the generalization gap as the difference between testing and training.

Since minimax problems involve both minimization and maximization, there are several performance measures to quantify the behavior of $A(S)$, including the strong primal-dual risk, weak primal-dual risk, and primal risk~\citep{zhang2021generalization,lei2021stability}. We focus on the weak primal-dual risk, which is widely used in the literature~\citep{zhang2021generalization,farnia2021train,lei2021stability,bassily2023differentially}. Let $\ebb[\cdot]$ denote the expectation w.r.t. both the algorithm $A$ and dataset~$S$. Denote by $B_R=\{\bw\in\rbb^d:\|\bw\|_2\leq R\}$ a ball of radius $R$. 


\begin{definition}[Weak Primal-Dual Risk]
{Let $\bw,\bv$ be two random variables, which can be the output model for applying the algorithm $A$ on the dataset $S$.} The weak Primal-Dual (PD) population risk of a  model $({\bw},{\bv})$ is defined as
\[
\triangle^w({\bw},{\bv})=\sup_{\bv'\in\vcal}\ebb_{\bw}\big[F({\bw},\bv')\big]-\inf_{\bw'\in\wcal}\ebb_{\bv}\big[F(\bw',{\bv})\big],
\]
which is a deterministic quantity
that depends only on the distributions of $\bw$ and $\bv$ (we abuse the notation $\triangle^w({\bw},{\bv})$ to indicate the dependency on the distribution of $\bw,\bv$ rather than the models).
\end{definition}

A standard approach to estimate population risk is to decompose it into a generalization gap and an optimization error~\citep{lei2021stability,zhu2023stability}
$$
\Delta^w(\mathbf{w}, \mathbf{v})=\left(\Delta^w(\mathbf{w}, \mathbf{v})-\Delta_S^w(\mathbf{w}, \mathbf{v})\right)+\Delta_S^w(\mathbf{w}, \mathbf{v}),
$$
where we refer to
\[
\triangle^w_S({\bw},{\bv})=\sup_{\bv'\in\vcal}\ebb\big[F_S({\bw},\bv')\big]-\inf_{\bw'\in\wcal}\ebb\big[F_S(\bw',{\bv})\big]
\]
as the weak PD optimization error of the model $({\bw},{\bv})$, and refer to $\triangle^w({\bw},{\bv})-\triangle^w_S({\bw},{\bv})$ as the weak PD generalization gap. The generalization gap quantifies the difference between training and testing for the output model, while the optimization error measures the training suboptimality of the output model as
compared to the best model. We will use algorithmic stability to study the generalization gap and use optimization theory to study the optimization error.

We now introduce some definitions and assumptions. Denote by $\|\cdot\|_2$ the Euclidean norm, and $\langle\cdot, \cdot\rangle$  the dot product in $\rbb^d$.
\begin{definition}
A function $g: \mathcal{W} \mapsto \mathbb{R}$ is said to be convex if for all $\mathbf{w}, \mathbf{w}^{\prime} \in \mathcal{W}$, $g(\mathbf{w}) \geq g\left(\mathbf{w}^{\prime}\right)+\left\langle\mathbf{w}-\mathbf{w}^{\prime}, \nabla g\left(\mathbf{w}^{\prime}\right)\right\rangle$, where $\nabla$ is the gradient operator. We say $g$ is concave if $-g$ is convex.
\end{definition}

\begin{definition}[Convexity-concavity]
Let $g:\wcal\times \vcal\mapsto\rbb$. We say $g$ is convex-concave if for any $\bv\in\vcal$, the map $\bw\mapsto g(\bw,\bv)$ is convex and for any $\bw\in\wcal$, the map $\bv\mapsto g(\bw,\bv)$ is concave.
\end{definition}

Our analysis relies on an $L$-smoothness assumption and a bounded second moment assumption. A related assumption is the bounded gradient (also called the Lipschitzness) assumption, meaning that $\|\nabla_{\bw}f(\bw,\bv;z)\|_2\leq G$ and $\|\nabla_{\bv}f(\bw,\bv;z)\|_2\leq G$ for all $z\in\zcal,\bw\in\wcal,\bv\in\vcal$. This assumption was widely used in the generalization analysis of minimax optimization problems~\citep{bassily2023differentially,lei2021stability,farnia2021train,zhang2024generalization}. We relax this assumption to a bounded second moment assumption.

\begin{assumption}\label{ass:smooth}
  Let $L>0$. For any $z$, we assume the function $(\bw,\bv)\mapsto f(\bw,\bv;z)$ is $L$-smooth, i.e., the following inequality holds for all $\bw,\bw'\in\wcal,\bv,\bv'\in\vcal$ and $z\in\zcal$
  \[
  \bigg\|\begin{pmatrix}
           \nabla_{\bw}f(\bw,\bv;z)-\nabla_{\bw}f(\bw',\bv';z) \\
           \nabla_{\bv}f(\bw,\bv;z)-\nabla_{\bv}f(\bw',\bv';z)
         \end{pmatrix}\bigg\|_2\!\leq\! L\bigg\|\begin{pmatrix}
                                             \bw\!-\!\bw' \\
                                             \bv\!-\!\bv'
                                           \end{pmatrix}\bigg\|_2.
  \]
\end{assumption}
\begin{assumption}\label{ass:lipschitz}
For any $\bw\in\wcal,\bv\in\vcal$, we impose the following bounded second moment condition
$$
\!\!\ebb_r[\|\nabla_{\bw} f(A_{\bw}(S),\bv;z_r)\|_2^2]\!\leq\! G_{S,A}^2,\quad
\ebb_r[\|\nabla_{\bv} f(\bw,A_{\bv}(S);z_r)\|_2^2]\!\leq\! G_{S,A}^2,
$$
where $r$ follows the uniform distribution over $[n]$, { $G_{S,A}>0$ is a constant depending on the dataset $S$ and algorithm $A$, and $\ebb[G_{S,A}^2]<\infty$.}
\end{assumption}
\section{Stability \& Generalization of Biased~SGDA}\label{sec:stab_gen}
\subsection{Connecting Stability and Generalization\label{sec:stab}}
Algorithmic stability is a fundamental concept for studying the generalization gap. We say two datasets are neighboring datasets if they differ by a single example.
\begin{definition}\label{def:neighbor}
Let $S=\left\{z_1, \ldots, z_n\right\},S'=\left\{z_1', \ldots, z_n'\right\}$ be drawn independently from $\pbb$. For any $i\in[n]$, define $S^{(i)}=\left\{z_1, \ldots, z_{i-1}, z_i', z_{i+1}, \ldots, z_n\right\}$ as the neighboring dataset by replacing the $i$-th element of $S$ with $z_i'$.
\end{definition}
Our first result is a connection between generalization gap and algorithmic stability.
We use the argument stability \cite{lei2021stability} to quantify the sensitivity of an algorithm w.r.t. the perturbation of the dataset.

\begin{definition}[Algorithmic Stability]
Let $S$ and $S^{(i)}$ be defined in Definition~\ref{def:neighbor}.
Let $A$ be a randomized algorithm and $\epsilon>0$. We say
$A$ is $\epsilon$-argument-stable in expectation if
    \[
    \ebb_{S,S',A}\bigg[\frac{1}{n}\sum_{i=1}^n\bigg\|\begin{pmatrix}
    A_{\bw}(S)-A_{\bw}(S^{(i)}) \\
    A_{\bv}(S)-A_{\bv}(S^{(i)})
  \end{pmatrix}\bigg\|_2^2\bigg]\leq\epsilon^2.
    \]
\end{definition}


\begin{theorem}\label{thm:stab-gen}
Let $A$ be a randomized algorithm and $\epsilon>0$.
If $A$ is $\epsilon$-argument-stable in expectation and Assumptions \ref{ass:smooth}, \ref{ass:lipschitz} hold, then the weak PD generalization error of $(A_{\bw}(S),A_{\bv}(S))$ satisfies
  \begin{align*}
        \triangle^w(A_{\bw}(S),A_{\bv}(S))\!-\!\triangle^w_S(A_{\bw}(S),A_{\bv}(S))
        \leq L\epsilon^2/2 + {\sqrt{2\ebb[G_{S,A}^2]}}\epsilon.
  \end{align*}
\end{theorem}
\begin{remark}
To our best knowledge, all existing generalization analyses of minimax problems impose a $G$-Lipschitzness assumption~\citep{bassily2023differentially,lei2021stability,farnia2021train,zhang2024generalization}, which is also used in convergence analyses of biased-SGDA.
By using this assumption, the work \citep{lei2021stability} showed that
\begin{equation}\label{lei2021-connection}
    \triangle^w(A_{\bw}(S),A_{\bv}(S))-\triangle^w_S(A_{\bw}(S),A_{\bv}(S))\leq G\epsilon.
\end{equation}
In comparison, we relax this bounded gradient assumption to the bounded second moment assumption (Assumption \ref{ass:lipschitz}). Instead, we also impose a smoothness assumption to achieve this relaxation from the Lipschitzness assumption to the bounded second moment assumption. This additional assumption is not restrictive since stability analysis often imposes a smoothness assumption to get effective stability bounds~\citep{hardt2016train}. We are mostly interested in the case with $\epsilon \lesssim 1$\footnote{We say $A\lesssim B$ if there exists a universal constant $C>0$ such that $A\leq BC$. We denote $A\asymp B$ if $A\lesssim B$ and $B\lesssim A$.}. Then, the bound in Theorem \ref{thm:stab-gen} is dominated by the second term $O(\sqrt{\ebb[G_{S,A}^2]}\epsilon)$, which is of the same order as in Eq.~\eqref{lei2021-connection}.
\end{remark}

\subsection{Biased-SGDA: Convex-Concave Case\label{sec:convex}}
We are interested in a general class of SGDA with biased gradient estimates $g(\bw_t,\bv_t;z_{i_t})$ for convex-concave minimax optimization problems (i.e., $\ebb[g(\bw_t,\bv_t;z_{i_t})|\bw_t,\bv_t,S]$ may not be equal to $\nabla F_S(\bw_t,\bv_t)$), where $g:\wcal\times\vcal\times\rbb^d\to \rbb^{2d}, \{(\bw_t,\bv_t)\}_t$ is a model sequence and $i_t$ is drawn from the uniform distribution over $[n]:=\{1,\ldots,n\}$, for any $t\in[T]$. We consider a generalized bias, where the bias is measured w.r.t. surrogate loss functions $\tilde{f}_1,\tilde{f}_2:\wcal\times\vcal\times\zcal\mapsto\rbb$. The intuition is that while $g(\bw_t,\bv_t;z_{i_t})$ may be a biased estimator of $\nabla F_S(\bw_t,\bv_t)$, it can be an unbiased estimator of $\nabla \widetilde{F}_{S}(\bw_t,\bv_t)=(\nabla_{\bw} \widetilde{F}_{1,S}(\bw_t,\bv_t),\nabla_{\bv} \widetilde{F}_{2,S}(\bw_t,\bv_t))^{\top}$, where we define $\widetilde{F}_{1,S},\widetilde{F}_{2,S}$ as follows
\begin{align}\label{widetilde-F-1}
\widetilde{F}_{j,S}(\bw,\bv):=\frac{1}{n}\sum_{i=1}^{n}\tilde{f}_j(\bw,\bv;\bz_i), \quad j=1,2.
\end{align}
A closely related construction can be found in the work~\citep{zeng2026stochastic}, where they use the surrogate loss to analyze the stability of biased SGD for minimization problems.
We extend this idea to the minimax setting through two surrogate losses $\tilde{f}_1(\bw,\bv;z)$ and $\tilde{f}_2(\bw,\bv;z)$, where $\tilde{f}_1$ and $\tilde{f}_2$ can be different.
As we will show later, this flexibility works well for Zeroth-order SGDA to get 0 bias, where the smoothing parameters can be different in the primal and dual variables.
\begin{definition}[Biased-SGDA]
Let $\bw_1=0\in\wcal$ and $\bv_1=0\in\vcal$ be the initial point. {Let $\text{Proj}_{\wcal}(\cdot)$ and $\text{Proj}_{\vcal}(\cdot)$ denote the projection onto $\wcal$ and $\vcal$, respectively.}
Let $\{\eta_t\}_t$ be a positive stepsize sequence. At each iteration, we randomly draw $i_t$ from the uniform distribution over $[n]:=\{1,2,\ldots,n\}$ and update
\begin{equation}\label{gda}
  \begin{cases}
    \bw_{t+1}=\text{Proj}_{\wcal}(\bw_t-\eta_t g_{\bw} (\bw_t,\bv_t;z_{i_t})), & \\
    \bv_{t+1}=\text{Proj}_{\vcal}(\bv_t+\eta_t g_{\bv} (\bw_t,\bv_t;z_{i_t})), &
  \end{cases}
\end{equation}
where
$
g(\bw_t,\bv_t;z_{i_t}) =
(g_{\bw} (\bw_t,\bv_t;z_{i_t}),g_{\bv} (\bw_t,\bv_t;z_{i_t}))^{\top}
$
is a (biased) estimator of $\nabla \widetilde{F}_{S}(\bw_t,\bv_t)=(\nabla_{\bw} \widetilde{F}_{1,S}(\bw_t,\bv_t),\nabla_{\bv} \widetilde{F}_{2,S}(\bw_t,\bv_t))^{\top}.
$
\end{definition}
Define $z_{i_t}^{(i)}=z_{i_t}$ if $i_t\neq i$ and $z_{i_t}^{(i)}=z_i'$ if $i_t=i$.  We  apply the same algorithm to $S^{(i)}$ and update 
\begin{equation}\label{gda_prime}
  \begin{cases}
    \bw^{(i)}_{t+1}=\text{Proj}_{\wcal}(\bw^{(i)}_t-\eta_t g_{\bw} (\bw^{(i)}_t,\bv^{(i)}_t;z^{(i)}_{i_t})), & \\
    \bv^{(i)}_{t+1}=\text{Proj}_{\vcal}(\bv^{(i)}_t+\eta_t g_{\bv} (\bw^{(i)}_t,\bv^{(i)}_t;z^{(i)}_{i_t})). &
  \end{cases}
\end{equation}
Then, we define the bias $b_t=(b_{t,\bw},b_{t,\bv})$ as follows
\begin{equation}\label{bias}
\!\!\!\!\!\!\!\begin{cases}
b_t=\ebb_{t}[g(\bw_t,\bv_t;z_{i_t})]-\nabla \widetilde{F}_{S}(\bw_t,\bv_t),&\\ b_t^{(i)}\!=\!\ebb_t[g(\bw_t^{(i)},\!\bv_t^{(i)};z^{(i)}_{i_t})]\!-\!\nabla \widetilde{F}_{S^{(i)}}(\bw_t^{(i)},\!\bv_t^{(i)}),\!\!&
\end{cases}
\end{equation}
where
$\ebb_t := \ebb\big[\cdot \big| \bw_t, \bw_t^{(j)}, \bv_t, \bv_t^{(j)}, S, S',j\in[n]\big]$.

Our main result is a general stability bound of Biased-SGDA for convex-concave problems. To this aim, we impose some \emph{generalized Lipschitz-type} assumptions as shown in Eqs.~\eqref{simple-condition-a},~\eqref{simple-condition-c} and ~\eqref{simple-condition-d}.
Eq.~\eqref{simple-condition-a} and Eq.~\eqref{simple-condition-c} control the change rate of the gradient estimate and bias versus the change of the arguments, respectively. Eq.~\eqref{simple-condition-d} controls the change rate of the surrogate function's gradient. As we will show, these generalized Lipschitz-type assumptions hold for various Biased-SGDA methods.

\begin{theorem}\label{lem:stab-gda}
  Assume for all $z$, the map $(\bw,\bv)\mapsto \tilde{f}_j(\bw,\bv;z)$ is convex-concave, $j=1,2$. Let $A$ be Biased-SGDA \eqref{gda} with $t$ iterations.
  Assume there exist $A,\bar{A}\geq0$, $B,\bar{B}\geq0$ and $\widetilde{G} \geq 0$ such that
\begin{align}\label{simple-condition-a}
\!\!\!\!\ebb\Bigg[\frac{1}{n}\sum_{i=1}^n\Bigg\|\!\begin{pmatrix}
  g_{\bw}(\bw_t^{(i)},\bv_t^{(i)};z_{i_t}^{(i)}) \!-\! g_{\bw}(\bw_t,\bv_t;z_{i_t})\\
  g_{\bv}(\bw_t,\bv_t;z_{i_t}) \!-\! g_{\bv}(\bw_t^{(i)},\bv_t^{(i)};z^{(i)}_{i_t})
  \end{pmatrix}\!\Bigg\|_2^2 \Bigg]
\!\leq\! \ebb\Bigg[\frac{A}{n}\sum_{i=1}^n\Bigg\|\begin{pmatrix}
                \bw_t\!-\!\bw_t^{(i)} \\
                \bv_t\!-\!\bv_t^{(i)}
            \end{pmatrix}\Bigg\|_2^2\!+\!B\Bigg],
\end{align}
  \begin{align}\label{simple-condition-c}
   \!\!\!\!\!\!\!\!\ebb\Bigg[\frac{1}{n}\!\sum_{i=1}^n \Bigg\|\!\begin{pmatrix}
                    b_{t,\bw} \!-\! b_{t,\bw}^{(i)}\\
                     b_{t,\bv}^{(i)} \!-\! b_{t,\bv}
                        \end{pmatrix}\!\Bigg\|_2^2\Bigg]
                        \!\leq\! \ebb\Bigg[\frac{\bar{A}}{n}\!\sum_{i=1}^n \!\Bigg\|\!\begin{pmatrix}
                \bw_t\!-\!\bw_t^{(i)} \\
                \bv_t\!-\!\bv_t^{(i)}
            \end{pmatrix}\!\Bigg\|_2^2\!+\!\bar{B}\Bigg],
  \end{align}
    \begin{align}\label{simple-condition-d}
 \!\!\!\!\!\!\!\!\ebb\Bigg[\frac{1}{n}\sum_{i=1}^n\Bigg\|\begin{pmatrix}
    \nabla_{\bw} \tilde{f}_1(\bw_t,\bv_t;z_i) - \nabla_{\bw} \tilde{f}_1(\bw^{(i)}_t,\bv^{(i)}_t;z_i')\\
    \nabla_{\bv} \tilde{f}_2(\bw^{(i)}_t,\bv^{(i)}_t;z_i')- \nabla_{\bv} \tilde{f}_2(\bw_t,\bv_t;z_i)
    \end{pmatrix}\Bigg\|_2^2\Bigg]
    \leq \widetilde{G}^2.
  \end{align}
  If $\eta_t=\eta$ with $A\eta^2+2\bar{A}^{\frac{1}{2}}\eta\leq 1/T$, then $A$ is $\big(2\eta \sqrt{B T}
  + 4\big(\sqrt{\bar{B}}+\frac{\widetilde{G}}{n}\big) \eta T\big)$-argument-stable in expectation.
\end{theorem}
\begin{remark}
A notable property of Theorem~\ref{lem:stab-gda} is that it applies to a rather general biased gradient descent ascent method, with the bias defined w.r.t. a surrogate function. Note that we do not need to compute the surrogate function in the implementation of the algorithm. To study the stability of any specific algorithm, one just needs to verify the generalized Lipschitzness condition, and plug the corresponding $A, B, \bar{A}, \bar{B}, \widetilde{G}$ into Theorem \ref{lem:stab-gda} to immediately get the stability bounds.
The assumption $A\eta^2+2\bar{A} \leq \frac{1}{T}$ holds if we choose a sufficiently small stepsize $\eta \lesssim 1/\sqrt{AT}$ and $\bar{A} \lesssim 1/T$, which, as we will see, is the case for SGDA, Zeroth-order SGDA, and Clipped-SGDA.
It is clear that the stability bound involves both $B$ and $\bar{B}$. The factor in front of $\sqrt{B}$ is $\eta\sqrt{T}$, while the factor in front of $\sqrt{\bar{B}}$ is $\eta T$. Therefore, the bias has larger impacts (as compared to gradient estimates) on the stability. To get meaningful stability bounds, we require small $B$ and $\bar{B}$, which is the case for Zeroth-order SGDA and Clipped-SGDA. In particular, as we will show, we can set $\bar{B}=0$ for zeroth-order SGDA by developing an appropriate surrogate function.
\end{remark}

\section{Applications}\label{sec:application}
This section shows the effectiveness of our general stability bound in Section \ref{sec:convex} by applying it to SGDA, Zeroth-order SGDA, and Clipped-SGDA. Before presenting our stability bounds of these specific algorithms, we consider another bounded second moment assumption, {which is stronger than Assumption~\ref{ass:lipschitz} (see by substituting $\mathbf{w} = A_{\mathbf{w}}(S)$, $\mathbf{v} = A_{\mathbf{v}}(S)$). Assumption~\ref{ass:moment} is still weaker than the Lipschitzness assumption:
for $f(w;z)=zw$ with $z\sim\mathcal{N}(0,1)$, one may take $G_S^2=\frac1n\sum_{i=1}^n z_i^2$ with $\ebb[G_S^2]=1$, while $\|\nabla f\|$ is unbounded.}
\begin{assumption}\label{ass:moment}
Let $r$ follow from the uniform distribution over $[n]$. Suppose for any $\bw\in\wcal,\bv\in\vcal$, {there exists a data-dependent constant $G_S$  such that ($\ebb[G_S^2]<\infty$)}
\begin{align*}
\!\!\!\!\!\!\!\!\max\Big\{\ebb_r[\|\nabla_{\bw} f(\bw,\bv;z_r)\|_2^2],
\ebb_r[\|\nabla_{\bv} f(\bw,\bv;z_r)\|_2^2]
\Big\}\!\leq\! G_S^2.
\end{align*}
\end{assumption}
Define $\bar{\bw}_T = \sum_{t=1}^T (\eta_t \bw_t)/\sum_{t=1}^T \eta_t$, $\bar{\bv}_T = \sum_{t=1}^T (\eta_t \bv_t)/\sum_{t=1}^T \eta_t$.

\begin{table*}[h]\centering\renewcommand{\arraystretch}{1.7}
\caption{\label{tab:risk} Excess risk bounds of Zeroth-order SGDA and Clipped-SGDA for convex-concave and smooth problems. For brevity, we assume $\eta_t=\eta$ and indicate the dependency on $n,L,G,R$. For Zeroth-order SGDA, we take $K\asymp d$ and a sufficiently small $\mu$, i.e., $\mu=(\mu_1^2+\mu_2^2)^{\frac{1}{2}} \lesssim
\frac{\sqrt{\ebb[G_S^2]}}{L\sqrt{n}}$. We take $\tau_1,\tau_2 \asymp L^{\frac{1}{2}} G T^{\frac{1}{2}}$ for Clipped-SGDA under Assumption \ref{ass:moment} with $G_S\leq G$ for some $G>0$.}
\begin{tabular}{|c|c|c|c|c|}
  \hline
  Algorithm & Reference & Step size & Iteration number & Excess risk bound \\ \hline
  \shortstack{Zeroth-order SGDA} & Remark~\ref{remark:zero-risk} & $\eta \asymp 1/(L\sqrt{T})$ & $T \asymp n$ & $\frac{\ebb[G_S^2]/L+LR^2}{\sqrt{n}}$ \\ \hline
  \shortstack{Clipped-SGDA} & Remark~\ref{remark:clip-risk}& $\eta \asymp 1/(L\sqrt{T})$  & $T \asymp n$  & $\frac{G^2/L+ L R^2 + G^2/L^2}{\sqrt{n}}$ \\
  \hline
\end{tabular}
\end{table*}

\subsection{SGDA}
We first consider SGDA with no bias, and our aim is to verify the effectiveness of our general result by recovering the existing results on the classical SGDA in \citet{lei2021stability}.
For SGDA, $\bw_{t+1},\bv_{t+1}$ are produced by Eq.~\eqref{gda} with
\begin{equation}\label{grad-sgda}
  \begin{cases}
    g_{\bw} (\bw_t,\bv_t;z_{i_t}) = \nabla_{\bw} f (\bw_t,\bv_t;z_{i_t}), & \\
    g_{\bv} (\bw_t,\bv_t;z_{i_t}) = \nabla_{\bv} f (\bw_t,\bv_t;z_{i_t}). &
  \end{cases}
\end{equation}
It is clear $g(\bw_t,\bv_t;z_{i_t})$ is an unbiased estimate of $\nabla F_S(\bw_t,\bv_t)$ and then we have $b_t=0$ by setting $\tilde{f}_1=\tilde{f}_2=f$.



\begin{theorem}[Stability of SGDA]\label{cor:sgda_stability}
Assume $\forall z, (\bw,\bv)\mapsto f(\bw,\bv;z)$ is convex-concave . Let Assumption \ref{ass:smooth}, \ref{ass:moment} hold. If $L^2 \eta^2 \leq \frac{1}{T}$, then SGDA is ${\frac{4\sqrt{2\ebb[G_S^2]}}{\sqrt{n}}}\eta\big(\sqrt{T}+2T/\sqrt{n}\big)$-argument-stable in expectation.
\end{theorem}

\begin{remark}
Under a $\bar{G}$-Lipschitzness assumption (i.e., $\|\nabla_{\bw} f(\bw,\bv;\bz)\|_2\leq \bar{G},\|\nabla_{\bv} f(\bw,\bv;\bz)\|_2\leq \bar{G}$ for all $\bz$),
\citet{lei2021stability} shows that SGDA is $\sqrt{8e(1+\frac{\bar{G}}{n})}\bar{G}\exp{(\frac{1}{2}L^2\eta^2 T)}\frac{\eta \sqrt{T}}{\sqrt{n}}$-argument-stable in expectation. In comparison, we relax the Lipschitzness assumption in \citet{lei2021stability} to the bounded second moment assumption (Assumption \ref{ass:moment}). If $n\asymp T$ and $L^2\eta^2\leq 1/T$, both our stability bound and the existing bound~\citep{lei2021stability} are of order $\eta\sqrt{T/n}$ (ignoring the dependency on $G_S,\bar{G}$). Then,  under a milder assumption, our general analysis for biased stochastic gradient methods recovers the existing stability bounds developed explicitly for SGDA.
\end{remark}

\subsection{Zeroth-order SGDA}

We now consider Zeroth-order SGDA~\citep{beznosikov2020gradient,xu2023zeroth}, which is a computationally efficient method for applications where explicit gradient is computationally infeasible, expensive, or impossible to compute. For example, this is the case for generating adversarial examples with black-box deep neural networks to test the adversarial robustness~\citep{chen2017zoo}, and also for solving automated machine learning tasks where the gradients w.r.t pipeline configuration parameters are infeasible to compute~\citep{wang2021flaml}. The basic idea is to approximate the first-order gradients based only on function values.

In particular, for Zeroth-order SGDA, $\bw_{t+1},\bv_{t+1}$ are produced by Eq.~\eqref{gda} with ($\bu_{1,t,l}, \bu_{2,t,l} \sim \operatorname{Unif}(\mathbb{S}^{d-1})$ denotes the uniform distribution over the unit sphere in $\rbb^d$)
\begin{equation}\label{g-zero}
\begin{cases}
g_{\bw}(\bw_t,\bv_t;z_{i_t}) = \frac{1}{K}\sum_{l=1}^K \frac{f(\bw_t+\mu_1\bu_{1,t,l},\bv_t;z_{i_t}) - f(\bw_t-\mu_1\bu_{1,t,l},\bv_t;z_{i_t})}{2\mu_1/d}\bu_{1,t,l},\\
g_{\bv}(\bw_t,\bv_t;z_{i_t}) = \frac{1}{K}\sum_{l=1}^K \frac{f(\bw_t,\bv_t+\mu_2\bu_{2,t,l};z_{i_t}) - f(\bw_t,\bv_t-\mu_2\bu_{2,t,l};z_{i_t})}{2\mu_2/d}\bu_{2,t,l},
\end{cases}
\end{equation}
where
$\mu_1,\mu_2 \in \mathbb{R}^{+}$ are smoothing parameters determining the step size in approximating the gradient of $f$. For simplicity, we omit the dependency on $\bu_{1,t,l}$ and $\bu_{2,t,l}$ in the notation of $g_{\bw}(\bw_t,\bv_t;z_{i_t})$ and $g_{\bv}(\bw_t,\bv_t;z_{i_t})$. We also denote $\mu=(\mu_1^2+\mu_2^2)^{1/2}$ {and assume $d\geq 2$}. The following lemma justifies the effectiveness of the surrogate function. The first statement show that the surrogate functions preserve the smoothness, convexity-concavity, and bounded second moment condition (Assumption~\ref{ass:moment}). While $g_{\bw}(\bw,\bv;\bz)$ and $g_{\bv}(\bw,\bv;\bz)$ are biased estimates of $\nabla_{\bw} F_S(\bw,\bv)$ and $\nabla_{\bv} F_S(\bw,\bv)$, the second statement shows that they are unbiased gradient estimates of the surrogate functions.
Then, the bias defined in Eq.~\eqref{bias} is 0.
This shows the benefit of
surrogate functions in the stability analyses.
Proofs are given in Section~\ref{sec:proof-zero-sgda}.

\begin{lemma}\label{lem:unbiase-zeroth}
Consider the surrogate functions
\begin{equation}\label{tilde-f-zero}
\tilde{f}_1(\bw,\bv;\bz)=\ebb_{\bu}\big[f(\bw+\mu_1\bu,\bv;\bz)\big],\quad
\tilde{f}_2(\bw,\bv;\bz)=\ebb_{\bu}\big[f(\bw,\bv+\mu_2\bu;\bz)\big],
\end{equation}
where $\bu \sim \operatorname{Unif}(\mathbb{B}^{d})$  with $\operatorname{Unif}(\mathbb{B}^{d})$ denoting the uniform distribution over the unit ball in $\rbb^d$, and we omit $\mu_1,\mu_2$ in $\tilde{f}_1,\tilde{f}_2$ for brevity.
\begin{itemize}\setlength\itemsep{1pt}
  \item Assume for any $z$, $(\bw,\bv) \mapsto f(\bw,\bv;\bz)$ is convex-concave, smooth, and satisfies Assumption~\ref{ass:moment}.  Then for any $z$, $(\bw,\bv) \mapsto \tilde{f}_1(\bw,\bv;\bz)$ and $(\bw,\bv) \mapsto \tilde{f}_2(\bw,\bv;\bz)$ are convex-concave, smooth, and satisfy Assumption~\ref{ass:moment}.
  \item
  Let $g$ be defined in Eq.~\eqref{g-zero}.
  Then, for any $\bw\in\wcal,\bv \in \vcal$, $\ebb_{\bu}[g_{\bw}(\bw,\bv;\bz)]=\nabla_{\bw}\tilde{f}_1(\bw,\bv;\bz)$ and $\ebb_{\bu}[g_{\bv}(\bw,\bv;\bz)]=\nabla_{\bv}\tilde{f}_2(\bw,\bv;\bz)$.
\end{itemize}
\end{lemma}


To apply Theorem \ref{lem:stab-gda}, we will also show that Eq.~\eqref{simple-condition-a} and Eq.~\eqref{simple-condition-d} hold with appropriate $A,B$ and $\widetilde{G}$, which are constructed in Lemma \ref{lem:condition-a-zero} and Lemma \ref{lem:tilde-difference}, proved in Section~\ref{sec:proof-zero-sgda}. We then apply Lemma \ref{lem:condition-a-zero} and Lemma \ref{lem:tilde-difference} to Theorem~\ref{lem:stab-gda} to derive the stability bounds for Zeroth-order SGDA.



\begin{theorem}[Stability of Zeroth-order SGDA]\label{cor:zero-stability}
Assume for all $z$, the map $(\bw,\bv)\mapsto f(\bw,\bv;z)$ is convex-concave. Let Assumptions \ref{ass:smooth}, \ref{ass:moment} hold. If $L^2(1+\frac{4d}{K})\eta^2 \leq \frac{1}{T}$, Zeroth-order SGDA is $\epsilon$-argument-stable in expectation with
\begin{align*}
\!\!\!\!\!\epsilon \lesssim \Big(\frac{\sqrt{\ebb[G_S^2]}}{\sqrt{n}}\Big(1+\frac{d^{\frac{1}{2}}}{K^{\frac{1}{2}}}\Big)\Big) \eta \sqrt{T} + \frac{\sqrt{\ebb[G_S^2]}}{n}\eta T.
\end{align*}
\end{theorem}

\begin{remark}[Comparison with SGDA]\label{remark:zero}
If we choose $K \asymp d$, then the conditions in Theorem \ref{cor:zero-stability} hold with the standard choice $\eta \lesssim 1/(L\sqrt{T})$. The stability bound of Zeroth-order SGDA is of the same order as the stability bound for SGDA in Theorem \ref{cor:sgda_stability} in this case, {which is also verified in our experiments (Section~\ref{sec:experiment}).}
\end{remark}

\begin{remark}[Parameter choices]\label{rem:para_zero}
The works~\citep{nikolakakis2022black,liu2023general} study Zeroth-order SGD with a fast-decaying step size $\eta_t \lesssim 1 / t$ and require the smoothing parameter $\mu \lesssim 1/(dn)$. More recently, \cite{zeng2026stochastic} improves their results with a milder condition $\eta \lesssim 1/\sqrt{T}$, but still requires $\mu \lesssim 1/(d\sqrt{n})$. In comparison, we conduct the first stability analysis for Zeroth-order SGDA under the same mild condition $\eta \lesssim 1/\sqrt{T}$, and do not need any requirement for $\mu$. {We give the detailed comparison in Table~\ref{tab:stability}.}
\end{remark}

\begin{table}[h]
\centering
\caption{Stability Analysis for Zeroth-order Methods, $G$ denotes the bounded gradient assumption and $G_S$ denotes the bounded second moment assumption (Assumption~\ref{ass:moment}), $L$ denotes the smootheness assumption (\citet{nikolakakis2022black,liu2023general,zeng2026stochastic} consider Zeroth-order SGD). We ignore all the factors independent of $d,K,n,T,\eta$.}
\tiny
\footnotesize
\setlength{\tabcolsep}{3pt}
\begin{tabular}{@{}ccccc@{}}
\toprule
\textbf{Analysis} & $\boldsymbol{\eta}$ & $\boldsymbol{\mu}$ & \textbf{Assumption} & \textbf{Bounds} \\
\midrule
Stability  in~\citep{nikolakakis2022black} & $O(\frac{1}{t})$ & $O(\frac{1}{d^{3/2}n})$ & $L,G$ & $O(\frac{T}{n})$ \\
Stability  in~\citep{liu2023general} & $O(\frac{1}{t})$ & $O(\frac{1}{dn})$ & $L,G$ & $O(\frac{T}{n})$ \\
Stability in~\citep{zeng2026stochastic} & $O(\frac{1}{\sqrt{T}})$ & $O(\frac{1}{d\sqrt{n}})$ & $L$ & $O(\frac{(1+\sqrt{d/K})\eta\sqrt{T}}{\sqrt{n}}+\frac{\eta T}{n})$ \\
Our stability & $O(\frac{1}{\sqrt{T}})$ & No requirement & $L,G_S$ & $O(\frac{(1+\sqrt{d/K})\eta\sqrt{T}}{\sqrt{n}}+\frac{\eta T}{n})$ \\
\bottomrule
\end{tabular}
\label{tab:stability}
\end{table}

We now bound optimization errors for Zeroth-order SGDA with $\wcal=\vcal=B_R$.


\begin{lemma}[Optimization of Zeroth-order SGDA]\label{lem:zero-optimization}
Assume for all $z$, the function $(\bw,\bv)\mapsto f(\bw,\bv;z)$ is convex-concave. Let Assumption~\ref{ass:moment} hold. If we choose the gradient estimator in Eq.~\eqref{g-zero}, then 
\begin{align*}
\ebb[\sup_{\bv \in B_R}F_S(\bar{\bw}_T,\bv)-\inf_{\bw \in B_R} F_S(\bw,\bar{\bv}_T)] \lesssim \frac{R^2}{\eta T}+ (1+\frac{d}{K})\ebb[G_S^2] \eta + \mu \sqrt{\ebb[G_S^2]}.
\end{align*}
\end{lemma}
\begin{remark}[Comparison]
Our optimization error recovers the results in existing works on the convergence analysis of Zeroth-order methods~\citep{duchi2015optimal,shamir2017optimal,nesterov2017random,beznosikov2020gradient}. Furthermore, we show that when $K \asymp d$, an optimization error of order $1/\sqrt{T}$ is achievable under dimension-independent choices of $\eta \lesssim 1/\sqrt{T}, \mu \lesssim 1/\sqrt{T}$, while the existing works require dimension-dependent parameters $\eta, \mu$.
\end{remark}

{We now combine Theorem~\ref{cor:zero-stability} and Lemma~\ref{lem:zero-optimization}
to derive excess risk bounds for projected Zeroth-order SGDA on $B_R$.}

\begin{theorem}[Excess risk of Zeroth-order SGDA\label{thm:stab-zero}]
Let assumptions and definitions in Theorem \ref{thm:stab-gen}, Theorem \ref{cor:zero-stability} and Lemma \ref{lem:zero-optimization} hold. Assume
\begin{align}\label{thm:zero-risk-condition}
\frac{1}{L\sqrt{n}}
+\frac{\eta T}{n}
\lesssim
\frac{1}{L}.
\end{align}
For $R>0$, let $\triangle^{w,R}({\bw},{\bv}):=\sup_{\bv'\in B_R}\ebb[F({\bw},\bv')]-\inf_{\bw'\in B_R}\ebb[F(\bw',{\bv})].
$ 
Then,
\begin{align}\label{thm:zero-risk}
\Delta^{w,R}(\bar{\mathbf{w}}_T,\bar{\mathbf{v}}_T)
\lesssim
\frac{\ebb[G_S^2]/L}{\sqrt{n}}
+\frac{\ebb[G_S^2]}{n}\eta T
+\frac{R^2}{\eta T}
+\frac{\ebb[G_S^2]}{L^2\eta T}
+\mu\sqrt{\ebb[G_S^2]}.
\end{align}
\end{theorem}

\begin{remark}[{Minimax-optimal excess risk}]\label{remark:zero-risk}
With our choice of
$
\mu\lesssim
\frac{\sqrt{\ebb[G_S^2]}}{L\sqrt{n}}$,
Eq.~\eqref{thm:zero-risk} implies
\begin{align*}
\Delta^{w,R}(\bar{\mathbf{w}}_T,\bar{\mathbf{v}}_T)
\lesssim
\frac{\ebb[G_S^2]/L}{\sqrt{n}}
+\frac{\ebb[G_S^2]}{n}\eta T
+\frac{R^2}{\eta T}
+\frac{\ebb[G_S^2]}{L^2\eta T}.
\end{align*}
By choosing $\eta\asymp 1/(L\sqrt{T})$ with a sufficiently small
universal constant and $T\asymp n$, this further yields
\begin{align*}
\Delta^{w,R}(\bar{\mathbf{w}}_T,\bar{\mathbf{v}}_T)
\lesssim
\frac{\ebb[G_S^2]/L+LR^2}{\sqrt{n}}.
\end{align*}
That is, our analysis gives an excess risk bound of order
$1/\sqrt{n}$, which has the minimax-optimal dependence on the sample
size $n$. To our knowledge, this gives the first optimal risk bounds for Zeroth-order SGDA. It is clear that Eq.~\eqref{thm:zero-risk-condition} holds under our
choices of $\eta$ and $T$. Moreover, if we take $K\asymp d$, then our parameter choice shows
$
L^2(1+\frac{d}{K})\eta^2
\lesssim \frac{1}{T}$
which is consistent with the condition $
L^2(1+\frac{4d}{K})\eta^2
\leq \frac{1}{T}$ required in
Theorem~\ref{cor:zero-stability}.
\end{remark}

\begin{table}[h]
\centering
\caption{Excess Risk Analysis for Zeroth-order Methods, $G_S$ denotes the bounded second moment assumption, and $L$ denotes the smootheness assumption (\citet{zeng2026stochastic} consider Zeroth-order SGD).}
\tiny
\footnotesize
\setlength{\tabcolsep}{3pt}
\begin{tabular}{@{}ccccc@{}}
\toprule
\textbf{Analysis} & $\boldsymbol{\eta}$ & $\boldsymbol{\mu}$ & \textbf{Assumption} & \textbf{Bounds} \\
\midrule
Excess risk in~\citep{zeng2026stochastic} & $O\left(\frac{\sqrt{K/d}}{\sqrt{n}}\right)$ & $O(\frac{1}{d\sqrt{n}})$ & $L$ & $O(n^{-1/2})$ \\
Our excess risk bound & $O\left(\frac{\sqrt{K/d}}{\sqrt{n}}\right)$ & $O(\frac{1}{\sqrt{n}})$ & $L,G_S$ & $O(n^{-1/2})$ \\
\bottomrule
\end{tabular}
\label{tab:key_results}
\end{table}

\subsection{Clipped-SGDA}
Gradient clipping is known to alleviate the problem of exploding gradients in recurrent neural networks.  
Clipped-SGDA~\citep{gorbunov2022clipped} applies the clipping operator
$
\clip(\bv,\tau):=\bv\min(1,\tau/\|\bv\|_2)
$ with the clipping  parameter $\tau$
to the gradient estimator. Specifically, for Clipped-SGDA,
$\bw_{t+1},\bv_{t+1}$ are produced by Eq.~\eqref{gda} with
\begin{equation}\label{g-clip}
\begin{cases}
g_{\bw}(\bw_t,\bv_t;z_{i_t})=\clip(\nabla_{\bw} f (\bw_t,\bv_t;z_{i_t}),\tau_1), &\\
g_{\bv}(\bw_t,\bv_t;z_{i_t})=\clip(\nabla_{\bv} f (\bw_t,\bv_t;z_{i_t}),\tau_2), &
\end{cases}
\end{equation}
where $\tau_1,\tau_2 >0$ are clipping parameters.

By clipping gradients, Clipped-SGDA improves the robustness and has found successful applications, e.g., training GPT models \cite{brown2020language} and differentially private algorithms~\citep{abadi2016deep}. 
To apply our general results to Clipped-SGDA, we verify Eq.~\eqref{simple-condition-a} and Eq.~\eqref{simple-condition-c} in Lemma~\ref{lem:condition-a-clip} and Lemma~\ref{lem:condition-b-clip}, proved in Section \ref{sec:proof-clipped-sgda}. We always assume $\tilde{f}_1=\tilde{f}_2=f$, by which we can easily show that Eq.~\eqref{simple-condition-d} holds with $\widetilde{G} = 2\sqrt{2\ebb[G_S^2]}$. We then apply Lemma~\ref{lem:condition-a-clip} and Lemma~\ref{lem:condition-b-clip} to Theorem~\ref{lem:stab-gda}, and get the stability bounds for Clipped-SGDA. 



\begin{theorem}[Stability of Clippeed-SGDA]\label{cor:clip-stability}
Assume for all $z$, the function $(\bw,\bv)\mapsto f(\bw,\bv;z)$ is convex-concave. Let Assumptions \ref{ass:smooth}, \ref{ass:moment} hold with $G_S\leq G$ for some $G>0$. For $\eta,\tau_1,\tau_2$ satisfying {$L^2\eta^2+4\sqrt{2}LG^2(\frac{1}{\tau_1^2}+\frac{1}{\tau_2^2})\eta \leq \frac{1}{T}$}, Clipped-SGDA is $\epsilon$-argument-stable in expectation with
$\epsilon \lesssim G\eta T/n + G \eta \sqrt{T}/\sqrt{n}$.
\end{theorem}

\begin{remark}\label{rem:clip-stab}
If we choose $\tau_1,\tau_2 \gtrsim G T^{\frac{1}{4}}$ and $\eta \lesssim 1/(L\sqrt{T})$, then the conditions in Theorem \ref{cor:clip-stability} hold. It is clear that the stability bound of Clipped-SGDA is of the same order as the stability bound for SGDA in Theorem \ref{cor:sgda_stability}, {which is also verified in our experiments (Section~\ref{sec:experiment}).}
\end{remark}

We now control optimization errors for Clipped-SGDA.
{Analogous to Lemma~\ref{lem:zero-optimization}, we consider projected Clipped-SGDA on $B_R$.}
\begin{lemma}[Optimization error for Clipped-SGDA\label{lem:opt-clip}]
Assume for all $z$, the map $(\bw,\bv)\mapsto f(\bw,\bv;z)$ is convex-concave. Let Assumptions \ref{ass:smooth}, \ref{ass:moment} hold with $G_S\leq G$ for some $G>0$.
Assume $G \lesssim \tau_1$ and $G \lesssim \tau_2$. If we choose gradient estimator in Eq.~\eqref{g-clip}, then 
\begin{align*}
\ebb\big[\sup_{\bv\in B_R}F_S(\bar{\bw}_T,\bv)-\inf_{\bw\in B_R}F_S(\bw,\bar{\bv}_T)\big] \lesssim \frac{R^2}{\eta T}+G^2\eta
+G^{4}(\tau_1^{-2}+\tau_2^{-2})\eta T.
\end{align*}
\end{lemma}

{
We now combine Theorem~\ref{cor:clip-stability} and Lemma~\ref{lem:opt-clip}
to get the following risk bounds for projected Clipped-SGDA on $B_R$.}
Recall $\Delta^{w,R}$ as defined in Theorem~\ref{thm:stab-zero}.
\begin{theorem}\label{thm:clip-gen}
Let assumptions and definitions in Theorem~\ref{thm:stab-gen}, Theorem~\ref{cor:clip-stability} and Lemma \ref{lem:opt-clip} hold.
Then,
\begin{align}
\Delta^{w,R}(\bar{\mathbf{w}}_T, \bar{\mathbf{v}}_T)
\lesssim \frac{LG^2\eta^2T^2}{n^2} + \frac{LG^2\eta^2T}{n} + \frac{G^2\eta T}{n}
+ \frac{G^2 \eta \sqrt{T}}{\sqrt{n}}
+\frac{R^2}{\eta T}+G^2\eta +G^4(\tau_1^{-2}+\tau_2^{-2})\eta T.\notag
\end{align}
\end{theorem}

\begin{remark}[{Minimax optimal excess risk}]\label{remark:clip-risk}
Setting $\tau_1,\tau_2 \asymp L^{\frac{1}{2}}G T^{\frac{1}{2}}$ similarly to Clipped-SGD~\citep{nguyen2023high}), Theorem~\ref{thm:clip-gen} implies
\begin{align*}
\Delta^{w,R}(\bar{\mathbf{w}}_T, \bar{\mathbf{v}}_T)
\lesssim \frac{LG^2\eta^2T^2}{n^2} + \frac{LG^2\eta^2T}{n} + \frac{G^2\eta T}{n}
+ \frac{G^2 \eta \sqrt{T}}{\sqrt{n}}
+\frac{R^2}{\eta T}+G^2\eta+\frac{G^2}{L}\eta T.
\end{align*}
By choosing $\eta \asymp 1/(L\sqrt{T})$ and $T \asymp n$, this further yields
\begin{align*}
\Delta^{w,R}(\bar{\mathbf{w}}_T, \bar{\mathbf{v}}_T)
\lesssim (G^2/L+ R^2 L + G^2/L^2)/\sqrt{n}.
\end{align*}
That is, our analysis gives excess risk rates of order $1/\sqrt{n}$. This shows our generalization analysis of general biased gradient methods implies optimal excess risk bounds for Clipped-SGDA.
With our choices of $\tau_1,\tau_2,\eta$, it is clear that
$
L^2 \eta^2 + L G^2 \big(\frac{1}{\tau_1^2}+\frac{1}{\tau_2^2}\big)\eta \lesssim 1/T,
$
which is consistent with condition $L^2\eta^2+4\sqrt{2}LG^2(\frac{1}{\tau_1^2}+\frac{1}{\tau_2^2})\eta \leq \frac{1}{T}$ in Theorem~\ref{cor:clip-stability}.
\end{remark}

\section{Future work}\label{sec:limitation}
There are some interesting problems for further study.
First, it is interesting to apply our framework to other biased minimax optimization methods such as decentralized SGDA~\cite{zhu2023stability}. Second, as we only consider the convex-concave problem, it is interesting to extend our analysis to strongly-convex-strongly-concave~\citep{palaniappan2016stochastic}, nonconvex-strongly-concave~\citep{lin2020gradient,luo2020stochastic} and
~\citep{lin2020gradient,luo2020stochastic,rafique2022weakly}
nonconvex-concave problems~\citep{yang2020global}. 
Furthermore, we can extend our work to minimax problems with neural networks, which satisfy the $\rho$-weak-convexity-weak-concavity assumption~\cite{lei2021stability}, with the parameter $\rho$ decaying as the order of $1 / \sqrt{m}$ ($m$ is the network width)~\citep{richards2021stability}.

\section*{Impact Statements}
This paper presents work whose goal is to advance the field of machine learning. There are many potential societal consequences of our work, none of which we feel must be specifically highlighted here.

\section*{Acknowledgments}
We are grateful to the anonymous reviewers and area chairs for their insightful and constructive
comments. The work of Yunwen Lei was supported by the Hong Kong Research Grants Council under the GRF projects 17305425 and 17302624.  The work of Yiming Ying is supported by Australian Research Council Discovery Project (DP250101359).

{\bibliographystyle{abbrvnat}\small\setlength{\bibsep}{0.04cm}
\bibliography{neurips}}

@inproceedings{lei2021stability,
  title={Stability and generalization of stochastic gradient methods for minimax problems},
  author={Lei, Yunwen and Yang, Zhenhuan and Yang, Tianbao and Ying, Yiming},
  booktitle={International Conference on Machine Learning},
  pages={6175--6186},
  year={2021},
  organization={PMLR}
}

@inproceedings{zhu2023stability,
  title={Stability and generalization of the decentralized stochastic gradient descent ascent algorithm},
  author={Zhu, Miaoxi and Shen, Li and Du, Bo and Tao, Dacheng},
  booktitle={Advances in Neural Information Processing Systems},
  volume={36},
  pages={31717--31751},
  year={2023}
}

@article{zeng2026stochastic,
  title={Stochastic Gradient Methods: Bias, Stability and Generalization},
  author={Zeng, Shuang and Lei, Yunwen},
  journal={Journal of Machine Learning Research},
  volume={27},
  number={6},
  pages={1--55},
  year={2026}
}

@article{duchi2015optimal,
  title={Optimal rates for zero-order convex optimization: The power of two function evaluations},
  author={Duchi, John C and Jordan, Michael I and Wainwright, Martin J and Wibisono, Andre},
  journal={IEEE Transactions on Information Theory},
  volume={61},
  number={5},
  pages={2788--2806},
  year={2015},
  publisher={IEEE}
}

@article{nesterov2017random,
  title={Random gradient-free minimization of convex functions},
  author={Nesterov, Yurii and Spokoiny, Vladimir},
  journal={Foundations of Computational Mathematics},
  volume={17},
  number={2},
  pages={527--566},
  year={2017},
  publisher={Springer}
}

@inproceedings{zhang2020adaptive,
  title={Why are adaptive methods good for attention models?},
  author={Zhang, Jingzhao and Karimireddy, Sai Praneeth and Veit, Andreas and Kim, Seungyeon and Reddi, Sashank and Kumar, Sanjiv and Sra, Suvrit},
  booktitle={Advances in Neural Information Processing Systems},
  volume={33},
  pages={15383--15393},
  year={2020}
}

@article{xu2023zeroth,
  title={Zeroth-order alternating gradient descent ascent algorithms for a class of nonconvex-nonconcave minimax problems},
  author={Xu, Zi and Wang, Zi-Qi and Wang, Jun-Lin and Dai, Yu-Hong},
  journal={Journal of Machine Learning Research},
  volume={24},
  number={313},
  pages={1--25},
  year={2023}
}

@inproceedings{beznosikov2020gradient,
  title={Gradient-free methods with inexact oracle for convex-concave stochastic saddle-point problem},
  author={Beznosikov, Aleksandr and Sadiev, Abdurakhmon and Gasnikov, Alexander},
  booktitle={International Conference on Mathematical Optimization Theory and Operations Research},
  pages={105--119},
  year={2020},
  organization={Springer}
}

@inproceedings{chen2017zoo,
  title={Zoo: Zeroth order optimization based black-box attacks to deep neural networks without training substitute models},
  author={Chen, Pin-Yu and Zhang, Huan and Sharma, Yash and Yi, Jinfeng and Hsieh, Cho-Jui},
  booktitle={ACM Workshop on Artificial Intelligence and Security},
  pages={15--26},
  year={2017}
}

@article{wang2021flaml,
  title={Flaml: A fast and lightweight automl library},
  author={Wang, Chi and Wu, Qingyun and Weimer, Markus and Zhu, Erkang},
  journal={Proceedings of Machine Learning and Systems},
  volume={3},
  pages={434--447},
  year={2021}
}

@inproceedings{nikolakakis2022black,
  title={Black-box generalization: Stability of zeroth-order learning},
  author={Nikolakakis, Konstantinos and Haddadpour, Farzin and Kalogerias, Dionysis and Karbasi, Amin},
  inproceedings={Advances in Neural Information Processing Systems},
  volume={35},
  pages={31525--31541},
  year={2022}
}

@inproceedings{brown2020language,
  title={Language models are few-shot learners},
  author={Brown, Tom and Mann, Benjamin and Ryder, Nick and Subbiah, Melanie and Kaplan, Jared D and Dhariwal, Prafulla and Neelakantan, Arvind and Shyam, Pranav and Sastry, Girish and Askell, Amanda and others},
  booktitle={Advances in Neural Information Processing Systems},
  volume={33},
  pages={1877--1901},
  year={2020}
}

@inproceedings{abadi2016deep,
  title={Deep learning with differential privacy},
  author={Abadi, Martin and Chu, Andy and Goodfellow, Ian and McMahan, H Brendan and Mironov, Ilya and Talwar, Kunal and Zhang, Li},
  booktitle={ACM SIGSAC Conference on Computer and Communications Security},
  pages={308--318},
  year={2016}
}

@article{tucat2025regularized,
  title={Regularized Gradient Clipping Provably Trains Wide and Deep Neural Networks},
  author={Tucat, M and Mukherjee, A and Sen, P and Sun, M and Rivasplata, O},
  journal={Transactions on Machine Learning Research},
  volume={2025},
  year={2025}
}

@article{rockafellar1976monotone,
  title={Monotone operators and the proximal point algorithm},
  author={Rockafellar, R Tyrrell},
  journal={SIAM Journal on Control and Optimization},
  volume={14},
  number={5},
  pages={877--898},
  year={1976},
  publisher={SIAM}
}

@inproceedings{zhang2021generalization,
  title={Generalization bounds for stochastic saddle point problems},
  author={Zhang, Junyu and Hong, Mingyi and Wang, Mengdi and Zhang, Shuzhong},
  booktitle={International Conference on Artificial Intelligence and Statistics},
  pages={568--576},
  year={2021},
  organization={PMLR}
}

@inproceedings{farnia2021train,
  title={Train simultaneously, generalize better: Stability of gradient-based minimax learners},
  author={Farnia, Farzan and Ozdaglar, Asuman},
  booktitle={International Conference on Machine Learning},
  pages={3174--3185},
  year={2021},
  organization={PMLR}
}

@inproceedings{bassily2023differentially,
  title={Differentially private algorithms for the stochastic saddle point problem with optimal rates for the strong gap},
  author={Bassily, Raef and Guzm{\'a}n, Crist{\'o}bal and Menart, Michael},
  booktitle={Annual Conference on Learning Theory},
  pages={2482--2508},
  year={2023},
  organization={PMLR}
}

@article{nemirovski2009robust,
  title={Robust stochastic approximation approach to stochastic programming},
  author={Nemirovski, Arkadi and Juditsky, Anatoli and Lan, Guanghui and Shapiro, Alexander},
  journal={SIAM Journal on optimization},
  volume={19},
  number={4},
  pages={1574--1609},
  year={2009},
  publisher={SIAM}
}

@article{lei2026learning,
  title={Learning Theory of Stochastic Regularized Dual Averaging},
  author={Lei, Yunwen and Wang, Zimeng and Yuan, Xiaoming},
  journal={IEEE Transactions on Information Theory},
  year={2026},
  volume={72},
  number={10},
  pages={7709--7730},
  publisher={IEEE}
}

@inproceedings{palaniappan2016stochastic,
  title={Stochastic variance reduction methods for saddle-point problems},
  author={Palaniappan, Balamurugan and Bach, Francis},
  booktitle={Advances in Neural Information Processing Systems},
  volume={29},
  year={2016}
}

@inproceedings{lin2020gradient,
  title={On gradient descent ascent for nonconvex-concave minimax problems},
  author={Lin, Tianyi and Jin, Chi and Jordan, Michael},
  booktitle={International Conference on Machine Learning},
  pages={6083--6093},
  year={2020},
  organization={PMLR}
}

@inproceedings{luo2020stochastic,
  title={Stochastic recursive gradient descent ascent for stochastic nonconvex-strongly-concave minimax problems},
  author={Luo, Luo and Ye, Haishan and Huang, Zhichao and Zhang, Tong},
  booktitle={Advances in Neural Information Processing Systems},
  volume={33},
  pages={20566--20577},
  year={2020}
}

@article{rafique2022weakly,
  title={Weakly-convex--concave min--max optimization: provable algorithms and applications in machine learning},
  author={Rafique, Hassan and Liu, Mingrui and Lin, Qihang and Yang, Tianbao},
  journal={Optimization Methods and Software},
  volume={37},
  number={3},
  pages={1087--1121},
  year={2022},
  publisher={Taylor \& Francis}
}

@article{liu2021first,
  title={First-order convergence theory for weakly-convex-weakly-concave min-max problems},
  author={Liu, Mingrui and Rafique, Hassan and Lin, Qihang and Yang, Tianbao},
  journal={Journal of Machine Learning Research},
  volume={22},
  number={169},
  pages={1--34},
  year={2021}
}

@inproceedings{yang2020global,
  title={Global convergence and variance reduction for a class of nonconvex-nonconcave minimax problems},
  author={Yang, Junchi and Kiyavash, Negar and He, Niao},
  booktitle={Advances in Neural Information Processing Systems},
  volume={33},
  pages={1153--1165},
  year={2020}
}

@article{ajalloeian2020convergence,
  title={On the convergence of {SGD} with biased gradients},
  author={Ajalloeian, Ahmad and Stich, Sebastian U},
  journal={arXiv preprint arXiv:2008.00051},
  year={2020}
}

@article{driggs2022biased,
  title={On biased stochastic gradient estimation},
  author={Driggs, Derek and Liang, Jingwei and Sch{\"o}nlieb, Carola-Bibiane},
  journal={Journal of Machine Learning Research},
  volume={23},
  number={24},
  pages={1--43},
  year={2022}
}

@inproceedings{hu2020biased,
  title={Biased stochastic first-order methods for conditional stochastic optimization and applications in meta learning},
  author={Hu, Yifan and Zhang, Siqi and Chen, Xin and He, Niao},
  booktitle={Advances in Neural Information Processing Systems},
  volume={33},
  pages={2759--2770},
  year={2020}
}

@inproceedings{liu2023general,
  title={General Stability Analysis for Zeroth-Order Optimization Algorithms},
  author={Liu, Xinyue and Zhang, Hualin and Gu, Bin and Chen, Hong},
  booktitle={International Conference on Learning Representations},
  year={2024}
}

@inproceedings{zhu2024towards,
  title={Towards sharper risk bounds for minimax problems},
  author={Zhu, Bowei and Li, Shaojie and Liu, Yong},
  booktitle={International Joint Conference on Artificial Intelligence},
  pages={5698--5706},
  year={2024}
}

@inproceedings{zhang2024generalization,
  title={Generalization bounds of nonconvex-(strongly)-concave stochastic minimax optimization},
  author={Zhang, Siqi and Hu, Yifan and Zhang, Liang and He, Niao},
  booktitle={International Conference on Artificial Intelligence and Statistics},
  pages={694--702},
  year={2024},
  organization={PMLR}
}

@article{bousquet2002stability,
  title={Stability and generalization},
  author={Bousquet, Olivier and Elisseeff, Andr{\'e}},
  journal={Journal of Machine Learning Research},
  volume={2},
  number={Mar},
  pages={499--526},
  year={2002}
}

@article{shalev2010learnability,
  title={Learnability, stability and uniform convergence},
  author={Shalev-Shwartz, Shai and Shamir, Ohad and Srebro, Nathan and Sridharan, Karthik},
  journal={The Journal of Machine Learning Research},
  volume={11},
  pages={2635--2670},
  year={2010},
  publisher={JMLR. org}
}

@InProceedings{feldman2019high,
  title={High probability generalization bounds for uniformly stable algorithms with nearly optimal rate},
  author={Feldman, Vitaly and Vondrak, Jan},
  booktitle = 	 {Conference on Learning Theory},
  pages = 	 {1270--1279},
  year = 	 {2019}
}

@inproceedings{bousquet2020sharper,
  title={Sharper bounds for uniformly stable algorithms},
  author={Bousquet, Olivier and Klochkov, Yegor and Zhivotovskiy, Nikita},
  booktitle={Conference on Learning Theory},
  pages={610--626},
  year={2020}
}

@inproceedings{kuzborskij2018data,
  title={Data-Dependent Stability of Stochastic Gradient Descent},
  author={Kuzborskij, Ilja and Lampert, Christoph},
  booktitle={International Conference on Machine Learning},
  pages={2820--2829},
  year={2018}
}

@inproceedings{hardt2016train,
  title={Train faster, generalize better: Stability of stochastic gradient descent},
  author={Hardt, Moritz and Recht, Ben and Singer, Yoram},
  booktitle={International Conference on Machine Learning},
  pages={1225--1234},
  year={2016}
}

@inproceedings{lei2020fine,
  title={Fine-Grained Analysis of Stability and Generalization for Stochastic Gradient Descent},
  author={Lei, Yunwen and Ying, Yiming},
  booktitle={International Conference on Machine Learning},
  pages={5809--5819},
  year={2020}
}

@inproceedings{bassily2020stability,
  title={Stability of Stochastic Gradient Descent on Nonsmooth Convex Losses},
  author={Bassily, Raef and Feldman, Vitaly and Guzm{\'a}n, Crist{\'o}bal and Talwar, Kunal},
  booktitle={Advances in Neural Information Processing Systems},
  pages={4381--4391},
  volume={33},
  year={2020}
}

@inproceedings{gorbunov2022clipped,
  title={Clipped stochastic methods for variational inequalities with heavy-tailed noise},
  author={Gorbunov, Eduard and Danilova, Marina and Dobre, David and Dvurechenskii, Pavel and Gasnikov, Alexander and Gidel, Gauthier},
  booktitle={Advances in Neural Information Processing Systems},
  volume={35},
  pages={31319--31332},
  year={2022}
}

@inproceedings{goodfellow2014generative,
  title={Generative adversarial nets},
  author={Goodfellow, Ian J and Pouget-Abadie, Jean and Mirza, Mehdi and Xu, Bing and Warde-Farley, David and Ozair, Sherjil and Courville, Aaron and Bengio, Yoshua},
  booktitle={Advances in Neural Information Processing Systems},
  volume={27},
  year={2014}
}

@article{madry2017towards,
  title={Towards deep learning models resistant to adversarial attacks},
  author={Madry, Aleksander and Makelov, Aleksandar and Schmidt, Ludwig and Tsipras, Dimitris and Vladu, Adrian},
  journal={arXiv preprint arXiv:1706.06083},
  year={2017}
}

@inproceedings{chen2017robust,
  title={Robust optimization for non-convex objectives},
  author={Chen, Robert S and Lucier, Brendan and Singer, Yaron and Syrgkanis, Vasilis},
  booktitle={Advances in Neural Information Processing Systems},
  volume={30},
  year={2017}
}

@inproceedings{ying2016stochastic,
  title={Stochastic online AUC maximization},
  author={Ying, Yiming and Wen, Longyin and Lyu, Siwei},
  booktitle={Advances in Neural Information Processing Systems},
  volume={29},
  year={2016}
}

@inproceedings{daskalakis2018limit,
  title={The limit points of (optimistic) gradient descent in min-max optimization},
  author={Daskalakis, Constantinos and Panageas, Ioannis},
  booktitle={Advances in Neural Information Processing Systems},
  volume={31},
  year={2018}
}

@article{chang2011libsvm,
  title={LIBSVM: A library for support vector machines},
  author={Chang, Chih-Chung and Lin, Chih-Jen},
  journal={ACM Transactions on Intelligent Systems and Technology},
  volume={2},
  number={3},
  pages={1--27},
  year={2011},
  publisher={Acm New York, NY, USA}
}

@article{gao2018information,
  title={On the information-adaptive variants of the ADMM: an iteration complexity perspective},
  author={Gao, Xiang and Jiang, Bo and Zhang, Shuzhong},
  journal={Journal of Scientific Computing},
  volume={76},
  number={1},
  pages={327--363},
  year={2018},
  publisher={Springer}
}

@article{shalev2012online,
  title={Online learning and online convex optimization},
  author={Shalev-Shwartz, Shai},
  journal={Foundations and Trends{\textregistered} in Machine Learning},
  volume={4},
  number={2},
  pages={107--194},
  year={2012},
  publisher={Now Publishers, Inc.}
}

@article{lecun2002gradient,
  title={Gradient-based learning applied to document recognition},
  author={LeCun, Yann and Bottou, L{\'e}on and Bengio, Yoshua and Haffner, Patrick},
  journal={Proceedings of the IEEE},
  volume={86},
  number={11},
  pages={2278--2324},
  year={2002},
  publisher={Ieee}
}

@article{nguyen2023high,
  title={High probability convergence of clipped-sgd under heavy-tailed noise},
  author={Nguyen, Ta Duy and Nguyen, Thien Hang and Ene, Alina and Nguyen, Huy Le},
  journal={arXiv preprint arXiv:2302.05437},
  year={2023}
}

@inproceedings{demidovich2023guide,
  title={A guide through the zoo of biased SGD},
  author={Demidovich, Yury and Malinovsky, Grigory and Sokolov, Igor and Richt{\'a}rik, Peter},
  booktitle={Advances in Neural Information Processing Systems},
  volume={36},
  pages={23158--23171},
  year={2023}
}

@inproceedings{surendran2024non,
  title={Non-asymptotic analysis of biased adaptive stochastic approximation},
  author={Surendran, Sobihan and Fermanian, Adeline and Godichon-Baggioni, Antoine and Le Corff, Sylvain},
  booktitle={Advances in Neural Information Processing Systems},
  volume={37},
  pages={12897--12943},
  year={2024}
}

@article{krizhevsky2009learning,
  title={Learning multiple layers of features from tiny images},
  author={Krizhevsky, Alex and Hinton, Geoffrey},
  year={2009},
  publisher={Toronto, ON, Canada}
}

@inproceedings{liu2015faceattributes,
  title = {Deep Learning Face Attributes in the Wild},
  author = {Liu, Ziwei and Luo, Ping and Wang, Xiaogang and Tang, Xiaoou},
  booktitle = {International Conference on Computer Vision},
  month = {December},
  year = {2015} 
}

@inproceedings{richards2021stability,
  title={Stability \& generalisation of gradient descent for shallow neural networks without the neural tangent kernel},
  author={Richards, Dominic and Kuzborskij, Ilja},
  booktitle={Advances in Neural Information Processing Systems},
  volume={34},
  pages={8609--8621},
  year={2021}
}

@article{shamir2017optimal,
  title={An optimal algorithm for bandit and zero-order convex optimization with two-point feedback},
  author={Shamir, Ohad},
  journal={Journal of Machine Learning Research},
  volume={18},
  number={52},
  pages={1--11},
  year={2017}
}

@book{ledoux2001concentration,
  title={The concentration of measure phenomenon},
  author={Ledoux, Michel},
  number={89},
  year={2001},
  publisher={American Mathematical Soc.}
}

@article{dolbeault2013sharp,
  title={Sharp interpolation inequalities on the sphere: new methods and consequences},
  author={Dolbeault, Jean and Esteban, Maria J and Kowalczyk, Michal and Loss, Michael},
  journal={Chinese Annals of Mathematics, Series B},
  volume={34},
  number={1},
  pages={99--112},
  year={2013},
  publisher={Springer}
}

@book{absil2008optimization,
  title={Optimization algorithms on matrix manifolds},
  author={Absil, P-A and Mahony, Robert and Sepulchre, Rodolphe},
  year={2008},
  publisher={Princeton University Press}
}


\appendix
\newpage
\section{Experiments}\label{sec:experiment}
In this section, we report experimental results to verify our theoretical findings. Our code is available  \href{https://github.com/Zeng-Shuang/Generalization-Analysis-of-Biased-Stochastic-Gradient-Methods-for-Minimax-Problems}{here}. Our experimental setup follows the work~\citep{lei2021stability}. Specifically, we consider the Area Under ROC Curve (AUC) maximization problem, which can be formulated as a minimax problem~\cite{ying2016stochastic}:
\[
\min _{(\mathbf{w}, a, b) \in \mathbb{R}^{d+2}} \max _{\alpha \in \mathbb{R}} \mathbb{E}[f(\mathbf{w}, a, b, \alpha ; z)],
\]
where
$
f(\mathbf{w}, a, b, \alpha ; z)=(h(\mathbf{w} ; x)- a)^2 \mathbb{I}_{[y=1]} / p+(h(\mathbf{w} ; x)-b)^2 \mathbb{I}_{[y=-1]} /(1-p)
+2(1- \alpha) h(\mathbf{w} ; x)(\mathbb{I}_{[y=-1]} /(1-p)-\mathbb{I}_{[y=1]} / p)-\alpha^2,
$
$h(\mathbf{w}; x)$ denotes a scoring function, $p=\mathbb{P}[y=1]$, and
$\mathbb{I}_{[\cdot]}$ is the indicator function.
We use three datasets from~\citet{chang2011libsvm}: \texttt{svmguide3}, \texttt{sonar}, and \texttt{DNA}. For each training dataset $S$, we build a neighboring dataset $S'$ by changing its last example. We then apply SGDA, Zeroth-order SGDA, and Clipped SGDA to both $S$ and $S'$ to get model sequences $\{(\mathbf{w}_t, \mathbf{v}_t)\}$ and $\{(\mathbf{w}_t^{\prime}, \mathbf{v}_t^{\prime})\}$. We then evaluate the stability and generalization error of these algorithms along the training process by the Euclidean distance $\Delta_t=(\|\mathbf{w}_t-\mathbf{w}_t^{\prime}\|_2^2+\|\mathbf{v}_t-\mathbf{v}_t^{\prime}\|_2^2)^{\frac{1}{2}}$ and Test AUC (AUC on the test dataset). \\
For each method applied to each data, we use $\eta_t = \eta/\sqrt{T}$ with $\eta \in \{0.01,0.03,0.05\}$, repeat the experiments 25 times to ensure reliability, and show results with mean and standard deviation. We choose $\mu_1=\mu_2=0.0001, K=d$ ($21$ for \texttt{svmguide3}, $60$ for \texttt{sonar} and $180$ for \texttt{DNA}) defined in Eq.~\eqref{g-zero} for Zeroth-order SGDA, and $\tau_1=\tau_2=
0.5\texttt{(svmguide3)},0.1\texttt{(sonar,DNA)}$ defined in Eq.~\eqref{g-clip} for Clipped-SGDA. We use the number of passes to denote the iteration number divided by the training sample size. As shown in Figure~\ref{fig:stability}, $\Delta_t$ of SGDA, Zeroth-order SGDA, and Clipped-SGDA increase along the training process and appear to be comparable (note that $\Delta_t$ of Zeroth-order SGDA is slightly larger due to the scalar $d$ in the definition of gradient estimators of Zeroth-order SGDA in Eq.~\eqref{g-zero}), which verifies our theoretical findings in Theorem~\ref{cor:zero-stability} and Theorem~\ref{cor:clip-stability}. As shown in Figure~\ref{fig:test-auc}, the generalization errors of SGDA, Zeroth-order SGDA, and Clipped-SGDA also appear to be comparable, which verifies our theoretical findings in Theorem~\ref{thm:stab-zero} and Theorem~\ref{thm:clip-gen}.

\begin{figure}[!ht]
    \centering
    \begin{minipage}{0.27\textwidth}
        \centering
        \includegraphics[width=\linewidth]{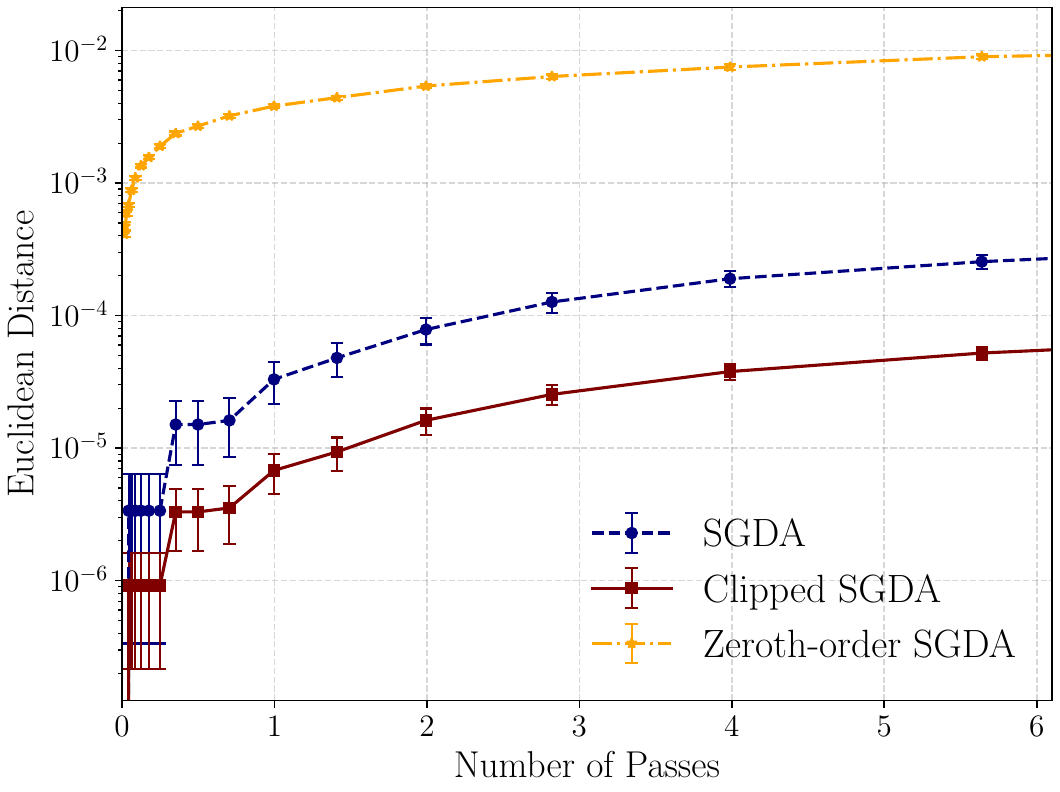}
        \vspace{3pt}
        \scriptsize \texttt{svmguide3}, $\eta=0.01$
    \end{minipage}
    \hfill
    \begin{minipage}{0.27\textwidth}
        \centering
        \includegraphics[width=\linewidth]{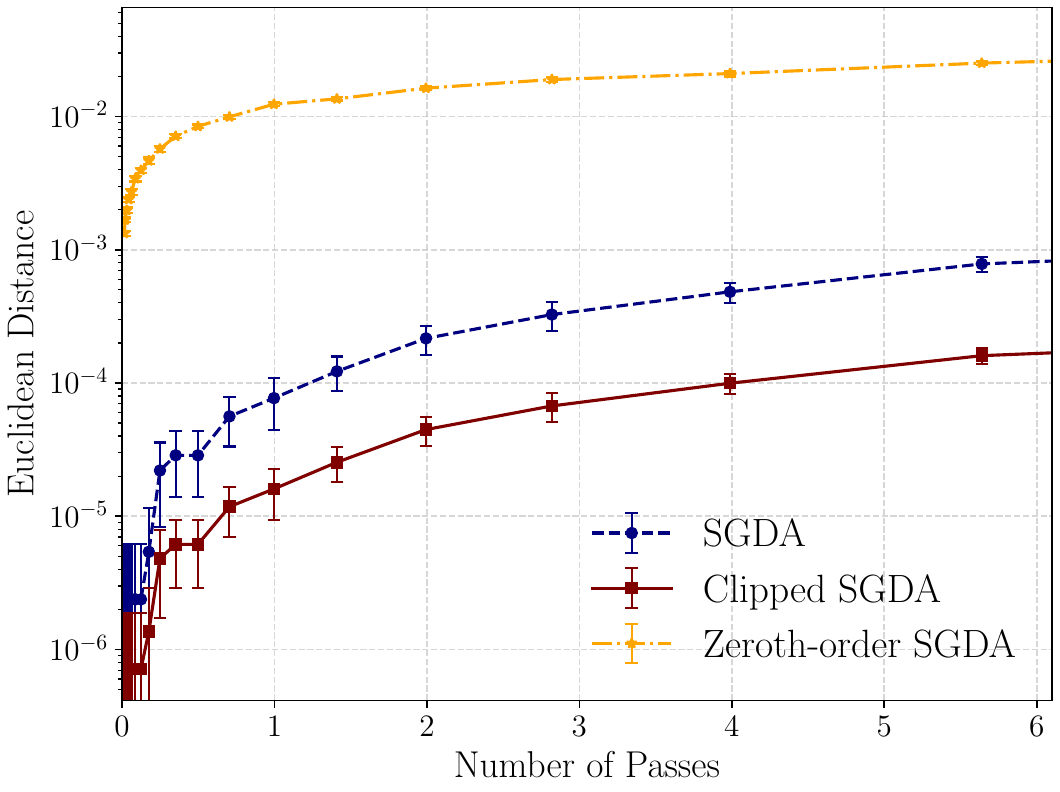}
        \vspace{3pt}
        \scriptsize \texttt{svmguide3}, $\eta=0.03$
    \end{minipage}
    \hfill
    \begin{minipage}{0.27\textwidth}
        \centering
        \includegraphics[width=\linewidth]{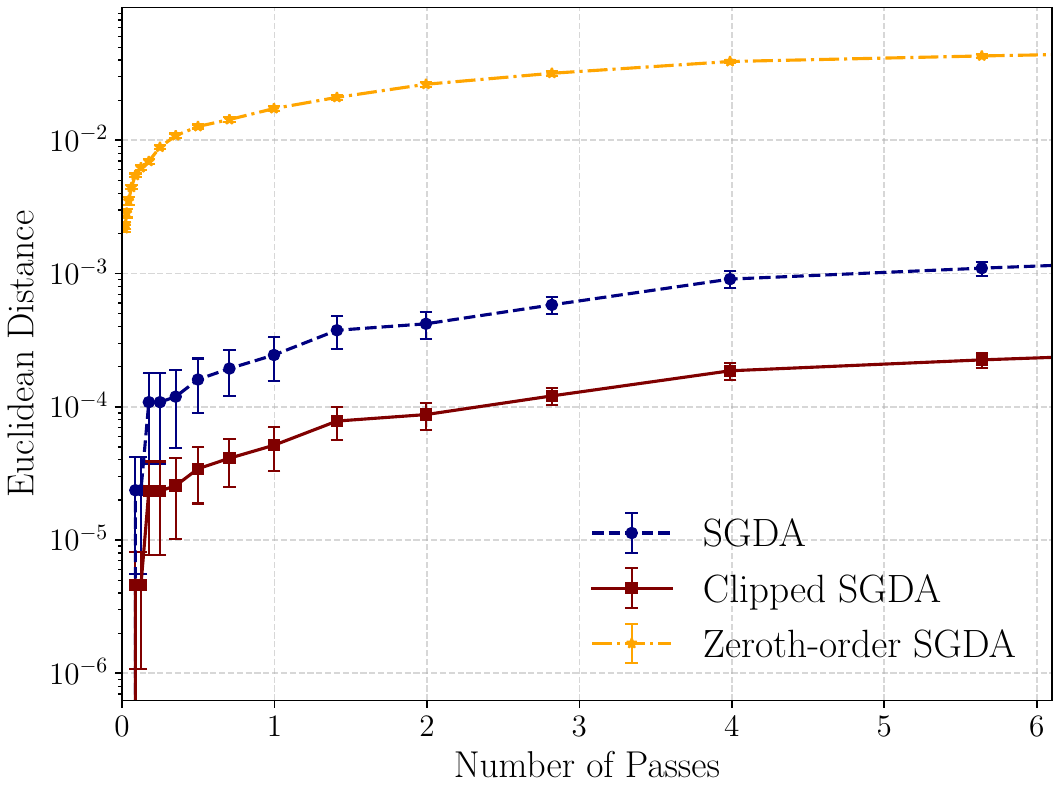}
        \vspace{3pt}
        \scriptsize \texttt{svmguide3}, $\eta=0.05$
    \end{minipage}

    \vspace{3pt} 

    \begin{minipage}{0.27\textwidth}
        \centering
        \includegraphics[width=\linewidth]{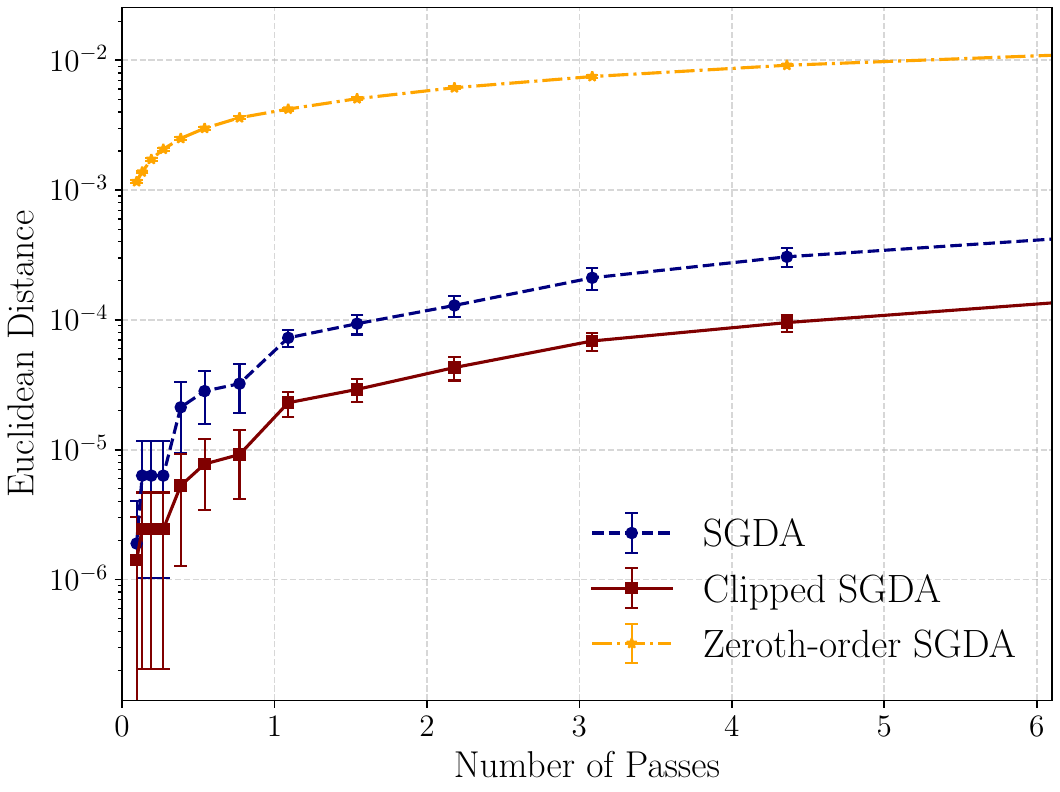}
        \vspace{3pt}
        \scriptsize \texttt{sonar}, $\eta=0.01$
    \end{minipage}
    \hfill
    \begin{minipage}{0.27\textwidth}
        \centering
        \includegraphics[width=\linewidth]{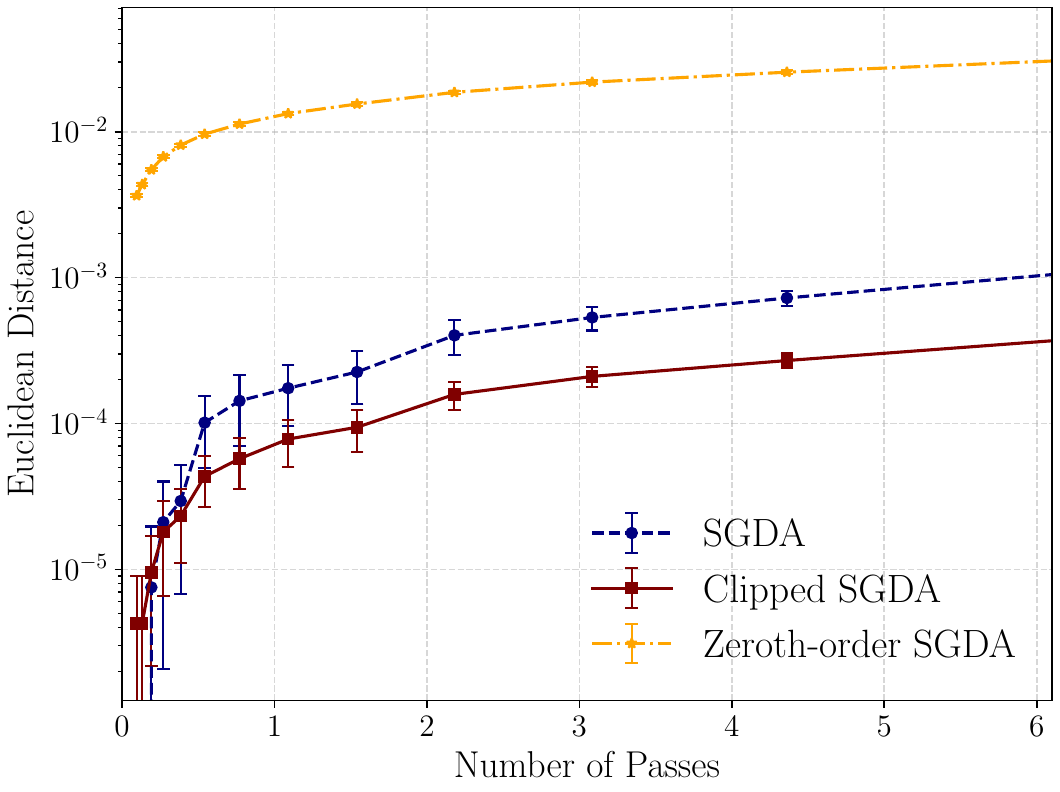}
        \vspace{3pt}
        \scriptsize \texttt{sonar}, $\eta=0.03$
    \end{minipage}
    \hfill
    \begin{minipage}{0.27\textwidth}
        \centering
        \includegraphics[width=\linewidth]{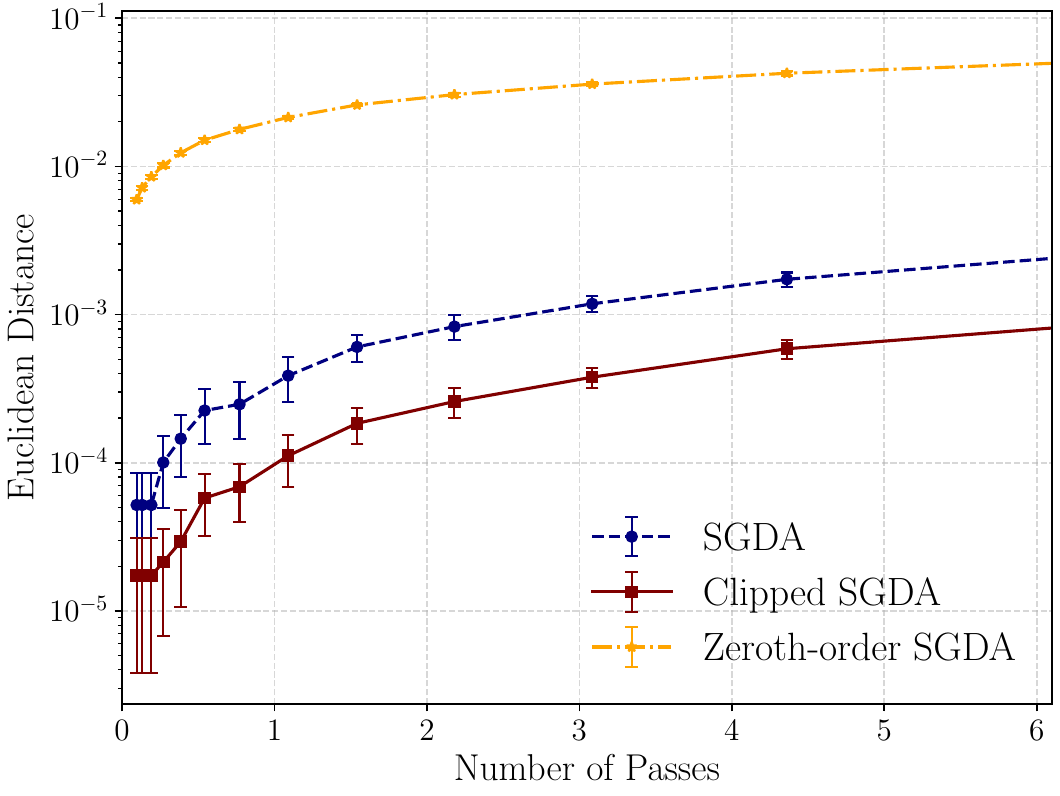}
        \vspace{3pt}
        \scriptsize \texttt{sonar}, $\eta=0.05$
    \end{minipage}

     \vspace{3pt}

    \begin{minipage}{0.27\textwidth}
        \centering
        \includegraphics[width=\linewidth]{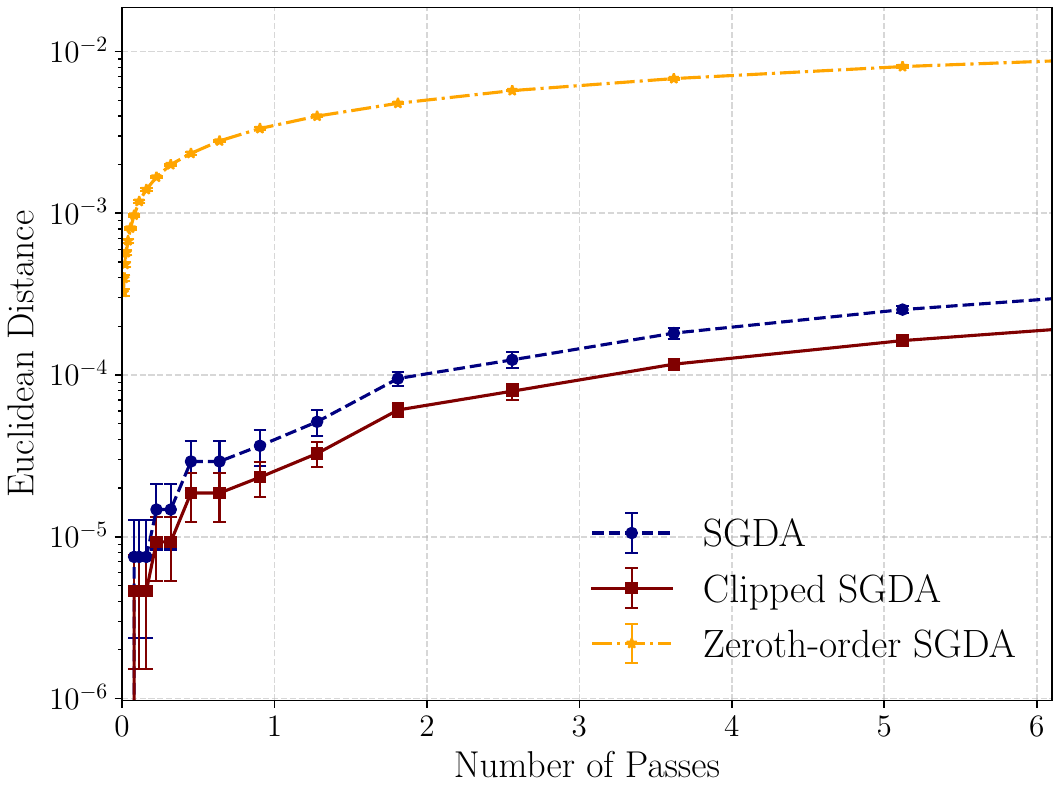}
        \vspace{3pt}
        \scriptsize \texttt{dna}, $\eta=0.01$
    \end{minipage}
    \hfill
    \begin{minipage}{0.27\textwidth}
        \centering
        \includegraphics[width=\linewidth]{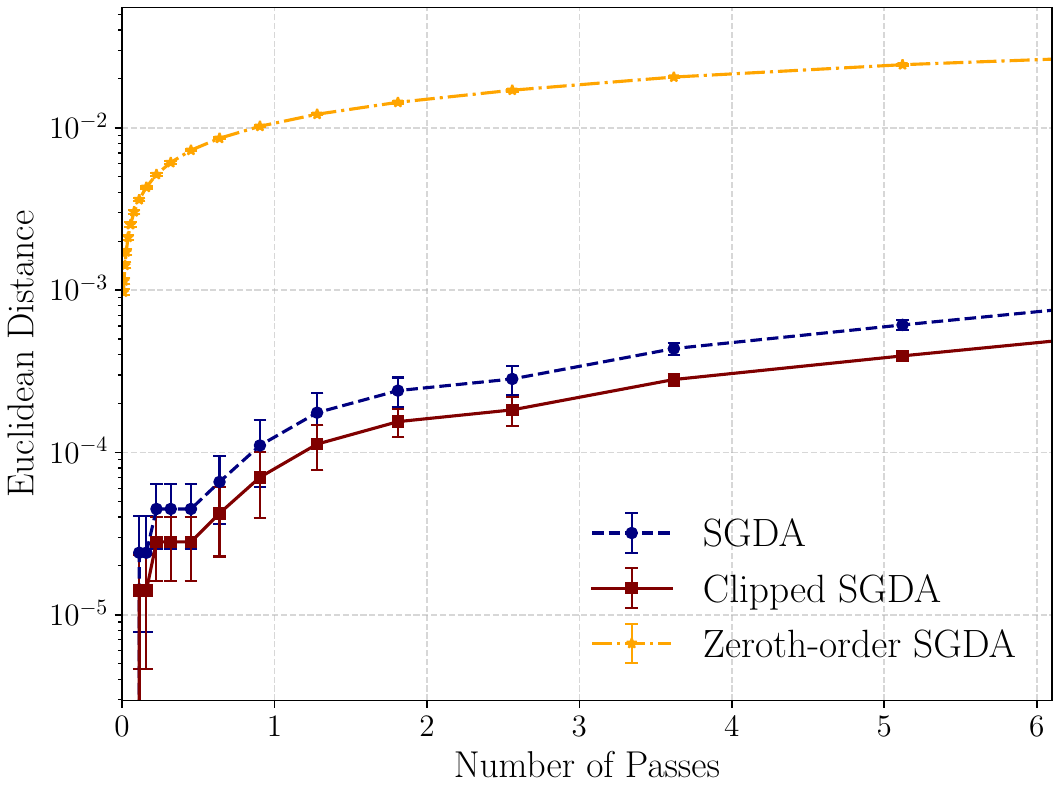}
        \vspace{3pt}
        \scriptsize \texttt{dna}, $\eta=0.03$
    \end{minipage}
    \hfill
    \begin{minipage}{0.27\textwidth}
        \centering
        \includegraphics[width=\linewidth]{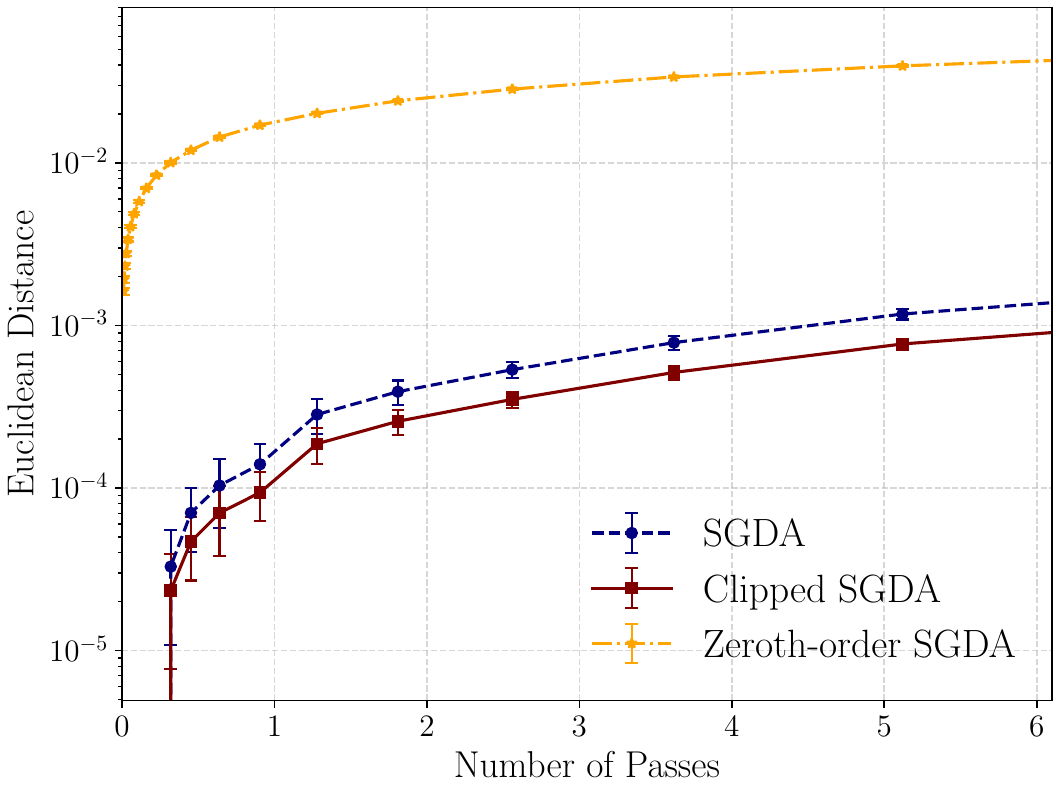}
        \vspace{3pt}
        \scriptsize \texttt{dna}, $\eta=0.05$
    \end{minipage}

     \vspace{3pt}
    \caption{$\Delta_t$ versus the number of passes for different datasets (\texttt{svmguide3}, \texttt{sonar}, \texttt{DNA}) and $\eta\in\{0.01,0.03,0.05\}$.}
    \label{fig:stability}
\end{figure}

\begin{figure}[!ht]
    \centering
    \begin{minipage}{0.27\textwidth}
        \centering
        \includegraphics[width=\linewidth]{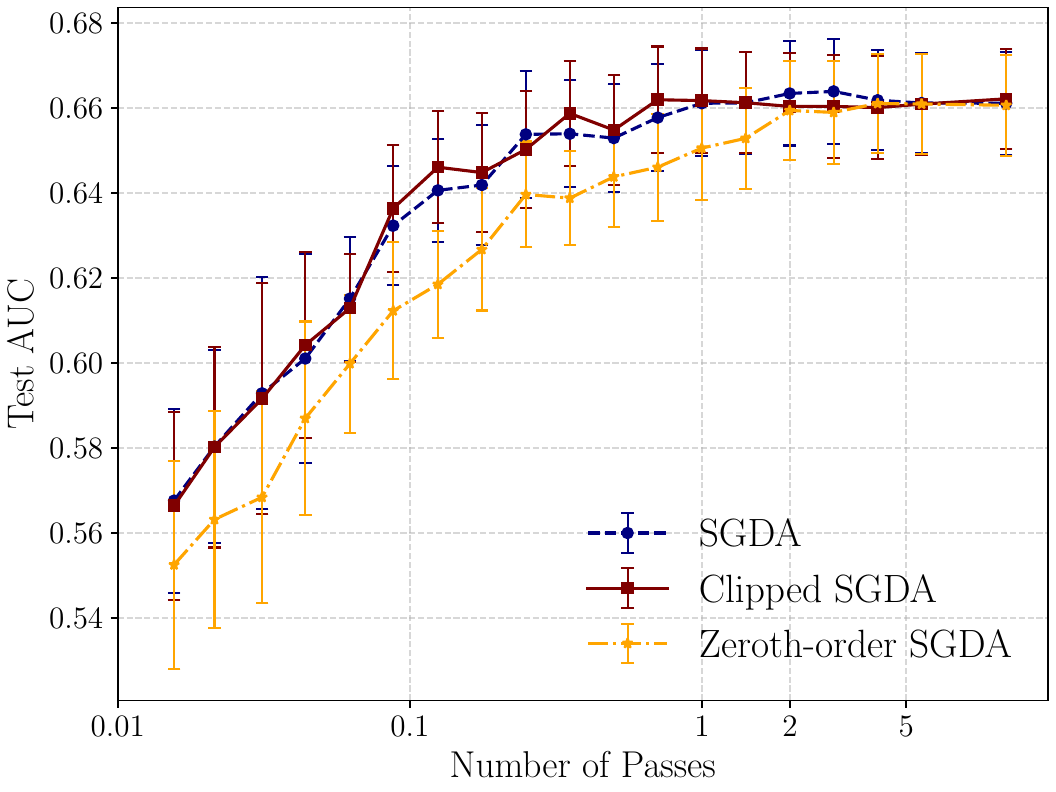}
        \vspace{3pt}
        \scriptsize \texttt{svmguide3}, $\eta=0.01$
    \end{minipage}
    \hfill
    \begin{minipage}{0.27\textwidth}
        \centering
        \includegraphics[width=\linewidth]{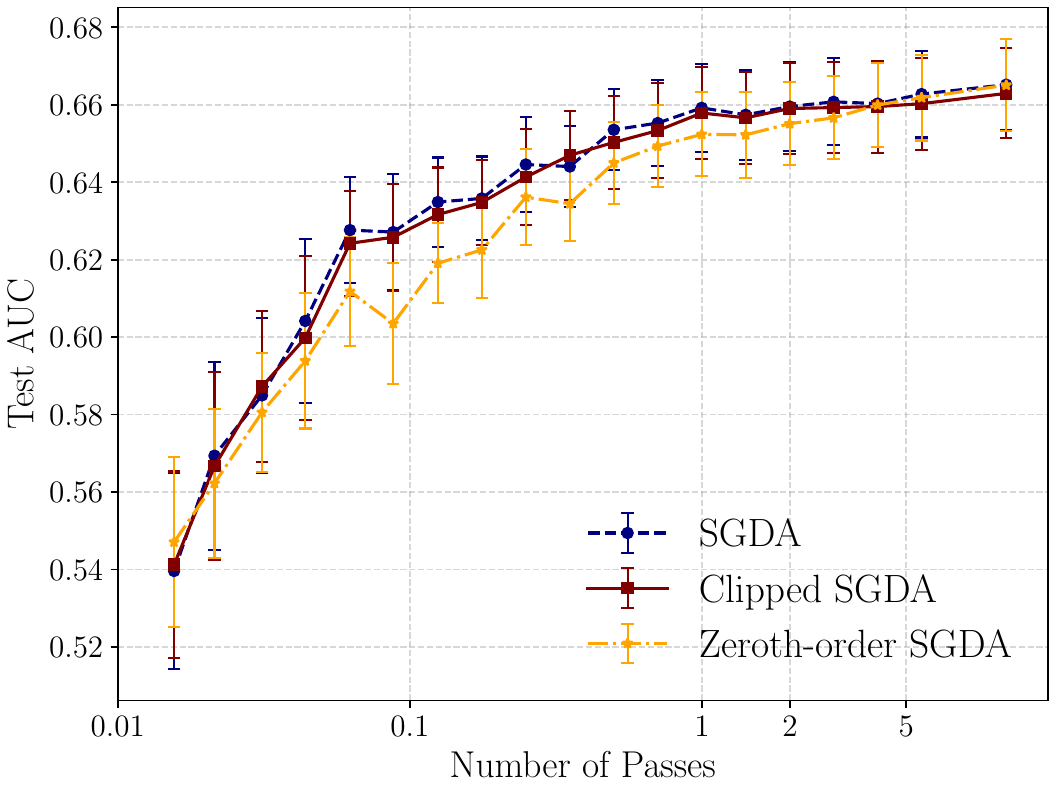}
        \vspace{3pt}
        \scriptsize \texttt{svmguide3}, $\eta=0.03$
    \end{minipage}
    \hfill
    \begin{minipage}{0.27\textwidth}
        \centering
        \includegraphics[width=\linewidth]{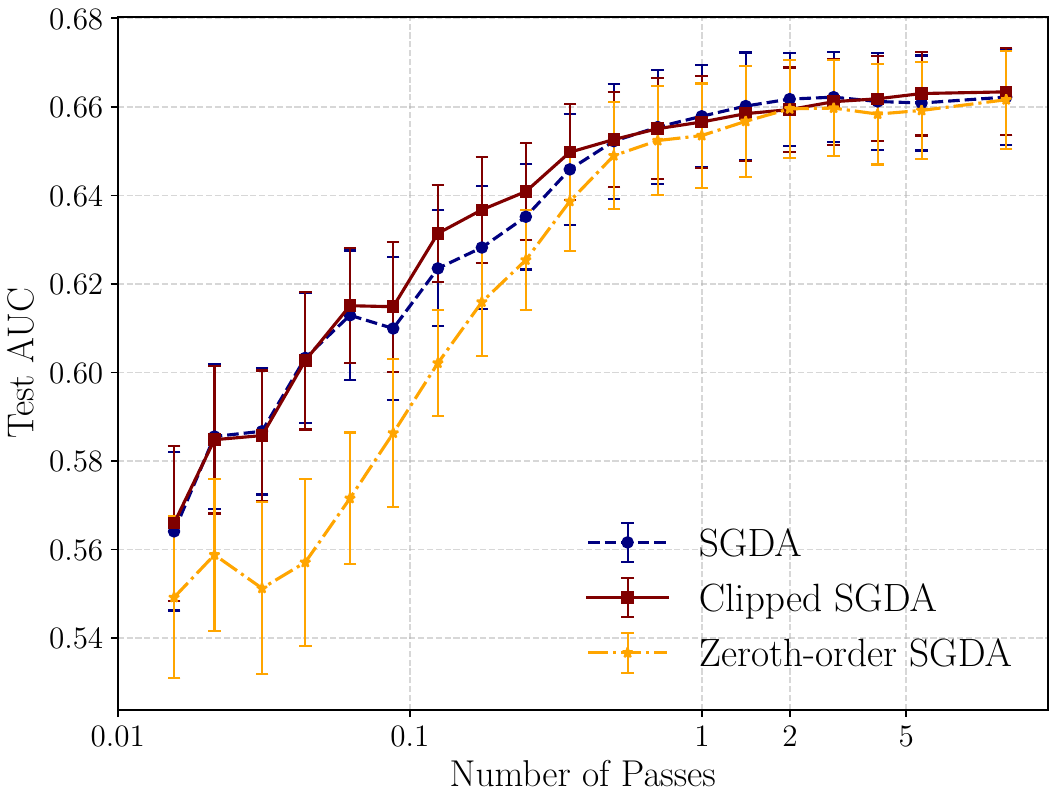}
        \vspace{3pt}
        \scriptsize \texttt{svmguide3}, $\eta=0.05$
    \end{minipage}

    \vspace{3pt} 

    \begin{minipage}{0.27\textwidth}
        \centering
        \includegraphics[width=\linewidth]{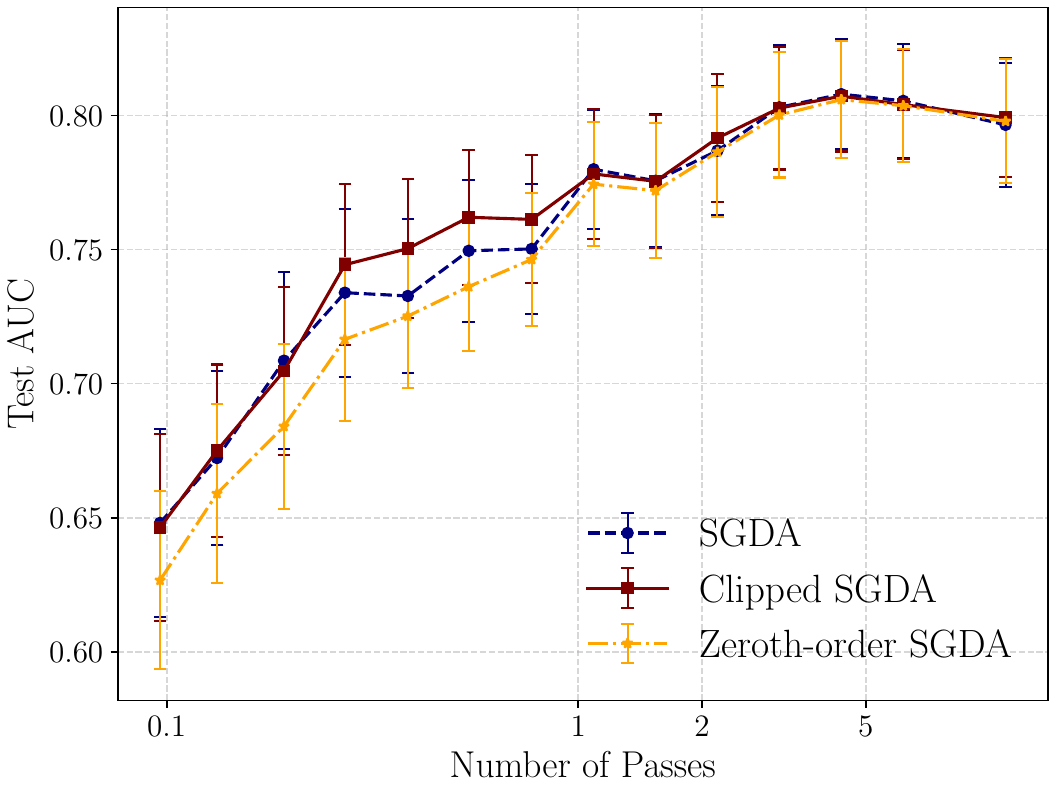}
        \vspace{3pt}
        \scriptsize \texttt{sonar}, $\eta=0.01$
    \end{minipage}
    \hfill
    \begin{minipage}{0.27\textwidth}
        \centering
        \includegraphics[width=\linewidth]{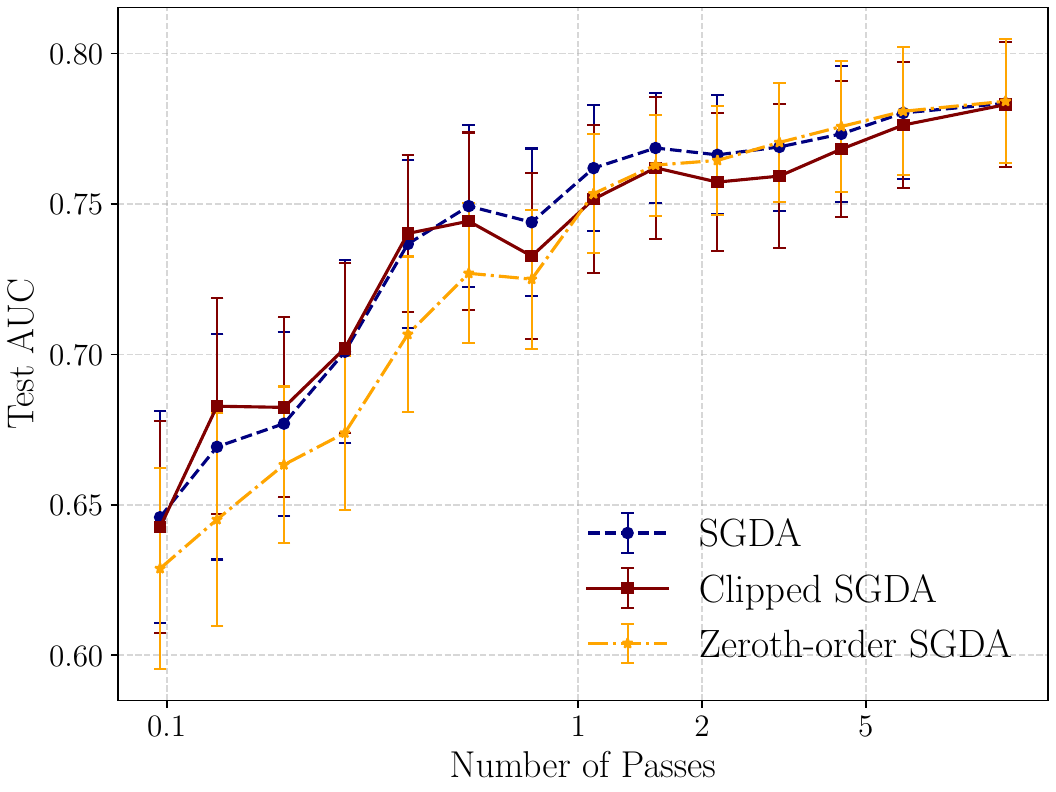}
        \vspace{3pt}
        \scriptsize \texttt{sonar}, $\eta=0.03$
    \end{minipage}
    \hfill
    \begin{minipage}{0.27\textwidth}
        \centering
        \includegraphics[width=\linewidth]{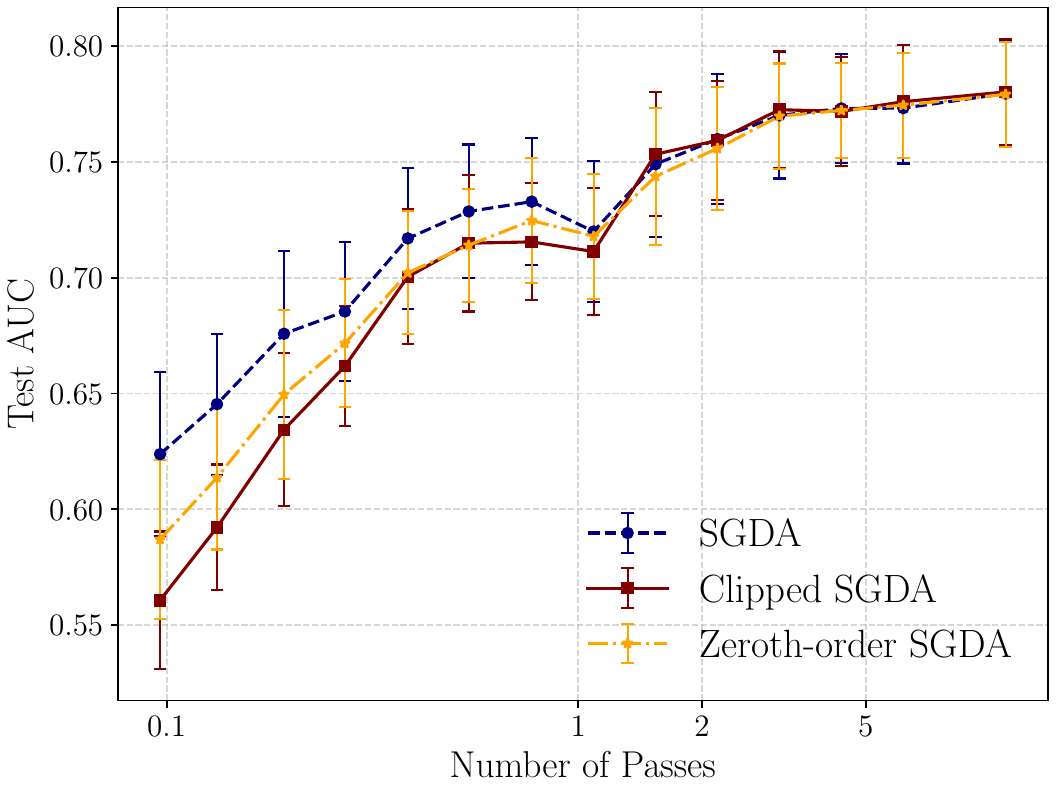}
        \vspace{3pt}
        \scriptsize \texttt{sonar}, $\eta=0.05$
    \end{minipage}

     \vspace{3pt}

    \begin{minipage}{0.27\textwidth}
        \centering
        \includegraphics[width=\linewidth]{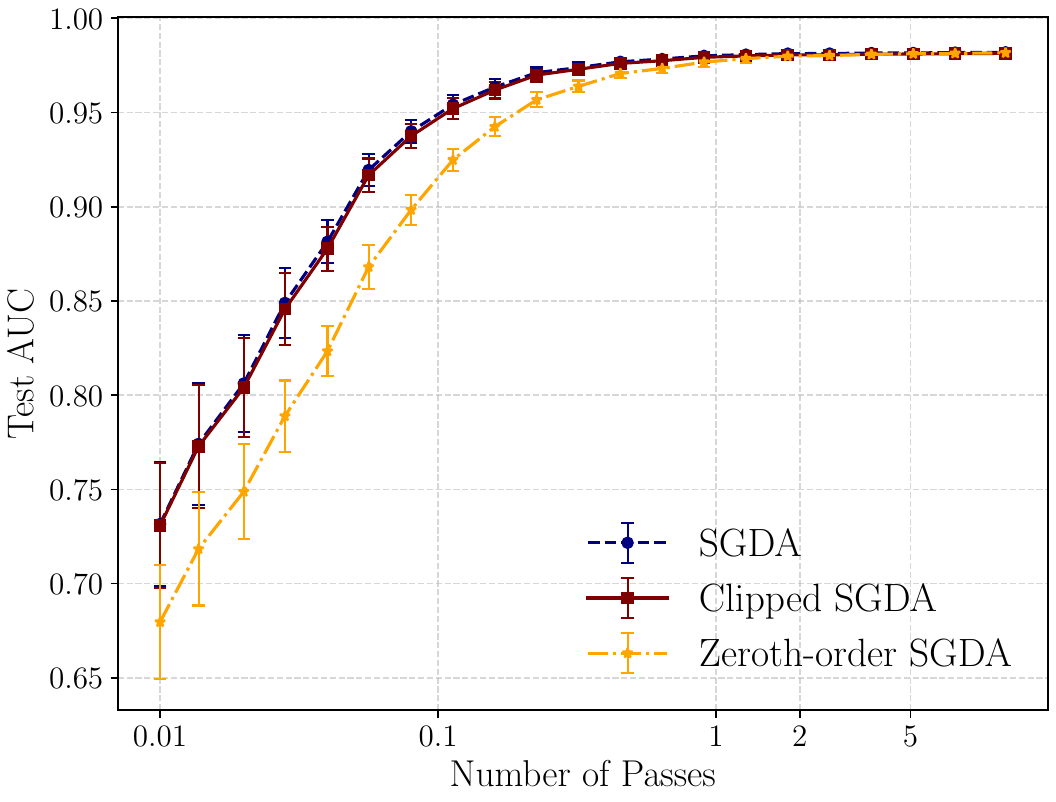}
        \vspace{3pt}
        \scriptsize \texttt{dna}, $\eta=0.01$
    \end{minipage}
    \hfill
    \begin{minipage}{0.27\textwidth}
        \centering
        \includegraphics[width=\linewidth]{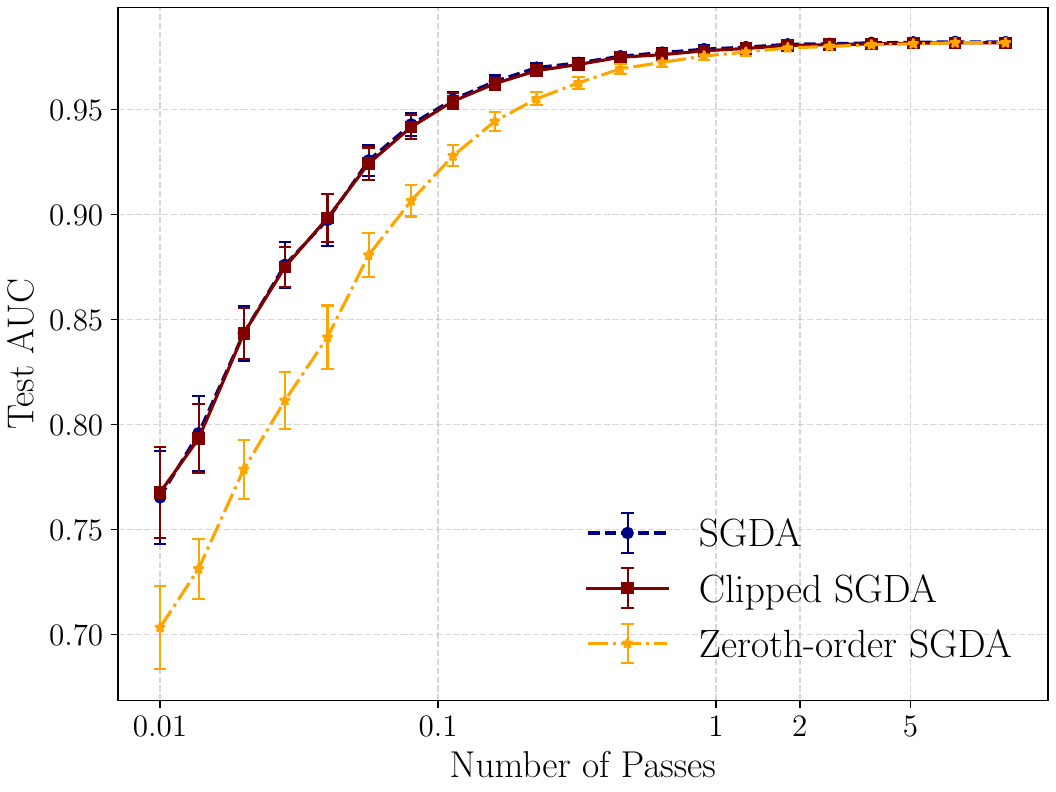}
        \vspace{3pt}
        \scriptsize \texttt{dna}, $\eta=0.03$
    \end{minipage}
    \hfill
    \begin{minipage}{0.27\textwidth}
        \centering
        \includegraphics[width=\linewidth]{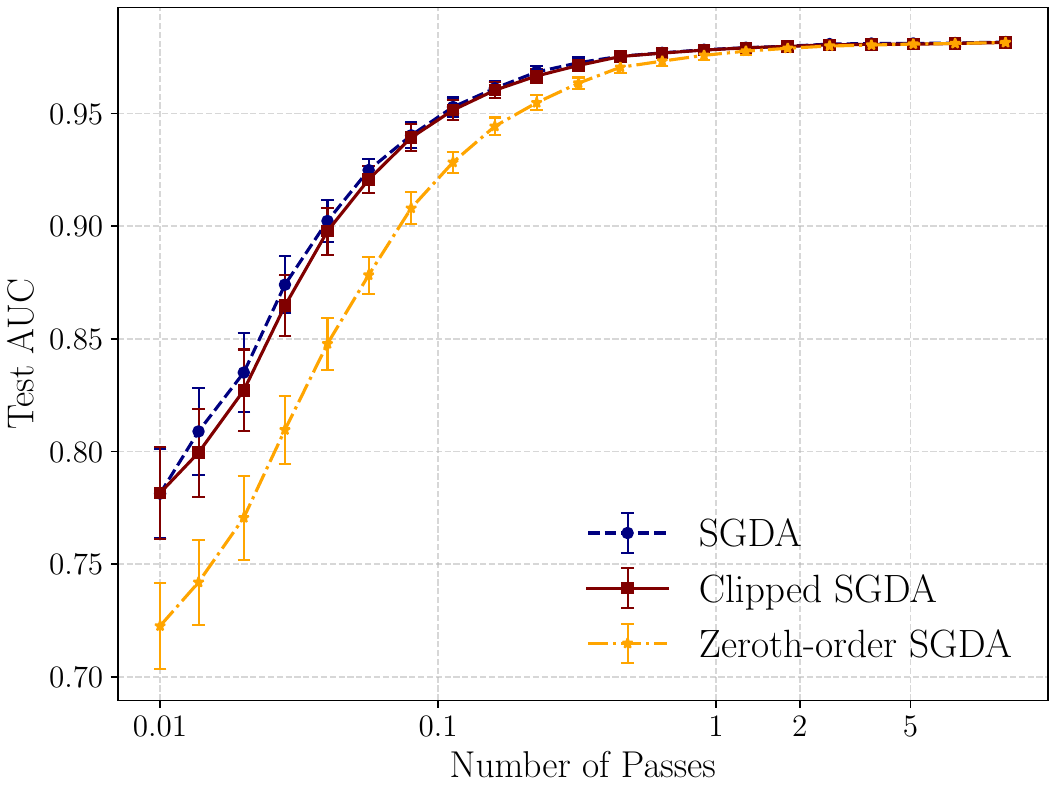}
        \vspace{3pt}
        \scriptsize \texttt{dna}, $\eta=0.05$
    \end{minipage}

     \vspace{3pt}
    \caption{Test AUC versus \# of passes for different datasets (\texttt{svmguide3}, \texttt{sonar}, \texttt{dna}) and $\eta \in\{0.01,0.03,0.05\}$.}
    \label{fig:test-auc}
\end{figure}

We also consider the stability of biased stochastic gradient methods for  a nonconvex-nonconcave problem: GAN~\cite{goodfellow2014generative}  on the MNIST~\cite{lecun2002gradient}, CIFAR-10~\cite{krizhevsky2009learning}, and CelebA~\cite{liu2015faceattributes} datasets.
We conduct the experiments with one NVIDIA GeForce RTX 3090 24GB GPU.
We use $\eta_t = 0.0002$,
choose $\mu_1=\mu_2=0.0001, K\asymp d$ for Zeroth-order SGDA, and clipping parameter $\tau_1=\tau_2=1$  in Eq.~\eqref{g-clip} for Clipped-SGDA.
As shown in Figure~\ref{fig:stability-gan}, $\Delta_t$ of SGDA, Zeroth-order SGDA, and Clipped-SGDA increase along the training process and also appear to be comparable. Interestingly, the stability of Clipped-SGDA outperforms that of SGDA in this case.

\begin{figure}[h]
    \centering

    \begin{subfigure}{\textwidth}
        \centering
        \includegraphics[width=0.8\linewidth]{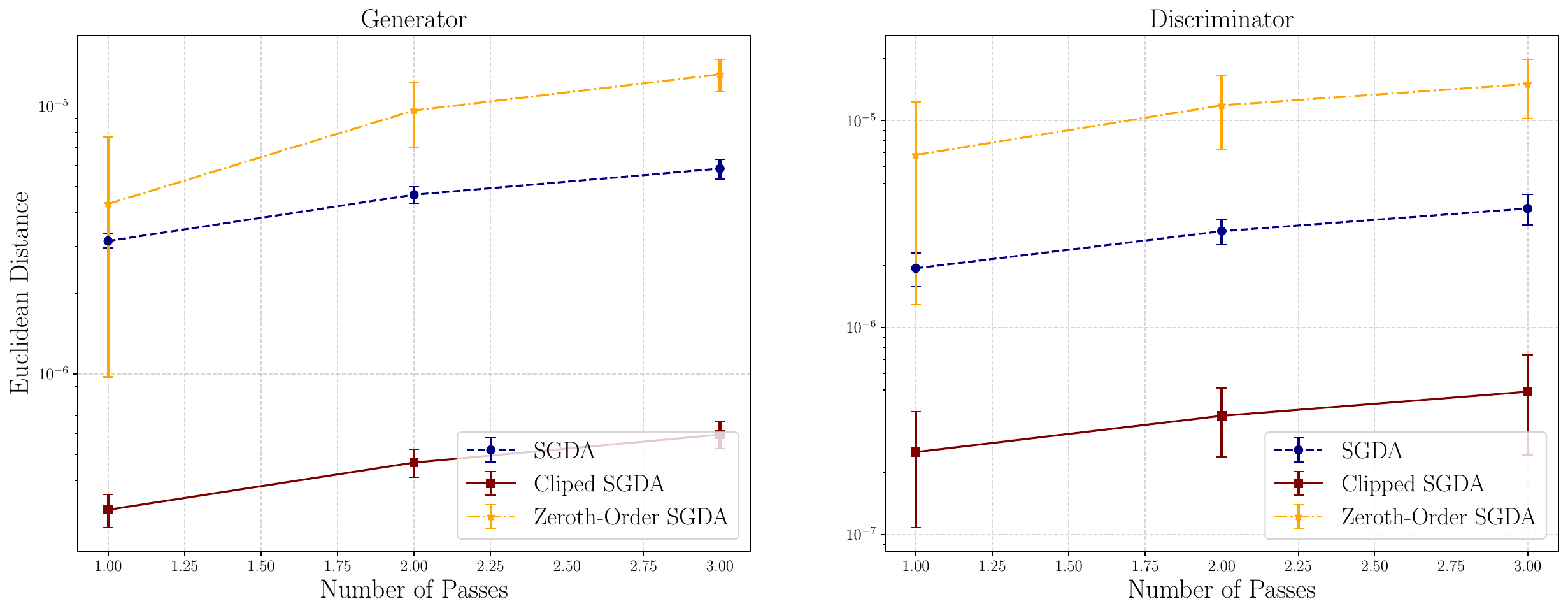}
        \caption{MNIST}
    \end{subfigure}

    \vspace{3pt} 

    \begin{subfigure}{\textwidth}
        \centering
        \includegraphics[width=0.8\linewidth]{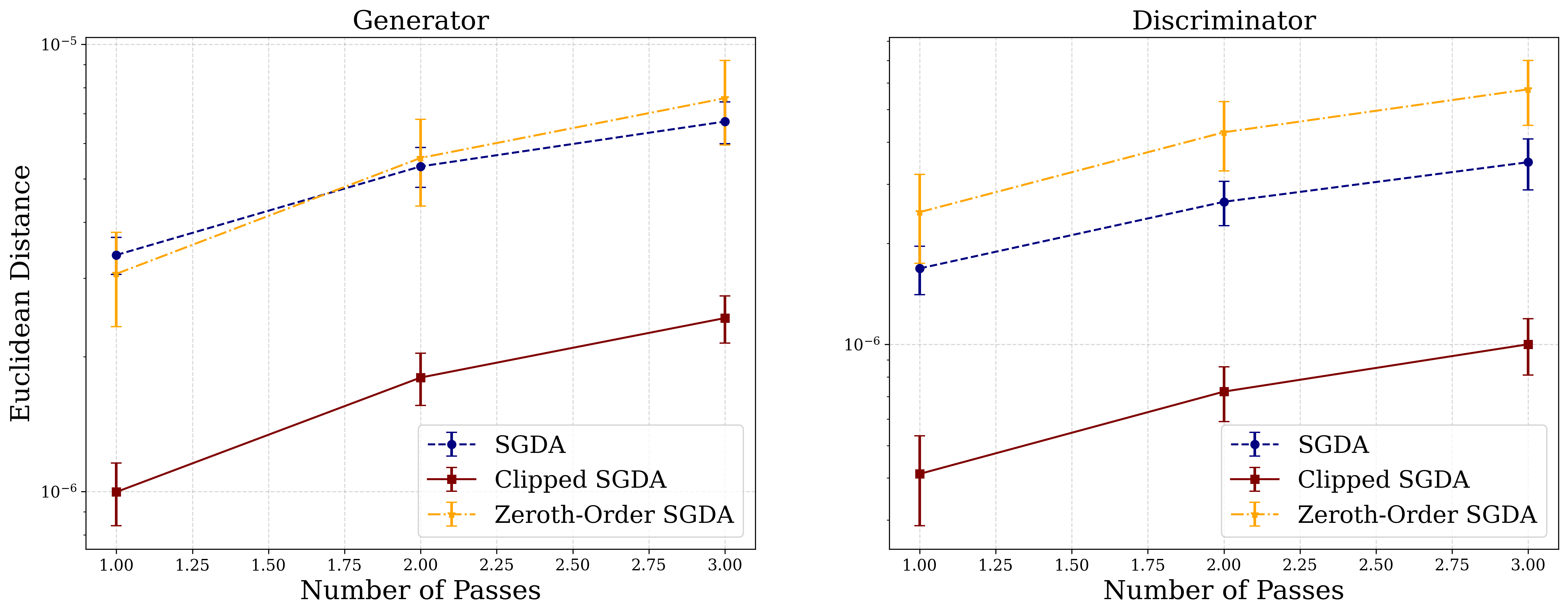}
        \caption{CIFAR-10}
    \end{subfigure}

    \vspace{3pt} 

    \begin{subfigure}{\textwidth}
        \centering
        \includegraphics[width=0.8\linewidth]{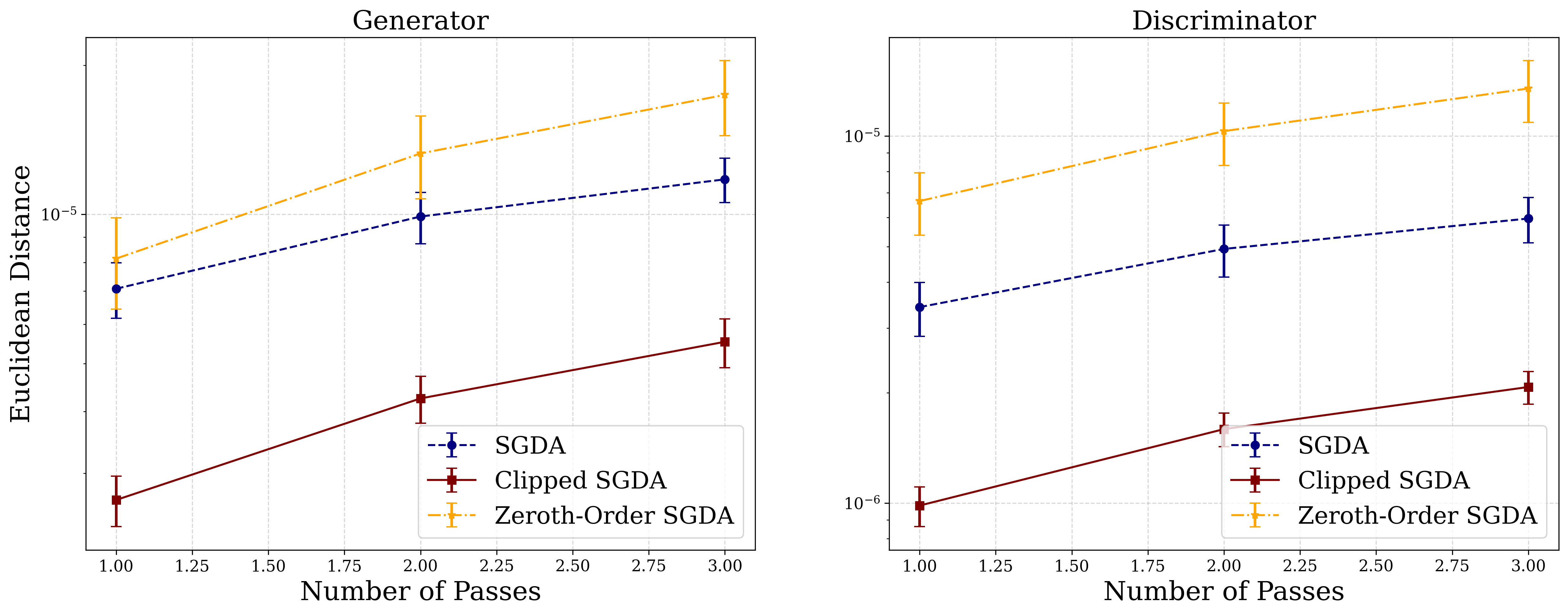}
        \caption{CelebA}
    \end{subfigure}

    \caption{$\Delta_t$ versus \# of passes. Left: generator, right: discriminator. $\eta=0.0002$, $K=300$, $\mu_1=\mu_2=0.0001,\tau_1=\tau_2=1$.}
    \label{fig:stability-gan}
\end{figure}

We further study the parameter sensitivity to the smoothing parameter $\mu$ and clipping threshold $\tau$.
Figure~\ref{fig:mu-param-sensitivity} shows that the stability of Zeroth-order SGDA is independent of the choice of
$\mu$, and that the generalization error of it matches that of SGDA when $\mu\lesssim \frac{\sqrt{\ebb[G_S^2]}}{L\sqrt{n}}$, verifying
our theoretical findings in Remark~\ref{rem:para_zero} and Remark~\ref{remark:zero-risk}.
Figure~\ref{fig:clip-param-sensitivity} shows that the stability and generalization error of Clipped SGDA match that of SGDA when $\tau_1,\tau_2 \gtrsim G T^{\frac{1}{4}}$ and $\eta \lesssim 1/\sqrt{T}$, verifying our theoretical findings in Remark~\ref{rem:clip-stab} and Remark~\ref{remark:clip-risk}.
\begin{figure}[!ht]
    \centering
    \begin{minipage}{0.48\textwidth}
        \centering
        \includegraphics[width=\linewidth]{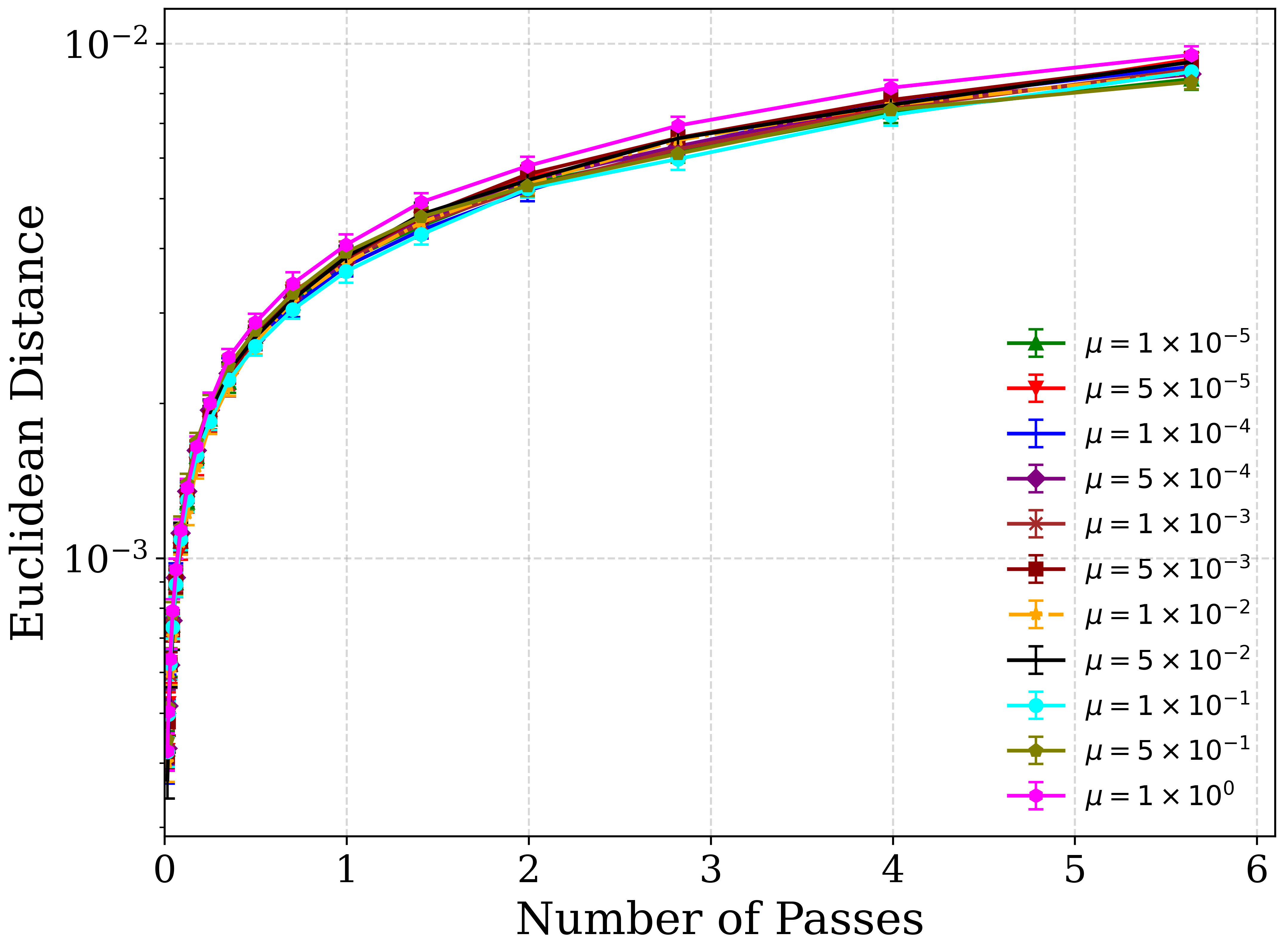}
        \vspace{3pt}
        \scriptsize (a) Euclidean distance $\Delta_t$
        \label{fig:mu-param-stab}
    \end{minipage}
    \hfill
    \begin{minipage}{0.48\textwidth}
        \centering
        \includegraphics[width=\linewidth]{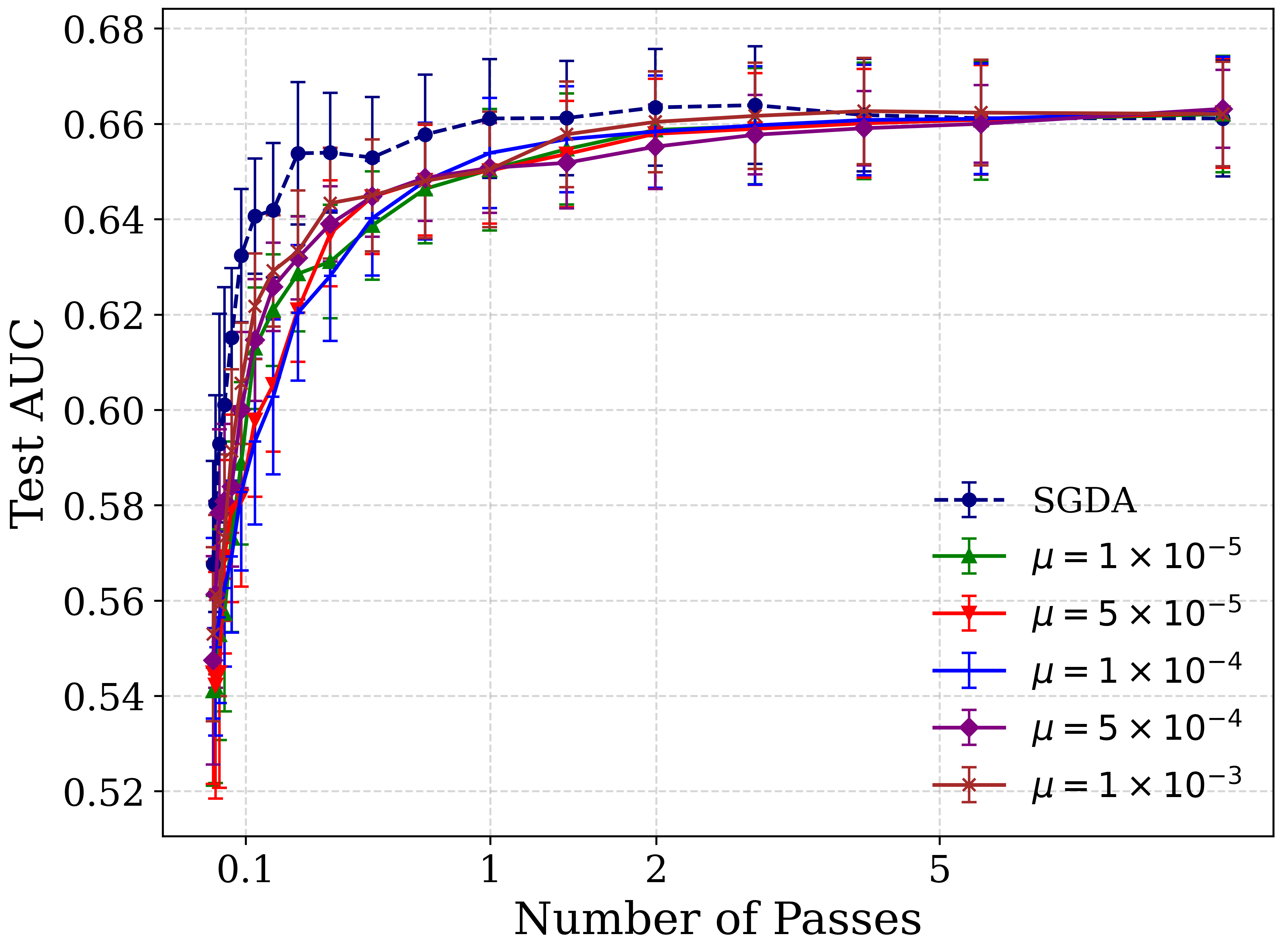}
        \vspace{3pt}
        \scriptsize (b) Test AUC
        \label{fig:mu-param-auc}
    \end{minipage}
    \caption{Parameter sensitivity of Zeroth-order SGDA w.r.t.\ the smoothing parameter $\mu$ on \texttt{svmguide3}:
    (a) $\Delta_t$ versus \# of passes; (b) Test AUC versus \# of passes. $\eta=0.01$.}
    \label{fig:mu-param-sensitivity}
\end{figure}

\begin{figure}[!ht]
    \centering
    \begin{minipage}{0.48\textwidth}
        \centering
        \includegraphics[width=\linewidth]{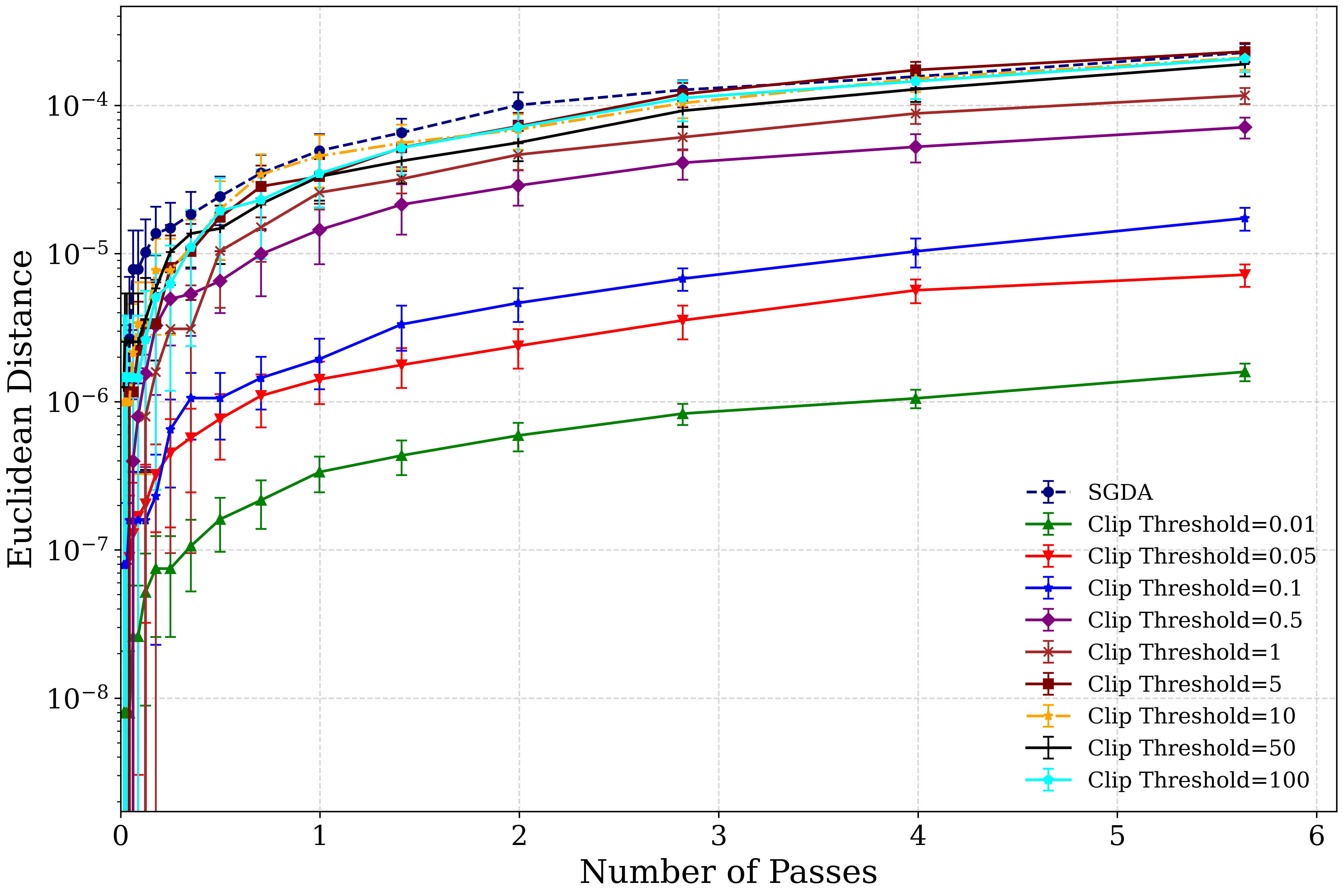}
        \vspace{3pt}
        \scriptsize (a) Euclidean distance $\Delta_t$
        \label{fig:clip-param-stab}
    \end{minipage}
    \hfill
    \begin{minipage}{0.48\textwidth}
        \centering
        \includegraphics[width=\linewidth]{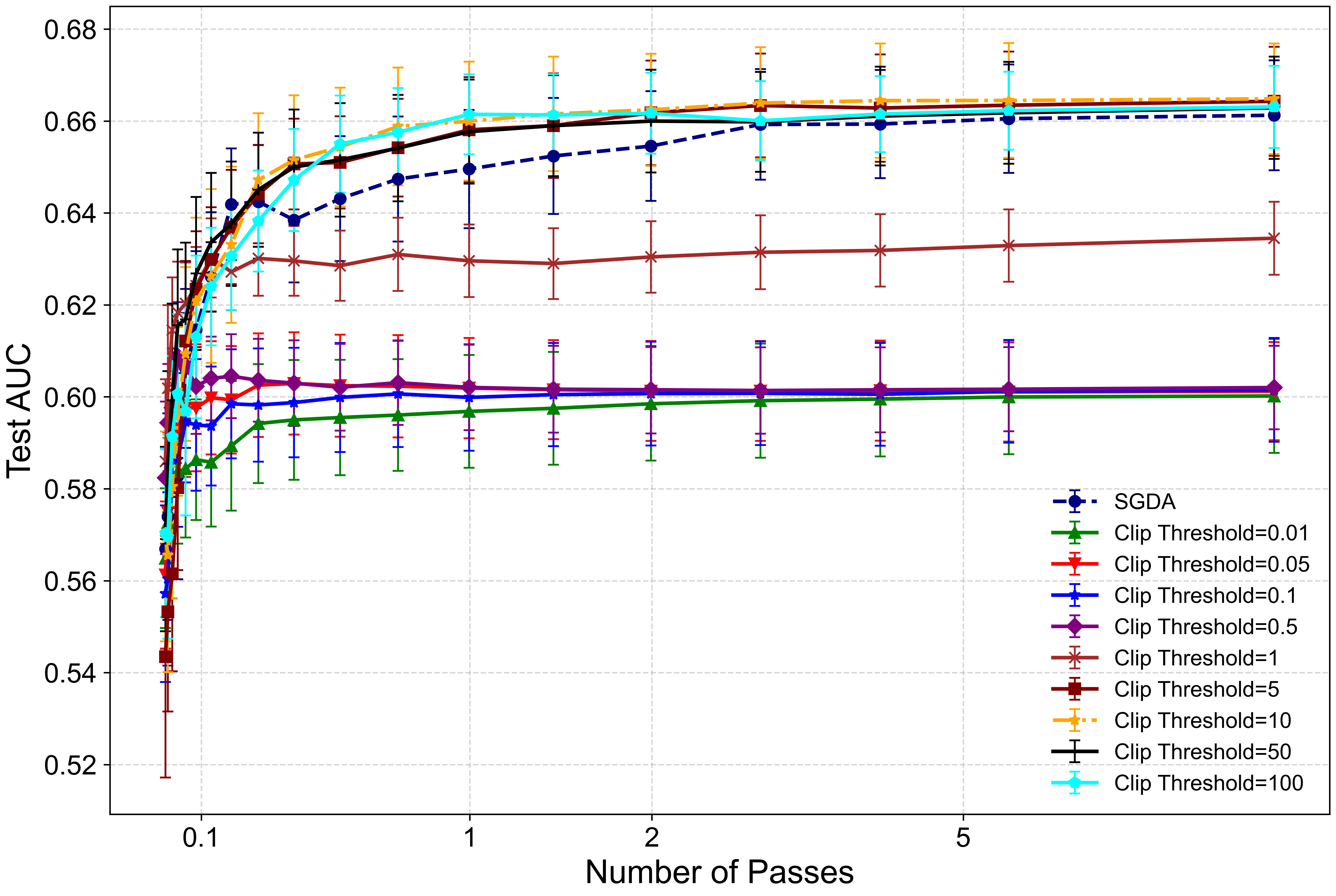}
        \vspace{3pt}
        \scriptsize (b) Test AUC
        \label{fig:clip-param-auc}
    \end{minipage}
    \caption{Parameter sensitivity of Clipped-SGDA w.r.t.\ the clipping threshold $\tau$ on \texttt{svmguide3}:
    (a) $\Delta_t$ versus \# of passes; (b) Test AUC versus \# of passes. $\eta=0.01$.}
    \label{fig:clip-param-sensitivity}
\end{figure}

\newpage
\section{Proof}
\subsection{Proof on Connecting Stability and Generalization}
In this section, we present the proof of our results on connecting stability and generalization.
\begin{proof}[Proof of Theorem~\ref{thm:stab-gen}]
Let $S=\{z_1,\ldots,z_n\}$ and $S'=\{z_1',\ldots,z_n'\}$ be two datasets drawn from the same distribution. For any $i\in[n]$, define $S^{(i)}=\{z_1,\ldots,z_{i-1},z_i',z_{i+1},\ldots,z_n\}$.
For any function $g,\widetilde{G}$, we have the basic inequalities
\begin{equation}\label{sup-inf}
  \begin{split}
      & \sup_{\bw}g(\bw)-\sup_{\bw}\widetilde{G}(\bw)\leq \sup_{\bw}\big(g(\bw)-\widetilde{G}(\bw)\big)\\
      & \inf_{\bw}g(\bw)-\inf_{\bw}\widetilde{G}(\bw)\leq \sup_{\bw}\big(g(\bw)-\widetilde{G}(\bw)\big).
  \end{split}
\end{equation}
It follows from \eqref{sup-inf} that
\begin{multline}\label{gen-decomp}
  \triangle^w(A_{\bw}(S),A_{\bv}(S))-\triangle^w_S(A_{\bw}(S),A_{\bv}(S)) \leq \sup_{\bv'\in\vcal}\ebb[F(A_{\bw}(S),\bv')-F_S(A_{\bw}(S),\bv')]\\+
  \sup_{\bw'\in\wcal}\ebb[F_S(\bw',A_{\bv}(S))-F(\bw',A_{\bv}(S))].
\end{multline}
According to the symmetry between $z_i$ and $z_i'$ we know
\begin{align*}
  \ebb[F(A_{\bw}(S),\bv')-F_S(A_{\bw}(S),\bv')] & =  \frac{1}{n}\sum_{i=1}^{n} \ebb[F(A_{\bw}(S^{(i)}),\bv')]-\ebb[F_S(A_{\bw}(S),\bv')]\\
  &=  \frac{1}{n}\sum_{i=1}^{n} \ebb\big[f(A_{\bw}(S^{(i)}),\bv';z_i)-f(A_{\bw}(S),\bv';z_i)\big],
\end{align*}
where the second identity holds since $z_i$ is not used to train $A_{\bw}(S^{(i)})$.
In a similar way, we can prove
\[
\ebb[F_S(\bw',A_{\bv}(S))-F(\bw',A_{\bv}(S))]=
\frac{1}{n}\sum_{i=1}^{n}\big[f(\bw',A_{\bv}(S^{(i)});z_i)-f(\bw',A_{\bv}(S);z_i)\big].
\]
As a combination of the above three inequalities we get
\begin{multline}\label{gen-decomp-26}
    \triangle^w(A_{\bw}(S),A_{\bv}(S))-\triangle^w_S(A_{\bw}(S),A_{\bv}(S)) \leq  \sup_{\bv'\in\vcal}\Big[\frac{1}{n}\sum_{i=1}^{n} \ebb\big[f(A_{\bw}(S^{(i)}),\bv';z_i)-f(A_{\bw}(S),\bv';z_i)\big]\Big]\\
    +\sup_{\bw'\in\wcal}\Big[\frac{1}{n}\sum_{i=1}^{n}\ebb\big[f(\bw',A_{\bv}(S^{(i)});z_i)-f(\bw',A_{\bv}(S);z_i)\big]\Big].
\end{multline}
According to Assumption \ref{ass:smooth}, we know
\begin{multline*}
f(A_{\mathbf{w}}(S^{(i)}), \bv' ; z_i)-f(A_{\mathbf{w}}(S), \bv' ; z_i)
\leq \bigg\langle\begin{pmatrix}A_{\mathbf{w}}(S^{(i)})-A_{\mathbf{w}}(S)\\
\bv'-\bv'\end{pmatrix},\begin{pmatrix}\nabla_{\mathbf{w}} f(A_{\mathbf{w}}(S), \bv' ; z_i)\\
\nabla_{\mathbf{v}} f(A_{\mathbf{w}}(S), \bv' ; z_i)\end{pmatrix}\bigg\rangle+ \\
\frac{L}{2}\bigg\|\begin{pmatrix}A_{\mathbf{w}}(S^{(i)})-A_{\mathbf{w}}(S)\\
\bv'-\bv'\end{pmatrix}\bigg\|_2^2
= \big\langle A_{\mathbf{w}}(S^{(i)})-A_{\mathbf{w}}(S), \nabla_{\mathbf{w}} f(A_{\mathbf{w}}(S), \bv' ; z_i) \big\rangle + \frac{L}{2}\big\|A_{\mathbf{w}}(S^{(i)})-A_{\mathbf{w}}(S)\big\|_2^2.
\end{multline*}
According to Schwartz’s inequality we know for any $\gamma>0$,
\begin{multline*}
\big\langle A_{\mathbf{w}}(S^{(i)})-A_{\mathbf{w}}(S), \nabla_{\mathbf{w}} f(A_{\mathbf{w}}(S), \bv' ; z_i) \big\rangle
\leq \big\|A_{\mathbf{w}}(S^{(i)})-A_{\mathbf{w}}(S)\big\|_2\big\|\nabla_{\mathbf{w}} f(A_{\mathbf{w}}(S), \bv' ; z_i)\big\|_2\\
\leq \frac{\gamma}{2}\big\|A_{\mathbf{w}}(S^{(i)})-A_{\mathbf{w}}(S)\big\|_2^2 + \frac{1}{2\gamma}\big\|\nabla_{\mathbf{w}} f(A_{\mathbf{w}}(S), \bv' ; z_i)\big\|_2^2.
\end{multline*}
Therefore,
\begin{multline*}
\ebb[f(A_{\mathbf{w}}(S^{(i)}), \bv' ; z_i)-f(A_{\mathbf{w}}(S), \bv' ; z_i)]
\leq \frac{L+\gamma}{2}\ebb\big[\big\|A_{\mathbf{w}}(S^{(i)})-A_{\mathbf{w}}(S)\big\|_2^2\big] \\
+ \frac{1}{2\gamma} \ebb\big[\big\|\nabla_{\mathbf{w}} f(A_{\mathbf{w}}(S), \bv' ; z_i)\big\|_2^2\big].
\end{multline*}
Similarly, we have
\begin{multline*}
\ebb\big[f(\bw',A_{\bv}(S^{(i)});z_i)-f(\bw',A_{\bv}(S);z_i)\big]\Big] \leq \frac{L+\gamma}{2}\ebb\big[\big\|A_{\mathbf{v}}(S^{(i)})-A_{\mathbf{v}}(S)\big\|_2^2\big]  \\
+ \frac{1}{2\gamma} \ebb\big[\big\|\nabla_{\mathbf{v}} f(\bw', A_{\mathbf{v}}(S) ; z_i)\big\|_2^2\big].
\end{multline*}
Combining the above two inequalities with Eq.~\eqref{gen-decomp-26}, we have
\begin{multline*}
\triangle^w(A_{\bw}(S),A_{\bv}(S))-\triangle^w_S(A_{\bw}(S),A_{\bv}(S)) \leq  \frac{L+\gamma}{2n}\sum_{i=1}^{n} \ebb\bigg[\bigg\|\begin{pmatrix}A_{\mathbf{w}}(S^{(i)})-A_{\mathbf{w}}(S)\\
A_{\mathbf{v}}(S^{(i)})-A_{\mathbf{v}}(S)
\end{pmatrix}\bigg\|_2^2 \bigg]\\
+ \frac{1}{2\gamma}\sup_{\bv'\in \vcal} \ebb\Big[\frac{1}{n}\sum_{i=1}^n \big\|\nabla_{\mathbf{w}} f(A_{\mathbf{w}}(S), \bv' ; z_i)\big\|_2^2\Big] + \frac{1}{2\gamma} \sup_{\bw'\in \wcal}\ebb\Big[\frac{1}{n}\sum_{i=1}^n\big\|\nabla_{\mathbf{v}} f(\bw', A_{\mathbf{v}}(S) ; z_i)\big\|_2^2\Big].
\end{multline*}
Using the fact that (Assumption~\ref{ass:lipschitz})
\begin{align*}
& \frac{1}{n}\sum_{i=1}^n \big\|\nabla_{\mathbf{w}} f(A_{\mathbf{w}}(S), \bv' ; z_i)\big\|_2^2 = \ebb_r[\|\nabla_{\bw} f(A_{\bw}(S),\bv';S_r)\|_2^2] \leq G_{S,A}^2,\\
& \frac{1}{n}\sum_{i=1}^n\big\|\nabla_{\mathbf{v}} f(\bw', A_{\mathbf{v}}(S) ; z_i)\big\|_2^2 = \ebb_r[\|\nabla_{\bv} f(\bw',A_{\bv}(S);S_r)\|_2^2] \leq G_{S,A}^2,
\end{align*}
we further get
\[
\triangle^w(A_{\bw}(S),A_{\bv}(S))-\triangle^w_S(A_{\bw}(S),A_{\bv}(S)) \leq  \frac{L+\gamma}{2n}\sum_{i=1}^{n} \ebb\bigg[\bigg\|\begin{pmatrix}A_{\mathbf{w}}(S^{(i)})-A_{\mathbf{w}}(S)\\
A_{\mathbf{v}}(S^{(i)})-A_{\mathbf{v}}(S)
\end{pmatrix}\bigg\|_2^2 \bigg]+\frac{\ebb[G_{S,A}^2]}{\gamma}.
\]
Since $A$ is $\epsilon$-argument-stable in expectation, we know
\[
    \triangle^w(A_{\bw}(S),A_{\bv}(S))-\triangle^w_S(A_{\bw}(S),A_{\bv}(S))\leq \frac{L+\gamma}{2}\epsilon^2 + \frac{\ebb[G_{S,A}^2]}{\gamma}.
\]
If we choose $\gamma = \sqrt{2\ebb[G_{S,A}^2]}/\epsilon$, then
\[
        \triangle^w(A_{\bw}(S),A_{\bv}(S))-\triangle^w_S(A_{\bw}(S),A_{\bv}(S))\leq L\epsilon^2/2 + \sqrt{2\ebb[G_{S,A}^2]}\epsilon.
\]
This completes the proof.
\end{proof}

\subsection{Proofs on General Stability Bounds}
In this section, we present the proof for our general stability bounds of Biased-SGDA under the
convex-concave case.

To prove Theorem~\ref{lem:stab-gda}, we require the following standard lemma~\cite{rockafellar1976monotone}, which shows the monotonicity of the gradient map for convex-concave functions.
\begin{lemma}[\cite{rockafellar1976monotone}]\label{lem:non-expansive-key}
  Let $f$ be a convex-concave function. Then
  \begin{equation}\label{non-expansive-key}
  \bigg\langle\begin{pmatrix}
    \bw-\bw' \\
    \bv-\bv'
  \end{pmatrix},\begin{pmatrix}
                  \nabla_{\bw}f(\bw,\bv)-\nabla_{\bw}f(\bw',\bv') \\
                  \nabla_{\bv}f(\bw',\bv')-\nabla_{\bv}f(\bw,\bv)
                \end{pmatrix}\bigg\rangle \geq 0.
  \end{equation}
\end{lemma}

\begin{proof}[Proof of Theorem~\ref{lem:stab-gda}]
Let $S=\{z_1,\ldots,z_n\}$ and $S^{(i)}=\{z_1,\ldots,z_{i-1},z_i',z_{i+1},\ldots,z_n\}$. Let $\{\bw_t,\bv_t\}$ and $\{\bw_t^{(i)},\bv_t^{(i)}\}$ be the sequence produced by \eqref{gda} w.r.t. $S$ and $S^{(i)}$, respectively. Note that the projection step is nonexpansive, according to Eq.~\eqref{gda}, Eq.~\eqref{gda_prime}, we know
\begin{align}
  \Bigg\|\begin{pmatrix}
           \bw_{t+1}-\bw_{t+1}^{(i)} \notag\\
           \bv_{t+1}-\bv_{t+1}^{(i)}
         \end{pmatrix}\Bigg\|_2^2
        & \leq \Bigg\|\begin{pmatrix}
                \bw_t-\bw_t^{(i)} \\
                \bv_t-\bv_t^{(i)}
                \end{pmatrix}\Bigg\|_2^2
  + \eta^2\Bigg\|\begin{pmatrix}
  g_{\bw}(\bw_t^{(i)},\bv_t^{(i)};z_{i_t}^{(i)}) - g_{\bw}(\bw_t,\bv_t;z_{i_t})\\
  g_{\bv}(\bw_t,\bv_t;z_{i_t}) - g_{\bv}(\bw_t^{(i)},\bv_t^{(i)};z^{(i)}_{i_t})
  \end{pmatrix}\Bigg\|_2^2 \notag\\
  &-2\eta\Bigg\langle\begin{pmatrix}
                      \bw_t-\bw_t^{(i)} \\
                      \bv_t-\bv_t^{(i)}
                    \end{pmatrix},\begin{pmatrix}
                        g_{\bw}(\bw_t,\bv_t;z_{i_t})-g_{\bw}(\bw_t^{(i)},\bv_t^{(i)};z_{i_t}^{(i)}) \\
                        g_{\bv}(\bw_t^{(i)},\bv_t^{(i)};z^{(i)}_{i_t})-g_{\bv}(\bw_t,\bv_t;z_{i_t})
                        \end{pmatrix}\Bigg\rangle.
                 \label{stab-gda-1}
\end{align}
Recall the definition of the bias as in Eq.~\eqref{bias}. Taking a conditional expectation implies
\begin{align}\label{stab-gda-2}
  & \ebb_t\Bigg[\Bigg\|\begin{pmatrix}
           \bw_{t+1}-\bw_{t+1}^{(i)} \notag\\
           \bv_{t+1}-\bv_{t+1}^{(i)}
         \end{pmatrix}\Bigg\|_2^2 \Bigg]\\
         & \leq \Bigg\|\begin{pmatrix}
                \bw_t-\bw_t^{(i)} \\
                \bv_t-\bv_t^{(i)}
                \end{pmatrix}\Bigg\|_2^2
+
  \eta^2 \ebb_t\Bigg[\Bigg\|\begin{pmatrix}
  g_{\bw}(\bw_t^{(i)},\bv_t^{(i)};z_{i_t}^{(i)}) - g_{\bw}(\bw_t,\bv_t;z_{i_t})\\
  g_{\bv}(\bw_t,\bv_t;z_{i_t}) - g_{\bv}(\bw_t^{(i)},\bv_t^{(i)};z^{(i)}_{i_t})
  \end{pmatrix}\Bigg\|_2^2 \Bigg] \notag\\
  & -2\eta\Bigg\langle\begin{pmatrix}
                      \bw_t-\bw_t^{(i)} \\
                      \bv_t-\bv_t^{(i)}
                    \end{pmatrix},\begin{pmatrix}
                        \nabla_{\bw} \widetilde{F}_{1,S}(\bw_t,\bv_t) + b_{t,\bw} - \nabla_{\bw} \widetilde{F}_{1,S^{(i)}}(\bw^{(i)}_t,\bv^{(i)}_t) - b_{t,\bw}^{(i)}\\
                        \nabla_{\bv} \widetilde{F}_{2,S^{(i)}}(\bw^{(i)}_t,\bv^{(i)}_t) + b_{t,\bv}^{(i)} - \nabla_{\bv} \widetilde{F}_{2,S}(\bw_t,\bv_t) - b_{t,\bv}
                        \end{pmatrix}\Bigg\rangle.
\end{align}
We know
\begin{align*}
& -2\eta \Bigg\langle\begin{pmatrix}
                      \bw_t-\bw_t^{(i)} \\
                      \bv_t-\bv_t^{(i)}
                    \end{pmatrix},\begin{pmatrix}
                        \nabla_{\bw} \widetilde{F}_{1,S}(\bw_t,\bv_t) - \nabla_{\bw} \widetilde{F}_{1,S^{(i)}}(\bw^{(i)}_t,\bv^{(i)}_t)\\
                        \nabla_{\bv} \widetilde{F}_{2,S^{(i)}}(\bw^{(i)}_t,\bv^{(i)}_t)- \nabla_{\bv} \widetilde{F}_{2,S}(\bw_t,\bv_t)
                        \end{pmatrix}\Bigg\rangle \notag\\
& = -\frac{2\eta}{n} \sum_{j:j\neq i} \Bigg\langle\begin{pmatrix}
                      \bw_t-\bw_t^{(i)} \\
                      \bv_t-\bv_t^{(i)}
                    \end{pmatrix},\begin{pmatrix}
                        \nabla_{\bw} \tilde{f}_1(\bw_t,\bv_t;z_j) - \nabla_{\bw} \tilde{f}_1(\bw^{(i)}_t,\bv^{(i)}_t;z_j)\\
                        \nabla_{\bv} \tilde{f}_2(\bw^{(i)}_t,\bv^{(i)}_t;z_j)- \nabla_{\bv} \tilde{f}_2(\bw_t,\bv_t;z_j)
                        \end{pmatrix}\Bigg\rangle \notag\\
& - \frac{2\eta}{n}
\Bigg\langle\begin{pmatrix}
    \bw_t-\bw_t^{(i)} \\
    \bv_t-\bv_t^{(i)}
    \end{pmatrix},\begin{pmatrix}
    \nabla_{\bw} \tilde{f}_1(\bw_t,\bv_t;z_i) - \nabla_{\bw} \tilde{f}_1(\bw^{(i)}_t,\bv^{(i)}_t;z_i')\\
    \nabla_{\bv} \tilde{f}_2(\bw^{(i)}_t,\bv^{(i)}_t;z_i')- \nabla_{\bv} \tilde{f}_2(\bw_t,\bv_t;z_i)
    \end{pmatrix}\Bigg\rangle.
\end{align*}
Noting that $\tilde{f}_1,\tilde{f}_2$ are convex w.r.t. $\bw$ and concave w.r.t $\bv$, it then follows from Lemma \ref{lem:non-expansive-key} and Schwartz’s inequality that
\begin{align}\label{inner_gradient}
& -2\eta \Bigg\langle\begin{pmatrix}
                      \bw_t-\bw_t^{(i)} \\
                      \bv_t-\bv_t^{(i)}
                    \end{pmatrix},\begin{pmatrix}
                        \nabla_{\bw} \widetilde{F}_{1,S}(\bw_t,\bv_t) - \nabla_{\bw} \widetilde{F}_{1,S^{(i)}}(\bw^{(i)}_t,\bv^{(i)}_t)\\
                        \nabla_{\bv} \widetilde{F}_{2,S^{(i)}}(\bw^{(i)}_t,\bv^{(i)}_t)- \nabla_{\bv} \widetilde{F}_{2,S}(\bw_t,\bv_t)
                        \end{pmatrix}\Bigg\rangle \notag\\
& \leq \frac{2\eta}{n}
\Bigg\|\begin{pmatrix}
    \bw_t-\bw_t^{(i)} \\
    \bv_t-\bv_t^{(i)}
    \end{pmatrix}\Bigg\|_2
    \Bigg\|\begin{pmatrix}
    \nabla_{\bw} \tilde{f}_1(\bw_t,\bv_t;z_i) - \nabla_{\bw} \tilde{f}_1(\bw^{(i)}_t,\bv^{(i)}_t;z_i')\\
    \nabla_{\bv} \tilde{f}_2(\bw^{(i)}_t,\bv^{(i)}_t;z_i')- \nabla_{\bv} \tilde{f}_2(\bw_t,\bv_t;z_i)
    \end{pmatrix}\Bigg\|_2.
\end{align}
Using Schwarz’s inequality, we also know
\begin{align}\label{inner_bias}
&-2\eta\Bigg\langle\begin{pmatrix}
                      \bw_t-\bw_t^{(i)} \\
                      \bv_t-\bv_t^{(i)}
                    \end{pmatrix},\begin{pmatrix}
                    b_{t,\bw} - b_{t,\bw}^{(i)}\\
                     b_{t,\bv}^{(i)} - b_{t,\bv}
                        \end{pmatrix}\Bigg\rangle \leq 2\eta
\Bigg\|\begin{pmatrix}
    \bw_t-\bw_t^{(i)} \\
    \bv_t-\bv_t^{(i)}
    \end{pmatrix}\Bigg\|_2
    \Bigg\|\begin{pmatrix}
                    b_{t,\bw} - b_{t,\bw}^{(i)}\\
                     b_{t,\bv}^{(i)} - b_{t,\bv}
                        \end{pmatrix}\Bigg\|_2.
\end{align}
Plugging Eq.~\eqref{inner_gradient} and Eq.~\eqref{inner_bias} into Eq.~\eqref{stab-gda-2}, we know
\begin{align*}
  &\ebb_t\Bigg[\Bigg\|\begin{pmatrix}
           \bw_{t+1}-\bw_{t+1}^{(i)} \notag\\
           \bv_{t+1}-\bv_{t+1}^{(i)}
         \end{pmatrix}\Bigg\|_2^2 \Bigg]\\
         & \leq \Bigg\|\begin{pmatrix}
                \bw_t-\bw_t^{(i)} \\
                \bv_t-\bv_t^{(i)}
                \end{pmatrix}\Bigg\|_2^2
+\eta^2 \ebb_t\Bigg[\Bigg\|\begin{pmatrix}
  g_{\bw}(\bw_t^{(i)},\bv_t^{(i)};z_{i_t}^{(i)}) - g_{\bw}(\bw_t,\bv_t;z_{i_t})\\
  g_{\bv}(\bw_t,\bv_t;z_{i_t}) - g_{\bv}(\bw_t^{(i)},\bv_t^{(i)};z^{(i)}_{i_t})
  \end{pmatrix}\Bigg\|_2^2 \Bigg]\\
  & +\frac{2\eta}{n}
\Bigg\|\begin{pmatrix}
    \bw_t-\bw_t^{(i)} \\
    \bv_t-\bv_t^{(i)}
    \end{pmatrix}\Bigg\|_2
    \Bigg\|\begin{pmatrix}
    \nabla_{\bw} \tilde{f}_1(\bw_t,\bv_t;z_i) - \nabla_{\bw} \tilde{f}_1(\bw^{(i)}_t,\bv^{(i)}_t;z_i')\\
    \nabla_{\bv} \tilde{f}_2(\bw^{(i)}_t,\bv^{(i)}_t;z_i')- \nabla_{\bv} \tilde{f}_2(\bw_t,\bv_t;z_i)
    \end{pmatrix}\Bigg\|_2 \\
& +2\eta
\Bigg\|\begin{pmatrix}
    \bw_t-\bw_t^{(i)} \\
    \bv_t-\bv_t^{(i)}
    \end{pmatrix}\Bigg\|_2
    \Bigg\|\begin{pmatrix}
                    b_{t,\bw} - b_{t,\bw}^{(i)}\\
                     b_{t,\bv}^{(i)} - b_{t,\bv}
                        \end{pmatrix}\Bigg\|_2.
\end{align*}
Taking an average of $i \in [n]$, and using the fact that $\sum_{i=1}^n a_i b_i \leq (\sum_{i=1}^n a_i^2)^{\frac{1}{2}}(\sum_{i=1}^n b_i^2)^{\frac{1}{2}}$, we get
\begin{align*}
  & \ebb_t\Bigg[\frac{1}{n}\sum_{i=1}^n \Bigg\|\begin{pmatrix}
           \bw_{t+1}-\bw_{t+1}^{(i)} \notag\\
           \bv_{t+1}-\bv_{t+1}^{(i)}
         \end{pmatrix}\Bigg\|_2^2 \Bigg]\\
    & \leq \frac{1}{n}\sum_{i=1}^n\Bigg\|\begin{pmatrix}
                \bw_t-\bw_t^{(i)} \\
                \bv_t-\bv_t^{(i)}
                \end{pmatrix}\Bigg\|_2^2
+\eta^2 \ebb_t\Bigg[\frac{1}{n}\sum_{i=1}^n\Bigg\|\begin{pmatrix}
  g_{\bw}(\bw_t^{(i)},\bv_t^{(i)};z_{i_t}^{(i)}) - g_{\bw}(\bw_t,\bv_t;z_{i_t})\\
  g_{\bv}(\bw_t,\bv_t;z_{i_t}) - g_{\bv}(\bw_t^{(i)},\bv_t^{(i)};z^{(i)}_{i_t})
  \end{pmatrix}\Bigg\|_2^2 \Bigg] \notag\\
&+\frac{2\eta}{n}
\Bigg(\frac{1}{n}\sum_{i=1}^n\Bigg\|\begin{pmatrix}
    \bw_t-\bw_t^{(i)} \\
    \bv_t-\bv_t^{(i)}
    \end{pmatrix}\Bigg\|_2^2\Bigg)^{\frac{1}{2}}
    \Bigg(\frac{1}{n}\sum_{i=1}^n\Bigg\|\begin{pmatrix}
    \nabla_{\bw} \tilde{f}_1(\bw_t,\bv_t;z_i) - \nabla_{\bw} \tilde{f}_1(\bw^{(i)}_t,\bv^{(i)}_t;z_i')\\
    \nabla_{\bv} \tilde{f}_2(\bw^{(i)}_t,\bv^{(i)}_t;z_i')- \nabla_{\bv} \tilde{f}_2(\bw_t,\bv_t;z_i)
    \end{pmatrix}\Bigg\|_2^2\Bigg)^{\frac{1}{2}} \notag\\
& +2\eta
\Bigg(\frac{1}{n}\sum_{i=1}^n\Bigg\|\begin{pmatrix}
    \bw_t-\bw_t^{(i)} \\
    \bv_t-\bv_t^{(i)}
    \end{pmatrix}\Bigg\|_2^2\Bigg)^{\frac{1}{2}}
    \Bigg(\frac{1}{n}\sum_{i=1}^n \Bigg\|\begin{pmatrix}
                    b_{t,\bw} - b_{t,\bw}^{(i)}\\
                     b_{t,\bv}^{(i)} - b_{t,\bv}
                        \end{pmatrix}\Bigg\|_2^2\Bigg)^{\frac{1}{2}}.
\end{align*}
Introduce
\begin{equation*}
\bar{\Delta}_{t}=\Bigg(\frac{1}{n}\sum_{i=1}^n\ebb\Bigg[\Bigg\|\begin{pmatrix}
                \bw_t-\bw_t^{(i)} \\
                \bv_t-\bv_t^{(i)}
                \end{pmatrix}\Bigg\|_2^2\Bigg]\Bigg)^{\frac{1}{2}},\quad\forall t\in[T].
\end{equation*}
Taking the expectation to the above inequality and using the inequality $\ebb[XY] \leq (\ebb[X^2])^{\frac{1}{2}} (\ebb[Y^2])^{\frac{1}{2}}$, we have
\begin{align*}
  \bar{\Delta}_{t+1}^2 & \leq \bar{\Delta}_{t}^2
+\eta^2 \ebb\Bigg[\frac{1}{n}\sum_{i=1}^n\Bigg\|\begin{pmatrix}
  g_{\bw}(\bw_t^{(i)},\bv_t^{(i)};z_{i_t}^{(i)}) - g_{\bw}(\bw_t,\bv_t;z_{i_t})\\
  g_{\bv}(\bw_t,\bv_t;z_{i_t}) - g_{\bv}(\bw_t^{(i)},\bv_t^{(i)};z^{(i)}_{i_t})
  \end{pmatrix}\Bigg\|_2^2 \Bigg] \notag\\
&+\frac{2\eta}{n}
\bar{\Delta}_{t}
    \Bigg(\ebb\Bigg[\frac{1}{n}\sum_{i=1}^n\Bigg\|\begin{pmatrix}
    \nabla_{\bw} \tilde{f}_1(\bw_t,\bv_t;z_i) - \nabla_{\bw} \tilde{f}_1(\bw^{(i)}_t,\bv^{(i)}_t;z_i')\\
    \nabla_{\bv} \tilde{f}_2(\bw^{(i)}_t,\bv^{(i)}_t;z_i')- \nabla_{\bv} \tilde{f}_2(\bw_t,\bv_t;z_i)
    \end{pmatrix}\Bigg\|_2^2\Bigg]\Bigg)^{\frac{1}{2}} \notag\\
& +2\eta
\bar{\Delta}_{t}
    \Bigg(\ebb\Bigg[\frac{1}{n}\sum_{i=1}^n \Bigg\|\begin{pmatrix}
                    b_{t,\bw} - b_{t,\bw}^{(i)}\\
                     b_{t,\bv}^{(i)} - b_{t,\bv}
                        \end{pmatrix}\Bigg\|_2^2\Bigg]\Bigg)^{\frac{1}{2}}.
\end{align*}
Using conditions as in Eq.~\eqref{simple-condition-a}, Eq.~\eqref{simple-condition-c} and Eq.~\eqref{simple-condition-d}, this further yields
\begin{align*}
  \bar{\Delta}_{t+1}^2
\leq (1+A\eta^2+2\bar{A}^{\frac{1}{2}}\eta)\bar{\Delta}_{t}^2
    +2\eta \big(\sqrt{\bar{B}}+\frac{\widetilde{G}}{n}\big)
\bar{\Delta}_{t}
  + \eta^2 B.
\end{align*}
Applying this inequality recursively and using the fact that $\bar{\Delta}_{1}=0$ imply that
\begin{align*}
  \bar{\Delta}_{t+1}^2  & \leq \eta^2 B \sum_{k=1}^{t}(1+A\eta^2+2\bar{A}^{\frac{1}{2}}\eta)^{t-k}
  + 2\eta \sum_{k=1}^{t}
\big(\sqrt{\bar{B}}+\frac{\widetilde{G}}{n}\big) \bar{\Delta}_{k}
  (1+A\eta^2+2\bar{A}^{\frac{1}{2}}\eta)^{t-k}
  \\
  & = \eta^2 B\frac{\big(1+A\eta^2+2\bar{A}^{\frac{1}{2}}\eta\big)^{t}-1}{A\eta^2+2\bar{A}^{\frac{1}{2}}\eta}
  + 2\eta \sum_{k=1}^{t} \big(\sqrt{\bar{B}}+\frac{\widetilde{G}}{n}\big)
\bar{\Delta}_{k}
  (1+A\eta^2+2\bar{A}^{\frac{1}{2}}\eta)^{t-k}.
\end{align*}
Since the above inequality applies for any $t\in[T]$, if we introduce
$\bar{\Delta}=\max_{t\leq T}\bar{\Delta}_{t}$, we further get
\begin{align*}
  \bar{\Delta}^2
  \leq \eta^2 B\frac{\big(1+A\eta^2+2\bar{A}^{\frac{1}{2}}\eta\big)^{T}-1}{A\eta^2+2\bar{A}^{\frac{1}{2}}\eta}
  + 2\eta \big(\sqrt{\bar{B}}+\frac{\widetilde{G}}{n}\big) \frac{\big(1+A\eta^2+2\bar{A}^{\frac{1}{2}}\eta\big)^{T}-1}{A\eta^2+2\bar{A}^{\frac{1}{2}}\eta}\bar{\Delta}.
\end{align*}
 Note the right-hand side of the above inequality is an increasing function of $A\eta^2+2\bar{A}^{\frac{1}{2}}\eta$.
By taking $\eta$ such that $A\eta^2+2\bar{A}^{\frac{1}{2}}\eta \leq \frac{1}{T}$ in the above inequality and using $(1+1/T)^T\leq e$, we get
\begin{align*}
  \bar{\Delta}^2
  \leq 2\eta^2 B T
  + 4\eta \big(\sqrt{\bar{B}}+\frac{\widetilde{G}}{n}\big)T\bar{\Delta}.
\end{align*}
Solving the above quadratic inequality  of $\bar{\Delta}$ implies that (if $x^2\leq ax+b$ for some $a,b\geq0$, then $x^2\leq a^2+2b$)
\begin{align*}
  \bar{\Delta}^2
  \leq 4B\eta^2T
  + 16\big(\sqrt{\bar{B}}+\frac{\widetilde{G}}{n}\big)^2 \eta^2 T^2.
\end{align*}
The stated bound follows directly.
\end{proof}

\subsection{Proofs on Applications\label{sec:proof-application}}
In this section, we present stability bounds for the specific Biased-SGDA algorithms for the convex-concave case. We first present a lemma that is often applied in our proof. This lemma  shows the condition Eq.~\eqref{simple-condition-d} holds with $\widetilde{G}=2\sqrt{2\ebb[G_S^2]}$ when we choose $\tilde{f}_1=\tilde{f}_2=f$.
\begin{lemma}\label{lem:gradient-norm}
Under Assumption~\ref{ass:moment},  we know
\[
\ebb\Bigg[\frac{1}{n}\sum_{i=1}^n\Bigg\|\begin{pmatrix}
  \nabla_{\bw} f(\bw_t^{(i)},\bv_t^{(i)};z_i') - \nabla_{\bw} f(\bw_t,\bv_t;z_i)\\
  \nabla_{\bv} f(\bw_t,\bv_t;z_i) - \nabla_{\bv} f(\bw_t^{(i)},\bv_t^{(i)};z_i')
  \end{pmatrix}\Bigg\|_2^2\Bigg] \leq 8\ebb[G_S^2].
\]
\end{lemma}
\begin{proof}
By the inequality that $(a+b)^2 \leq 2a^2 + 2b^2$, we know
\begin{align*}
& \Bigg\|\begin{pmatrix}
  \nabla_{\bw} f(\bw_t^{(i)},\bv_t^{(i)};z_i') - \nabla_{\bw} f(\bw_t,\bv_t;z_i)\\
  \nabla_{\bv} f(\bw_t,\bv_t;z_i) - \nabla_{\bv} f(\bw_t^{(i)},\bv_t^{(i)};z_i')
\end{pmatrix}\Bigg\|_2^2\\
& = \big\|\nabla_{\bw} f(\bw_t^{(i)},\bv_t^{(i)};z_i') - \nabla_{\bw} f(\bw_t,\bv_t;z_i)\big\|_2^2
+ \big\|\nabla_{\bv} f(\bw_t,\bv_t;z_i) - \nabla_{\bv} f(\bw_t^{(i)},\bv_t^{(i)};z_i')\big\|_2^2\\
& \leq 2\Big(\big\|\nabla_{\bw} f(\bw_t^{(i)},\bv_t^{(i)};z_i')\big\|_2^2 + \big\|\nabla_{\bw} f(\bw_t,\bv_t;z_i)\big\|_2^2
+ \big\|\nabla_{\bv} f(\bw_t,\bv_t;z_i) \big\|_2^2
+\big\|\nabla_{\bv} f(\bw_t^{(i)},\bv_t^{(i)};z_i')\big\|_2^2\Big).
\end{align*}
Taking the expectation of the above inequality, then taking an average of $i \in [n]$, we get ($r$ follows the uniform distribution over $[n]$)
\begin{align*}
& \ebb\Bigg[\frac{1}{n}\sum_{i=1}^n\Bigg\|\begin{pmatrix}
  \nabla_{\bw} f(\bw_t^{(i)},\bv_t^{(i)};z_i') - \nabla_{\bw} f(\bw_t,\bv_t;z_i)\\
  \nabla_{\bv} f(\bw_t,\bv_t;z_i) - \nabla_{\bv} f(\bw_t^{(i)},\bv_t^{(i)};z_i')
  \end{pmatrix}\Bigg\|_2^2\Bigg]\\
& \leq 4\ebb\bigg[\frac{1}{n}\sum_{i=1}^n \big\|\nabla_{\bw} f(\bw_t,\bv_t;z_i)\big\|_2^2 \bigg] + 4\ebb\bigg[\frac{1}{n}\sum_{i=1}^n \big\|\nabla_{\bv} f(\bw_t,\bv_t;z_i)\big\|_2^2 \bigg]\\
& = 4\ebb\big[\|\nabla_{\bw} f(\bw_t,\bv_t;z_r)\|_2^2 \big] + 4\ebb\big[\|\nabla_{\bv} f(\bw_t,\bv_t;z_r)\|_2^2 \big] \leq 8\ebb[G_S^2],
\end{align*}
where we have used the identity
\[\ebb[\|\nabla_{\bw} f(\bw_t^{(i)},\bv_t^{(i)};z_i')\|_2^2] \!=\! \ebb[\|\nabla_{\bw} f(\bw_t,\bv_t;z_i)\|_2^2],\ebb[\|\nabla_{\bv} f(\bw_t^{(i)},\bv_t^{(i)};z_i')\|_2^2] \!=\! \ebb[\|\nabla_{\bv} f(\bw_t,\bv_t;z_i)\|_2^2]\]
in the first step, and Assumption~\ref{ass:moment} in the last step.
\end{proof}

\subsubsection{SGDA \label{sec:proof-sgda}}
In this subsection, we give the proofs for SGDA.
We first prove Lemma~\ref{lem:conditions-sgda} to show that Eq.~\eqref{simple-condition-a}, Eq.~\eqref{simple-condition-c} hold for SGDA.

\begin{lemma}\label{lem:conditions-sgda}
Assume for all $z$, the function $(\bw,\bv)\mapsto f(\bw,\bv;z)$ is convex-concave. Let Assumption \ref{ass:smooth}, \ref{ass:moment} hold. If we choose the gradient estimator in Eq.~\eqref{grad-sgda} and $\tilde{f}_1=\tilde{f}_2 = f$, then Eq.~\eqref{simple-condition-a}, Eq.~\eqref{simple-condition-c} hold with $A=L^2,B={\frac{8\ebb[G_S^2]}{n}},\bar{A}=0,\bar{B}=0$.
\end{lemma}

\begin{proof}[Proof of Lemma~\ref{lem:conditions-sgda}]
  Since $b_t=b_t^{(i)}=0$, it is clear that Eq.~\eqref{simple-condition-c} holds with $\bar{A}=0,\bar{B}=0.$
  Furthermore, we have
\begin{align}\label{eq:gradient-bound}
& \ebb\Bigg[\frac{1}{n}\sum_{i=1}^n\Bigg\|\begin{pmatrix}
  g_{\bw}(\bw_t^{(i)},\bv_t^{(i)};z_{i_t}^{(i)}) - g_{\bw}(\bw_t,\bv_t;z_{i_t})\\
  g_{\bv}(\bw_t,\bv_t;z_{i_t}) - g_{\bv}(\bw_t^{(i)},\bv_t^{(i)};z^{(i)}_{i_t})
  \end{pmatrix}\Bigg\|_2^2 \Bigg] \notag\\
 & = \ebb\Bigg[\frac{1}{n^2}\sum_{i=1}^n\sum_{j:j\neq i}\Bigg\|\begin{pmatrix}
  \nabla_{\bw} f(\bw_t^{(i)},\bv_t^{(i)};z_j) - \nabla_{\bw} f(\bw_t,\bv_t;z_j)\\
  \nabla_{\bv} f(\bw_t,\bv_t;z_j) - \nabla_{\bv} f(\bw_t^{(i)},\bv_t^{(i)};z_j)
  \end{pmatrix}\Bigg\|_2^2\Bigg] \notag\\
& +\ebb\Bigg[\frac{1}{n^2}\sum_{i=1}^n\Bigg\|\begin{pmatrix}
  \nabla_{\bw} f(\bw_t^{(i)},\bv_t^{(i)};z_i') - \nabla_{\bw} f(\bw_t,\bv_t;z_i)\\
  \nabla_{\bv} f(\bw_t,\bv_t;z_i) - \nabla_{\bv} f(\bw_t^{(i)},\bv_t^{(i)};z_i')
  \end{pmatrix}\Bigg\|_2^2\Bigg] \notag\\
 & \leq L^2 \ebb\Bigg[\frac{1}{n}\sum_{i=1}^n\Bigg\|\begin{pmatrix}
                \bw_t-\bw_t^{(i)} \\
                \bv_t-\bv_t^{(i)}
            \end{pmatrix}\Bigg\|_2^2\Bigg] + {\frac{8\ebb[G_S^2]}{n}},
\end{align}
where we have used the fact that $i_t$ follows the uniform distribution over $[n]$ for the first step, Assumption~\ref{ass:smooth} and Lemma~\ref{lem:gradient-norm} for the last step.
This shows Eq.~\eqref{simple-condition-a} holds with $A=L^2,B={\frac{8\ebb[G_S^2]}{n}}$.
\end{proof}

We then plug these estimates into Theorem~\ref{lem:stab-gda} to immediately get Theorem~\ref{cor:sgda_stability}. We omit the details for brevity.

\subsubsection{Zeroth-order SGDA\label{sec:proof-zero-sgda}}
In this subsection, we give the proofs for Zeroth-order SGDA. Lemma~\ref{lemma_zeroth_order_sgd} is a simple variant of Lemma~1 in~\cite{nikolakakis2022black}. We give the proof for completeness.

\begin{lemma}\label{sphere_moment}
Let $\mathbf{u} \in \mathbb{R}^d$ from $\operatorname{Unif}(\mathbb{S}^{d-1})$, then
$
\mathbb{E}[\bu \bu^{\top}] = \frac{1}{d} \mathbf{I}_d,
$
where $\mathbf{I}_d$ denotes the $d$-dimensional identity matrix.
\end{lemma}
\begin{proof}
By definition, the uniform distribution on the unit sphere is rotationally invariant. That is, for any orthogonal matrix $\mathbf{R} \in \rbb^d$, the random vectors $\mathbf{Ru}$ and $\bu$ have the same distribution.\\
Define  $\mathbf{A} = \mathbb{E}[\bu \bu^{\top}]$ . We first show that $\mathbf{A}=c \mathbf{I}_d$ for some constant $c \in \rbb$.\\
For any orthogonal matrix $\mathbf{R}$, using linearity of expectation and rotational invariance,
\[
   \mathbf{R} \mathbf{A} \mathbf{R}^{\top} = \mathbf{R} \, \mathbb{E}[\bu \bu^{\top}] \, \mathbf{R}^{\top} = \mathbb{E}[\mathbf{R} \bu (\mathbf{R} \bu)^{\top}] = \mathbb{E}[\bu \bu^{\top}] = \mathbf{A}.
\]
This implies that $\mathbf{A}$ commutes with all orthogonal matrices. Since any matrix that commutes with all orthogonal matrices must be a scalar multiple of the identity, we conclude that $\mathbf{A}=c \mathbf{I}_d$ for some constant $c$.\\
Now we determine $c$. Since
   \begin{align*}
   \operatorname{tr}(\mathbf{A}) = \operatorname{tr}\big(\mathbb{E}[\bu \bu^{\top}]\big)
   = \mathbb{E}\big[\operatorname{tr}(\bu \bu^{\top})\big]
   = \mathbb{E}[\bu^{\top} \bu]
   = \mathbb{E}[\|\bu\|_2^2]
   = 1,
   \quad \textbf{and} \quad
   \operatorname{tr}(\mathbf{A}) = \operatorname{tr}(c\mathbf{I}_d) = cd,
   \end{align*}
   we know $c=\frac{1}{d}$. The result follows directly.
\end{proof}

\begin{lemma}
\label{lemma_zeroth_order_sgd}
Let $\mathbf{u}_l \in \mathbb{R}^d, l \in\{1,2 \ldots, K\}$ be drawn independently from $\operatorname{Unif}(\mathbb{S}^{d-1})$. For every random vector $\mathbf{v} \in \mathbb{R}^d$ independent of all $\mathbf{u}_l, l \in\{1,2 \ldots, K\}$, it is true that
$$
\mathbb{E}_{\bu}\Big[\Big\|\frac{1}{K} \sum_{l=1}^K\langle\mathbf{v}, \mathbf{u}_l\rangle \mathbf{u}_l\Big\|_2^2\Big] =  \frac{1}{K^2}\Big(\frac{K}{d}+\frac{K(K-1)}{d^2}\Big)\|\mathbf{v}\|_2^2=\frac{K+d-1}{K d^2}\|\mathbf{v}\|_2^2,
$$
where $\mathbb{E}_{\bu}[\cdot]$ denotes the conditional expectation w.r.t. $\bu_1,\ldots,\bu_K$.
\end{lemma}

\begin{proof}
For any fixed $\mathbf{v} \in \mathbb{R}^d$, we have
\begin{align}\label{zeroth-1-unif}
&\Big\|\frac{1}{K} \sum_{l=1}^K\langle \mathbf{v}, \mathbf{u}_l\rangle \mathbf{u}_l\Big\|_2^2
=  \frac{1}{K^2}\Big[\sum_{l=1}^K \mathbb{E}[(\langle\mathbf{v}, \mathbf{u}_l\rangle)^2]+\sum_{l \neq m} \mathbb{E}[\langle\mathbf{v}, \mathbf{u}_l\rangle\langle\mathbf{u}_l, \mathbf{u}_m\rangle\langle\mathbf{v}, \mathbf{u}_m\rangle]\Big].
\end{align}
For the first term in Eq.~\eqref{zeroth-1-unif}, we have
\begin{align*}
\mathbb{E}[(\langle\mathbf{v}, \mathbf{u}_l\rangle)^2]=\mathbf{v}^{\top} \mathbb{E}[\mathbf{u}_l \mathbf{u}_l^{\top}] \mathbf{v}=\frac{1}{d}\|\mathbf{v}\|_2^2,
\end{align*}
where we use the fact that $\mathbb{E}[\mathbf{u}_l \mathbf{u}_l^{\top}]=\frac{1}{d} \mathbf{I}_d$ from Lemma~\ref{sphere_moment}.\\
For the second term in Eq.~\eqref{zeroth-1-unif}, due to the independency of $\bu_l$ and $\bu_m$, we know
\begin{align*}
\mathbb{E}[\langle\mathbf{v}, \mathbf{u}_l\rangle\langle\mathbf{u}_l, \mathbf{u}_m\rangle\langle\mathbf{v}, \mathbf{u}_m\rangle]=\mathbb{E}[(\langle\mathbf{v}, \mathbf{u}_l\rangle \mathbf{u}_l)^{\top}] \cdot \mathbb{E}[\langle\mathbf{v}, \mathbf{u}_m\rangle \mathbf{u}_m]
=(\frac{1}{d} \mathbf{v})^{\top} \cdot(\frac{1}{d} \mathbf{v})=\frac{1}{d^2}\|\mathbf{v}\|_2^2,
\end{align*}
where we use the fact that $\mathbb{E}[(\langle\mathbf{v}, \mathbf{u}\rangle \mathbf{u})]=\frac{1}{d} \mathbf{v}$ for any $\bu$ from $\operatorname{Unif}(\mathbb{S}^{d-1})$.
We plug the above two identities back into Eq.~\eqref{zeroth-1-unif}, and derive
\begin{align*}
& \mathbb{E}_{\bu}\Big[\Big\|\frac{1}{K} \sum_{l=1}^K\langle\mathbf{v}, \mathbf{u}_l\rangle \mathbf{u}_l\Big\|_2^2\Big]
=\frac{1}{K^2}\Big(\frac{K}{d}+\frac{K(K-1)}{d^2}\Big)\|\mathbf{v}\|_2^2=\frac{K+d-1}{K d^2}\|\mathbf{v}\|_2^2.
\end{align*}
The proof is completed.
\end{proof}

The following lemma is a direct application of Lemma~\ref{lemma_zeroth_order_sgd}, which is useful for our main results.

\begin{lemma}
\label{lemma_zeroth_order_sgda}
Let $\mathbf{u}_1, \mathbf{u}_2 \in \mathbb{R}^{d}$ be drawn independently from $\operatorname{Unif}(\mathbb{S}^{d-1})$. For every random vector $\mathbf{v}_1 \in \mathbb{R}^{d}$ independent of all $\mathbf{u}_1$ and $\mathbf{v}_2 \in \mathbb{R}^{d}$ independent of all $\mathbf{u}_2$, it is true that
$$
\ebb_{\bu}\bigg[\bigg\|\begin{pmatrix}
  \bv_1^\top \bu_1\bu_1\\
 \bv_2^{\top}\bu_2\bu_2
  \end{pmatrix}\bigg\|_2^2\Bigg]
  = \frac{1}{d} \bigg\|\begin{pmatrix}
  \bv_1\\
 \bv_2
  \end{pmatrix}\bigg\|_2^2,
$$
where $\mathbb{E}_{\bu}[\cdot]$ denotes the conditional expectation w.r.t. $\bu_1,\bu_2$.
\end{lemma}
\begin{proof}
By Lemma \ref{lemma_zeroth_order_sgd} with $K=1$, we derive
\begin{align*}
\ebb_{\bu}\bigg[\bigg\|\begin{pmatrix}
  \bv_1^\top \bu_1\bu_1\\
 \bv_2^{\top}\bu_2\bu_2
  \end{pmatrix}\bigg\|_2^2\bigg]
& = \ebb_{\bu_1}\big[\|
  \bv_1^\top \bu_1\bu_1\|_2^2\big] + \ebb_{\bu_2}\big[\|
  \bv_2^\top \bu_2\bu_2\|_2^2\big] = \frac{1}{d}\|\bv_1\|_2^2 + \frac{1}{d}\|\bv_2\|_2^2.
\end{align*}
The stated bound follows directly.
\end{proof}

We now prove an important lemma which shows that $\tilde{f}_1,\tilde{f}_2$ preserve the smoothness, bounded second moment condition, convexity, and concavity of $f$, and that $g$ is an unbiased estimator.

\begin{proof}[Proof of Lemma~\ref{lem:unbiase-zeroth}]
\textbf{1. }Since for any $\bz$, $\bw\mapsto f(\bw,\bv;\bz)$ is convex, from Lemma 4.1 in \citet{gao2018information}, we know
$\bw \mapsto \tilde{f}_1(\bw,\bv;\bz)$ is convex. Since for any $\bz$, $\bv\mapsto f(\bw,\bv;\bz)$ is concave, we know $\bv \mapsto \tilde{f}_1(\bw,\bv;\bz)$ is concave. Then, $(\bw,\bv) \mapsto \tilde{f}_1(\bw,\bv;\bz)$ is convex-concave. Similarly, $(\bw,\bv) \mapsto \tilde{f}_2(\bw,\bv;\bz)$ is convex-concave.\\
Recall the surrogate functions
\begin{align*}
&\tilde{f}_1(\bw,\bv;\bz)=\ebb_{\bu \sim \operatorname{Unif}(\mathbb{B}^{d})}\big[f(\bw+\mu_1\bu,\bv;\bz)\big] = \frac{1}{\kappa} \int_{\mathbb{B}^d} f(\bw+\mu_1\bu,\bv;\bz) d\bu,\\
&\tilde{f}_2(\bw,\bv;\bz)=\ebb_{\bu \sim \operatorname{Unif}(\mathbb{B}^{d})}\big[f(\bw,\bv+\mu_2\bu;\bz)\big]= \frac{1}{\kappa} \int_{\mathbb{B}^d} f(\bw,\bv+\mu_2\bu;\bz) d\bu,
\end{align*}
where $\kappa$ is the volume of the ball $\mathbb{B}^d$,
we know
\begin{align*}
\nabla_{\bw} \tilde{f}_1 (\bw,\bv;\bz) = \frac{1}{\kappa} \int_{\mathbb{B}^d} \nabla_{\bw} f(\bw+\mu_1\bu,\bv;\bz) d\bu,\quad \nabla_{\bw} \tilde{f}_2 (\bw,\bv;\bz) = \frac{1}{\kappa} \int_{\mathbb{B}^d} \nabla_{\bw} f(\bw,\bv+\mu_2 \bu;\bz) d\bu,\\
\nabla_{\bv} \tilde{f}_1 (\bw,\bv;\bz) = \frac{1}{\kappa} \int_{\mathbb{B}^d} \nabla_{\bv} f(\bw+\mu_1 \bu,\bv;\bz) d\bu, \quad \nabla_{\bv} \tilde{f}_2 (\bw,\bv;\bz) = \frac{1}{\kappa} \int_{\mathbb{B}^d} \nabla_{\bv} f(\bw,\bv+\mu_2 \bu;\bz) d\bu.
\end{align*}
Then for any $\bw,\bw' \in \wcal$ and $\bv,\bv' \in \vcal$,
\begin{align*}
&\Bigg\|\begin{pmatrix}
           \nabla_{\bw} \tilde{f}_1 (\bw,\bv;\bz)-\nabla_{\bw} \tilde{f}_1 (\bw',\bv';\bz) \\
           \nabla_{\bv} \tilde{f}_1 (\bw,\bv;\bz)-\nabla_{\bv} \tilde{f}_1 (\bw',\bv';\bz)
         \end{pmatrix}\Bigg\|_2\\
& = \Bigg\|\begin{pmatrix}
           \frac{1}{\kappa}\int_{\mathbb{B}^d} (\nabla_{\bw}f(\bw+\mu_1 \bu,\bv;\bz)-\nabla_{\bw}f(\bw'+\mu_1 \bu,\bv';\bz))d\bu \\
           \frac{1}{\kappa}\int_{\mathbb{B}^d} (\nabla_{\bv}f(\bw+\mu_1 \bu,\bv;\bz)-\nabla_{\bv}f(\bw'+\mu_1 \bu,\bv';\bz)) d\bu
         \end{pmatrix}\Bigg\|_2\\
& = \Bigg\|\frac{1}{\kappa}\int_{\mathbb{B}^d} \begin{pmatrix}
           \nabla_{\bw}f(\bw+\mu_1 \bu,\bv;\bz)-\nabla_{\bw}f(\bw'+\mu_1 \bu,\bv';\bz) \\
          \nabla_{\bv}f(\bw+\mu_1 \bu,\bv;\bz)-\nabla_{\bv}f(\bw'+\mu_1 \bu,\bv';\bz)
         \end{pmatrix}d\bu \Bigg\|_2\\
& \leq \frac{1}{\kappa}\int_{\mathbb{B}^d} \bigg\|\begin{pmatrix}
           \nabla_{\bw}f(\bw+\mu_1 \bu,\bv;\bz)-\nabla_{\bw}f(\bw'+\mu_1 \bu,\bv';\bz) \\
          \nabla_{\bv}f(\bw+\mu_1 \bu,\bv;\bz)-\nabla_{\bv}f(\bw'+\mu_1 \bu,\bv';\bz)
         \end{pmatrix} \bigg\|_2 d\bu\\
& \!\leq\!
         L\bigg\|\begin{pmatrix}
        \bw\!-\!\bw' \\
        \bv\!-\!\bv'
\end{pmatrix}\bigg\|_2,
\end{align*}
where we use Assumption~\ref{ass:smooth} for the last step.
In this way, we show that $(\bw,\bv) \mapsto \tilde{f}_1(\bw,\bv;\bz)$ is $L$-smooth, and similarly, $(\bw,\bv) \mapsto \tilde{f}_2(\bw,\bv;\bz)$ is $L$-smooth.\\
We know
\begin{align*}
\nabla_{\mathbf{w}} \tilde{f}_1(\mathbf{w}, \mathbf{v} ; S_r)=\nabla_{\mathbf{w}} \mathbb{E}_{\mathbf{u}}[f(\mathbf{w}+\mu_1 \mathbf{u}, \mathbf{v} ; S_r)]=\mathbb{E}_{\mathbf{u}}[\nabla_{\mathbf{w}} f(\mathbf{w}+\mu_1 \mathbf{u}, \mathbf{v} ; S_r)],
\end{align*}
where $\bu \sim \operatorname{Unif}(\mathbb{S}^{d-1})$. Then, applying Jensen's inequality, we have
\begin{align*}
\|\nabla_{\mathbf{w}} \tilde{f}_1(\mathbf{w}, \mathbf{v} ; S_r)\|_2^2
= \|\mathbb{E}_{\mathbf{u}}[\nabla_{\mathbf{w}} f(\mathbf{w}+\mu_1 \mathbf{u}, \mathbf{v} ; S_r)]\|_2^2
\leq \mathbb{E}_{\mathbf{u}}[\|\nabla_{\mathbf{w}} f(\mathbf{w}+\mu_1 \mathbf{u}, \mathbf{v} ; S_r)\|_2^2].
\end{align*}
Taking expectation w.r.t $r$ to both sides above,
\begin{align*}
\ebb_r [\|\nabla_{\mathbf{w}} \tilde{f}_1(\mathbf{w}, \mathbf{v} ; S_r)\|_2^2 ]
&\leq \mathbb{E}_{\mathbf{u},r}[\|\nabla_{\mathbf{w}} f(\mathbf{w}+\mu_1 \mathbf{u}, \mathbf{v} ; S_r)\|_2^2]\\
& = \mathbb{E}_{\mathbf{u}}\big[\ebb_r[\|\nabla_{\mathbf{w}} f(\mathbf{w}+\mu_1 \mathbf{u}, \mathbf{v} ; S_r)\|_2^2]\big]
\leq G_S^2,
\end{align*}
where we use Assumption~\ref{ass:moment} for the last step. The results hold similarly for $\ebb_r[\|\nabla_{\bv} \tilde{f}_2(\bw,\bv;S_r)\|_2^2]$.\\
\textbf{2. }Follows directly from Lemma 4.4 in~\citet{shalev2012online}.
\end{proof}

We now present the proof of Lemma~\ref{lem:tilde-difference} and Lemma~\ref{lem:condition-a-zero}, which show that Eq.~\eqref{simple-condition-a} and Eq.~\eqref{simple-condition-d} hold for Zeroth-order SGDA.
To prove Lemma~\ref{lem:tilde-difference}, we require the following lemma to show that the Hessian matrices can be bounded.
\begin{lemma}\label{lem:hessian}
Let Assumption \ref{ass:smooth} hold. For any matrix $A$, the spectral norm is
$$
\|\mathbf{A}\|_2=\sup _{\mathbf{x} \in \mathbb{R}^n, \|\mathbf{x}\|_2=1} \|\mathbf{A} \mathbf{x}\|_2.
$$
Then for all $\bw\in\wcal,\bv\in\vcal$ and $z\in\zcal$,
\[
\|\nabla_{\bw}^2 f(\bw,\bv;z)\|_2 \leq L,
\quad
\|\nabla_{\bv}^2 f(\bw,\bv;z)\|_2 \leq L.
\]
\end{lemma}

\begin{proof}
From Assumption \ref{ass:smooth}, we know $\|\nabla^2 f(\bw,\bv ; z)\|_2 \leq L$, where
\begin{align}\label{hessian}
\nabla^2 f(\bw,\bv ; z)=\begin{pmatrix}
\nabla_{\bw}^2 f(\bw,\bv ; z) & \nabla_{\mathbf{w}} \nabla_{\bv} f(\bw,\bv ; z) \\
\nabla_{\bv} \nabla_{\bw} f(\bw,\bv ; z) & \nabla_{\bv}^2 f(\bw,\bv ; z)
\end{pmatrix}.
\end{align}
Then, for any $\mathbf{h}_{\bw} \in \rbb^d$ such that $\|\mathbf{h}_{\bw}\|_2=1$, we construct $\mathbf{h}=(\mathbf{h}_{\bw}^{\top},0,\ldots,0)^{\top} \in \rbb^{2d}$. According to the definition of the matrix norm and Eq.~\eqref{hessian}, we know
\begin{align*}
L &\geq \|\nabla^2 f(\bw,\bv ; z) \cdot \mathbf{h}\|_2=\sqrt{\|\nabla_{\mathbf{w}}^2 f(\bw,\bv ; z) \cdot \mathbf{h}_{\bw}\|_2^2+\|\nabla_{\mathbf{v}, \mathbf{w}}^2 f(\bw,\bv ; z) \cdot \mathbf{h}_{\bw}\|_2^2}\\
& \geq\|\nabla_{\mathbf{w}}^2 f(\bw,\bv ; z)\cdot \mathbf{h}_{\bw}\|_2
\end{align*}
Therefore, $\|\nabla_{\mathbf{w}}^2 f(\bw,\bv ; z)\|_2 \leq L$. Similarly, $\|\nabla_{\mathbf{v}}^2 f(\bw,\bv ; z)\|_2 \leq L$.
\end{proof}

\begin{lemma}\label{lem:tilde-difference}
Let Assumption \ref{ass:moment} hold, $\tilde{f}_1,\tilde{f}_2$ be given in Eq.~\eqref{tilde-f-zero}. Then, Eq.~\eqref{simple-condition-d} holds with $\widetilde{G}^2 = 8\ebb[G_S^2]$.
\end{lemma}

\begin{proof}[Proof of Lemma~\ref{lem:tilde-difference}]
By the inequality that $(a+b)^2 \leq 2a^2 + 2b^2$, we know
\begin{align*}
& \Bigg\|\begin{pmatrix}
  \nabla_{\bw} \tilde{f}(\bw_t^{(i)},\bv_t^{(i)};z_i') - \nabla_{\bw} \tilde{f}(\bw_t,\bv_t;z_i)\\
  \nabla_{\bv} \tilde{f}(\bw_t,\bv_t;z_i) - \nabla_{\bv} \tilde{f}(\bw_t^{(i)},\bv_t^{(i)};z_i')
\end{pmatrix}\Bigg\|_2^2\\
& = \big\|\nabla_{\bw} \tilde{f}(\bw_t^{(i)},\bv_t^{(i)};z_i') - \nabla_{\bw} \tilde{f}(\bw_t,\bv_t;z_i)\big\|_2^2
+ \big\|\nabla_{\bv} \tilde{f}(\bw_t,\bv_t;z_i) - \nabla_{\bv} \tilde{f}(\bw_t^{(i)},\bv_t^{(i)};z_i')\big\|_2^2\\
& \leq 2\Big(\big\|\nabla_{\bw} \tilde{f}(\bw_t^{(i)},\bv_t^{(i)};z_i')\big\|_2^2 +\big\|\nabla_{\bw} \tilde{f}(\bw_t,\bv_t;z_i)\big\|_2^2
+ \big\|\nabla_{\bv} \tilde{f}(\bw_t,\bv_t;z_i) \big\|_2^2
+\big\|\nabla_{\bv} \tilde{f}(\bw_t^{(i)},\bv_t^{(i)};z_i')\big\|_2^2\Big).
\end{align*}
Taking the expectation of the above inequality, then taking an average of $i \in [n]$, we get ($r$ follows the uniform distribution over $[n]$)
\begin{align*}
& \ebb\Bigg[\frac{1}{n}\sum_{i=1}^n\Bigg\|\begin{pmatrix}
  \nabla_{\bw} \tilde{f}(\bw_t^{(i)},\bv_t^{(i)};z_i') - \nabla_{\bw} \tilde{f}(\bw_t,\bv_t;z_i)\\
  \nabla_{\bv} \tilde{f}(\bw_t,\bv_t;z_i) - \nabla_{\bv} \tilde{f}(\bw_t^{(i)},\bv_t^{(i)};z_i')
  \end{pmatrix}\Bigg\|_2^2\Bigg]\\
& \leq 4\ebb\bigg[\frac{1}{n}\sum_{i=1}^n \big\|\nabla_{\bw} \tilde{f}(\bw_t,\bv_t;z_i)\big\|_2^2 \bigg] + 4\ebb\bigg[\frac{1}{n}\sum_{i=1}^n \big\|\nabla_{\bv} \tilde{f}(\bw_t,\bv_t;z_i)\big\|_2^2 \bigg]\\
& = 4\ebb\big[\|\nabla_{\bw} \tilde{f}(\bw_t,\bv_t;z_r)\|_2^2 \big] + 4\ebb\big[\|\nabla_{\bv} \tilde{f}(\bw_t,\bv_t;z_r)\|_2^2 \big] \leq 8\ebb[G_S^2],
\end{align*}
where in the first step we have used the fact $\ebb[\|\nabla_{\bw} \tilde{f}(\bw_t^{(i)},\bv_t^{(i)};z_i')\|_2^2] = \ebb[\|\nabla_{\bw} \tilde{f}(\bw_t,\bv_t;z_i)\|_2^2],\ebb[\|\nabla_{\bv} \tilde{f}(\bw_t^{(i)},\bv_t^{(i)};z_i')\|_2^2] = \ebb[\|\nabla_{\bv} \tilde{f}(\bw_t,\bv_t;z_i)\|_2^2]$, and in the last step we have used Lemma~\ref{lem:unbiase-zeroth}.
\end{proof}

To prove Lemma \ref{lem:condition-a-zero}, we need the following two lemmas. 


\begin{lemma}[Spherical Poincaré inequality; see, e.g.,
\cite{dolbeault2013sharp,ledoux2001concentration}]
\label{lem:hu}
Let $d\geq 2$ and
$\bu\sim\operatorname{Unif}(\mathbb{S}^{d-1})$.
For any locally Lipschitz function
$h:\mathbb{S}^{d-1}\to\mathbb{R}$, we have
\[
\mathbb{E}_{\bu}\![
  (h(\bu)-\mathbb{E}_{\bu}[h(\bu)])^2
]
\leq
\frac{1}{d-1}
\mathbb{E}_{\bu}\![
  \|\nabla_{\mathbb{S}}h(\bu)\|_2^2
],
\]
where $\nabla_{\mathbb{S}}h(\bu)$ denotes the tangential gradient of
$h$ at $\bu$. More precisely, if $\widetilde h$ is a differentiable
extension of $h$ to a neighborhood of $\mathbb{S}^{d-1}$, denoting $\mathbf{I}_d \in \rbb^{d \times d}$ as the identity matrix, then
$
\nabla_{\mathbb{S}}h(\bu)
=
(\mathbf{I}_d-\bu\bu^\top)\nabla_{\bu}\widetilde h(\bu)
$
(see, e.g., Section 4.6.1 in~\cite{absil2008optimization}).
\end{lemma}
\begin{lemma}\label{lem:lip_u_central_local}
Let Assumption~\ref{ass:smooth} hold. For any $\bw,\bv,\bw',\bv'$,
the function
\begin{align*}
h(\mathbf{u}) := \left[ f(\mathbf{w}'+\mu \mathbf{u}, \mathbf{v}'; z) - f(\mathbf{w}'-\mu \mathbf{u}, \mathbf{v}'; z) \right] - \left[ f(\mathbf{w}+\mu \mathbf{u}, \mathbf{v}; z) - f(\mathbf{w}-\mu \mathbf{u}, \mathbf{v}; z) \right]
\end{align*}
with $\mathbf{u} \in \mathbb{S}^{d-1}$ is $2\mu L \Big\| \begin{pmatrix}
\mathbf{w} - \mathbf{w}' \\
\mathbf{v} - \mathbf{v}'
\end{pmatrix} \Big\|_2$-Lipschitz continuous with respect to $\mathbf{u}$.
\end{lemma}

\begin{proof}
Take arbitrary $\mathbf{u}_1, \mathbf{u}_2 \in \mathbb{S}^{d-1}$. Decompose the difference $h(\mathbf{u}_1) - h(\mathbf{u}_2)$ as follows:
\begin{align}\label{hu1_hu2}
& h(\mathbf{u}_1) - h(\mathbf{u}_2)\notag\\
&= \left[ f(\mathbf{w}'+\mu\mathbf{u}_1, \mathbf{v}'; z) - f(\mathbf{w}'-\mu\mathbf{u}_1, \mathbf{v}'; z) - f(\mathbf{w}+\mu\mathbf{u}_1, \mathbf{v}; z) + f(\mathbf{w}-\mu\mathbf{u}_1, \mathbf{v}; z) \right] \notag\\
&\quad - \left[ f(\mathbf{w}'+\mu\mathbf{u}_2, \mathbf{v}'; z) - f(\mathbf{w}'-\mu\mathbf{u}_2, \mathbf{v}'; z) - f(\mathbf{w}+\mu\mathbf{u}_2, \mathbf{v}; z) + f(\mathbf{w}-\mu\mathbf{u}_2, \mathbf{v}; z) \right] \notag\\
&= \left[ f(\mathbf{w}'+\mu\mathbf{u}_1, \mathbf{v}'; z) - f(\mathbf{w}'+\mu\mathbf{u}_2, \mathbf{v}'; z) \right] - \left[ f(\mathbf{w}'-\mu\mathbf{u}_1, \mathbf{v}'; z) - f(\mathbf{w}'-\mu\mathbf{u}_2, \mathbf{v}'; z) \right] \notag\\
&\quad - \left[ f(\mathbf{w}+\mu\mathbf{u}_1, \mathbf{v}; z) - f(\mathbf{w}+\mu\mathbf{u}_2, \mathbf{v}; z) \right] + \left[ f(\mathbf{w}-\mu\mathbf{u}_1, \mathbf{v}; z) - f(\mathbf{w}-\mu\mathbf{u}_2, \mathbf{v}; z) \right].
\end{align}

Apply the Newton-Leibniz formula
$$
f(\mathbf{a}+\mathbf{b}, \mathbf{c}; z) - f(\mathbf{a}, \mathbf{c}; z) = \int_0^1 \nabla_{\mathbf{w}} f(\mathbf{a} + t\mathbf{b}, \mathbf{c}; z) \cdot \mathbf{b} dt
$$
to each of the four difference terms:

For the first term, let $\mathbf{a} = \mathbf{w}'+\mu\mathbf{u}_2$, $\mathbf{b} = \mu(\mathbf{u}_1 - \mathbf{u}_2)$, and $\mathbf{c} = \mathbf{v}'$:
    $$
    f(\mathbf{w}'+\mu\mathbf{u}_1, \mathbf{v}'; z) - f(\mathbf{w}'+\mu\mathbf{u}_2, \mathbf{v}'; z) = \mu \int_0^1 \nabla_{\mathbf{w}} f\left( \mathbf{w}' + \mu(\mathbf{u}_2 + t(\mathbf{u}_1-\mathbf{u}_2)), \mathbf{v}'; z \right) \cdot (\mathbf{u}_1 - \mathbf{u}_2) dt.
    $$

For the second term, let $\mathbf{a} = \mathbf{w}'-\mu\mathbf{u}_2$, $\mathbf{b} = -\mu(\mathbf{u}_1 - \mathbf{u}_2)$, and $\mathbf{c} = \mathbf{v}'$:
    $$
    f(\mathbf{w}'-\mu\mathbf{u}_1, \mathbf{v}'; z) - f(\mathbf{w}'-\mu\mathbf{u}_2, \mathbf{v}'; z) = -\mu \int_0^1 \nabla_{\mathbf{w}} f\left( \mathbf{w}' - \mu(\mathbf{u}_2 + t(\mathbf{u}_1-\mathbf{u}_2)), \mathbf{v}'; z \right) \cdot (\mathbf{u}_1 - \mathbf{u}_2) dt.
    $$

For the third term, let $\mathbf{a} = \mathbf{w}+\mu\mathbf{u}_2$, $\mathbf{b} = \mu(\mathbf{u}_1 - \mathbf{u}_2)$, and $\mathbf{c} = \mathbf{v}$:
    $$
    f(\mathbf{w}+\mu\mathbf{u}_1, \mathbf{v}; z) - f(\mathbf{w}+\mu\mathbf{u}_2, \mathbf{v}; z) = \mu \int_0^1 \nabla_{\mathbf{w}} f\left( \mathbf{w} + \mu(\mathbf{u}_2 + t(\mathbf{u}_1-\mathbf{u}_2)), \mathbf{v}; z \right) \cdot (\mathbf{u}_1 - \mathbf{u}_2) dt.
    $$

For the fourth term, let $\mathbf{a} = \mathbf{w}-\mu\mathbf{u}_2$, $\mathbf{b} = -\mu(\mathbf{u}_1 - \mathbf{u}_2)$, and $\mathbf{c} = \mathbf{v}$:
    $$
    f(\mathbf{w}-\mu\mathbf{u}_1, \mathbf{v}; z) - f(\mathbf{w}-\mu\mathbf{u}_2, \mathbf{v}; z) = -\mu \int_0^1 \nabla_{\mathbf{w}} f\left( \mathbf{w} - \mu(\mathbf{u}_2 + t(\mathbf{u}_1-\mathbf{u}_2)), \mathbf{v}; z \right) \cdot (\mathbf{u}_1 - \mathbf{u}_2) dt.
    $$

Let $\boldsymbol{\xi}_t = \mathbf{u}_2 + t(\mathbf{u}_1-\mathbf{u}_2)$. Note that $\|\boldsymbol{\xi}_t\|_2 \leq 1$ for all $t \in [0,1]$ since $\mathbf{u}_1, \mathbf{u}_2 \in \mathbb{S}^{d-1}$. Substitute the four integrals back into Eq.~\eqref{hu1_hu2}, we get
\begin{align}\label{hu1_hu2_}
h(\mathbf{u}_1) - h(\mathbf{u}_2) = \mu \int_0^1 [ ( \nabla_{\mathbf{w}} f( \mathbf{w}' + \mu\boldsymbol{\xi}_t, \mathbf{v}'; z ) - \nabla_{\mathbf{w}} f( \mathbf{w} + \mu\boldsymbol{\xi}_t, \mathbf{v}; z ) )\notag\\
+ ( \nabla_{\mathbf{w}} f( \mathbf{w}' - \mu\boldsymbol{\xi}_t, \mathbf{v}'; z ) - \nabla_{\mathbf{w}} f( \mathbf{w} - \mu\boldsymbol{\xi}_t, \mathbf{v}; z ) ) ] \cdot (\mathbf{u}_1 - \mathbf{u}_2) dt.
\end{align}

Apply the Cauchy-Schwarz inequality $|\mathbf{a} \cdot \mathbf{b}| \leq \|\mathbf{a}\|_2 \|\mathbf{b}\|_2$ followed by the triangle inequality $\|\mathbf{a} + \mathbf{b}\|_2 \leq \|\mathbf{a}\|_2 + \|\mathbf{b}\|_2$ to the integrand:
\begin{align*}
& \left| \left[ (\nabla_{\mathbf{w}} f'_+ - \nabla_{\mathbf{w}} f_+) + (\nabla_{\mathbf{w}} f'_- - \nabla_{\mathbf{w}} f_-) \right] \cdot (\mathbf{u}_1 - \mathbf{u}_2) \right| \\
&\leq \left( \left\| \nabla_{\mathbf{w}} f'_+ - \nabla_{\mathbf{w}} f_+ \right\|_2 + \left\| \nabla_{\mathbf{w}} f'_- - \nabla_{\mathbf{w}} f_- \right\|_2 \right) \cdot \|\mathbf{u}_1 - \mathbf{u}_2\|_2,
\end{align*}
where we use the shorthand $f'_+ = f(\mathbf{w}'+\mu\boldsymbol{\xi}_t, \mathbf{v}'; z)$, $f'_- = f(\mathbf{w}'-\mu\boldsymbol{\xi}_t, \mathbf{v}'; z)$, $f_+ = f(\mathbf{w}+\mu\boldsymbol{\xi}_t, \mathbf{v}; z)$, and $f_- = f(\mathbf{w}-\mu\boldsymbol{\xi}_t, \mathbf{v}; z)$.

Now apply the joint $L$-smoothness in Assumption~\ref{ass:smooth} to each gradient difference. For the first difference:
\begin{align*}
\left\| \nabla_{\mathbf{w}} f'_+ - \nabla_{\mathbf{w}} f_+ \right\|_2
& = \left\| \nabla_{\mathbf{w}} f(\mathbf{w}'+\mu\boldsymbol{\xi}_t, \mathbf{v}'; z) - \nabla_{\mathbf{w}} f(\mathbf{w}+\mu\boldsymbol{\xi}_t, \mathbf{v}; z) \right\|_2\\
&\leq L \left\| \begin{pmatrix} (\mathbf{w}' + \mu\boldsymbol{\xi}_t) - (\mathbf{w} + \mu\boldsymbol{\xi}_t) \\ \mathbf{v}' - \mathbf{v} \end{pmatrix} \right\|_2
= L \left\| \begin{pmatrix} \mathbf{w}' - \mathbf{w} \\ \mathbf{v}' - \mathbf{v} \end{pmatrix} \right\|_2.
\end{align*}
For the second difference, the $-\mu\boldsymbol{\xi}_t$ terms cancel identically:
\begin{align*}
\left\| \nabla_{\mathbf{w}} f'_- - \nabla_{\mathbf{w}} f_- \right\|_2
&\leq L \left\| \begin{pmatrix} (\mathbf{w}' - \mu\boldsymbol{\xi}_t) - (\mathbf{w} - \mu\boldsymbol{\xi}_t) \\ \mathbf{v}' - \mathbf{v} \end{pmatrix} \right\|_2
= L \left\| \begin{pmatrix} \mathbf{w}' - \mathbf{w} \\ \mathbf{v}' - \mathbf{v} \end{pmatrix} \right\|_2.
\end{align*}

Combining these estimates:
$$
\left\| \nabla_{\mathbf{w}} f'_+ - \nabla_{\mathbf{w}} f_+ \right\|_2 + \left\| \nabla_{\mathbf{w}} f'_- - \nabla_{\mathbf{w}} f_- \right\|_2 \leq 2L \left\| \begin{pmatrix} \mathbf{w} - \mathbf{w}' \\ \mathbf{v} - \mathbf{v}' \end{pmatrix} \right\|_2.
$$

Take the absolute value of Eq.~\eqref{hu1_hu2_} and substitute the above back into the expression for $|h(\mathbf{u}_1) - h(\mathbf{u}_2)|$ and integrate over $t \in [0,1]$:
\begin{align*}
|h(\mathbf{u}_1) - h(\mathbf{u}_2)|
&\leq \mu \int_0^1 2L \left\| \begin{pmatrix} \mathbf{w} - \mathbf{w}' \\ \mathbf{v} - \mathbf{v}' \end{pmatrix} \right\|_2 \cdot \|\mathbf{u}_1 - \mathbf{u}_2\|_2 dt
= 2\mu L \left\| \begin{pmatrix} \mathbf{w} - \mathbf{w}' \\ \mathbf{v} - \mathbf{v}' \end{pmatrix} \right\|_2 \cdot \|\mathbf{u}_1 - \mathbf{u}_2\|_2.
\end{align*}
This completes the proof.
\end{proof}

\begin{lemma}\label{lem:condition-a-zero}
Let $g(\bw_t,\bv_t;z_{i_t})$ be given in Eq.~\eqref{g-zero}. If Assumptions \ref{ass:smooth}, \ref{ass:moment} hold, then \eqref{simple-condition-a} holds with $A = L^2(1+\frac{4d}{K})$,
$B = \frac{8\ebb[G_S^2]}{n}(1+\frac{2d}{K})$.
\end{lemma}

\begin{proof}[Proof of Lemma~\ref{lem:condition-a-zero}]
For $\bu\sim\operatorname{Unif}(\mathbb{S}^{d-1})$, define the one-direction gradient estimators
\begin{align*}
\widehat{g}_{\bw}(\bw,\bv;z,\bu)
&=
\frac{f(\bw+\mu_1\bu,\bv;z)-f(\bw-\mu_1\bu,\bv;z)}
     {2\mu_1/d}\bu,\\
\widehat{g}_{\bv}(\bw,\bv;z,\bu)
&=
\frac{f(\bw,\bv+\mu_2\bu;z)-f(\bw,\bv-\mu_2\bu;z)}
     {2\mu_2/d}\bu.
\end{align*}
Since $\ebb[X^2]=\ebb[(X-\ebb[X])^2]+(\ebb[X])^2$ for a random variable $X$, we know
\begin{align}\label{condition-a-zero}
   & \ebb_{\bu}\Bigg[\Bigg\|\begin{pmatrix}
  g_{\bw}(\bw_t^{(i)},\bv_t^{(i)};z_{i_t}^{(i)}) - g_{\bw}(\bw_t,\bv_t;z_{i_t})\\
  g_{\bv}(\bw_t,\bv_t;z_{i_t}) - g_{\bv}(\bw_t^{(i)},\bv_t^{(i)};z^{(i)}_{i_t})
  \end{pmatrix}\Bigg\|_2^2\Bigg]\notag\\
  & = \Bigg\|\ebb_{\bu}\Bigg[\begin{pmatrix}
  g_{\bw}(\bw_t^{(i)},\bv_t^{(i)};z_{i_t}^{(i)}) - g_{\bw}(\bw_t,\bv_t;z_{i_t})\\
  g_{\bv}(\bw_t,\bv_t;z_{i_t}) - g_{\bv}(\bw_t^{(i)},\bv_t^{(i)};z^{(i)}_{i_t})
  \end{pmatrix}\Bigg]\Bigg\|_2^2\notag\\
  & \quad + \ebb_{\bu}\Bigg[\Bigg\|\begin{pmatrix}
  g_{\bw}(\bw_t^{(i)},\bv_t^{(i)};z_{i_t}^{(i)}) - g_{\bw}(\bw_t,\bv_t;z_{i_t})\\
  g_{\bv}(\bw_t,\bv_t;z_{i_t}) - g_{\bv}(\bw_t^{(i)},\bv_t^{(i)};z^{(i)}_{i_t})
  \end{pmatrix}
  -
   \ebb_{\bu}\Bigg[\begin{pmatrix}
  g_{\bw}(\bw_t^{(i)},\bv_t^{(i)};z_{i_t}^{(i)}) - g_{\bw}(\bw_t,\bv_t;z_{i_t})\\
  g_{\bv}(\bw_t,\bv_t;z_{i_t}) - g_{\bv}(\bw_t^{(i)},\bv_t^{(i)};z^{(i)}_{i_t})
  \end{pmatrix}\Bigg]\Bigg\|_2^2\Bigg].
\end{align}

Since an average of $K$ independent random vectors divides the
variance by a factor of $K$, we know
\begin{align}\label{eq:variance_reduce}
& \ebb_{\bu}\Bigg[\Bigg\|\begin{pmatrix}
  g_{\bw}(\bw_t^{(i)},\bv_t^{(i)};z_{i_t}^{(i)}) - g_{\bw}(\bw_t,\bv_t;z_{i_t})\\
  g_{\bv}(\bw_t,\bv_t;z_{i_t}) - g_{\bv}(\bw_t^{(i)},\bv_t^{(i)};z^{(i)}_{i_t})
  \end{pmatrix}
  -
   \ebb_{\bu}\Bigg[\begin{pmatrix}
  g_{\bw}(\bw_t^{(i)},\bv_t^{(i)};z_{i_t}^{(i)}) - g_{\bw}(\bw_t,\bv_t;z_{i_t})\\
  g_{\bv}(\bw_t,\bv_t;z_{i_t}) - g_{\bv}(\bw_t^{(i)},\bv_t^{(i)};z^{(i)}_{i_t})
  \end{pmatrix}\Bigg]\Bigg\|_2^2\Bigg]\notag\\
& \leq \frac{1}{K} \ebb_{\bu}\Bigg[\Bigg\|\begin{pmatrix}
  \widehat{g}_{\bw}(\bw_t^{(i)},\bv_t^{(i)};z_{i_t}^{(i)},\bu_{1,t,1})
  -\widehat{g}_{\bw}(\bw_t,\bv_t;z_{i_t},\bu_{1,t,1})\\
  \widehat{g}_{\bv}(\bw_t,\bv_t;z_{i_t},\bu_{2,t,1})
  -\widehat{g}_{\bv}(\bw_t^{(i)},\bv_t^{(i)};z_{i_t}^{(i)},\bu_{2,t,1})
  \end{pmatrix}
  \Bigg\|_2^2\Bigg],
\end{align}
where we have also used the inequality $\ebb[(X-\ebb[X])^2]\leq\ebb[X^2]$.

We next establish an empirical second-moment bound for these
one-direction estimators. For fixed $\bw,\bv$ and $z_i$, let
\[
h_i(\bu)
=
f(\bw+\mu_1\bu,\bv;z_i)
-
f(\bw-\mu_1\bu,\bv;z_i).
\]
Since $h_i(-\bu)=-h_i(\bu)$, we have
$\ebb_{\bu}[h_i(\bu)]=0$. By Lemma~\ref{lem:hu},
\[
\ebb_{\bu}[h_i(\bu)^2]
\leq
\frac{1}{d-1}\ebb_{\bu}
\big[\|\nabla_{\mathbb{S}}h_i(\bu)\|_2^2\big],
\]
where $\nabla_{\mathbb{S}}$ denotes the spherical gradient.
Furthermore,
\begin{align*}
\|\nabla_{\mathbb{S}}h_i(\bu)\|_2^2
&\leq
\mu_1^2
\big\|
\nabla_{\bw}f(\bw+\mu_1\bu,\bv;z_i)
+
\nabla_{\bw}f(\bw-\mu_1\bu,\bv;z_i)
\big\|_2^2\\
&\leq
2\mu_1^2
\Big(
\|\nabla_{\bw}f(\bw+\mu_1\bu,\bv;z_i)\|_2^2
+
\|\nabla_{\bw}f(\bw-\mu_1\bu,\bv;z_i)\|_2^2
\Big).
\end{align*}
Then, by Assumption~\ref{ass:moment},
\begin{align}
&\ebb_{\bu}\bigg[
\frac1n\sum_{i=1}^n
\|\widehat{g}_{\bw}(\bw,\bv;z_i,\bu)\|_2^2
\bigg]=
\frac{d^2}{4\mu_1^2}
\ebb_{\bu}\Big[
\frac1n\sum_{i=1}^n h_i(\bu)^2
\Big]
\leq
\frac{d^2}{4\mu_1^2(d-1)}
\ebb_{\bu}\Big[
\frac1n\sum_{i=1}^n
\|\nabla_{\mathbb{S}}h_i(\bu)\|_2^2
\Big]\notag\\
&\leq
\frac{d^2}{2(d-1)}
\ebb_{\bu}\bigg[
\frac1n\sum_{i=1}^n
\Big(
\|\nabla_{\bw}f(\bw+\mu_1\bu,\bv;z_i)\|_2^2
+
\|\nabla_{\bw}f(\bw-\mu_1\bu,\bv;z_i)\|_2^2
\Big)
\bigg]\notag\\
&\leq
\frac{d^2}{d-1}G_S^2
\leq 2dG_S^2,                                      \label{eq:empirical-zero-w}
\end{align}
where we have used $d/(d-1)\leq2$ for $d\geq2$. Analogously,
\begin{equation}\label{eq:empirical-zero-v}
\ebb_{\bu}\bigg[
\frac1n\sum_{i=1}^n
\|\widehat{g}_{\bv}(\bw,\bv;z_i,\bu)\|_2^2
\bigg]
\leq
\frac{d^2}{d-1}G_S^2
\leq 2dG_S^2.
\end{equation}

We now estimate the right-hand side of
Eq.~\eqref{eq:variance_reduce} by considering the following two
cases.

\textbf{Case 1:} $i_t=i$, which holds with probability $1/n$.
Using $(a+b)^2\leq2a^2+2b^2$,
Eqs.~\eqref{eq:empirical-zero-w} and
\eqref{eq:empirical-zero-v}, and the symmetry between
$(S^{(i)},z_i')$ and $(S,z_i)$, we obtain
\begin{align}
&\ebb\Bigg[
\frac1n\sum_{i=1}^n
\ebb_{\bu}\Bigg[
\Bigg\|
\begin{pmatrix}
\widehat{g}_{\bw}(\bw_t^{(i)},\bv_t^{(i)};z_i',\bu_{1,t,1})
-\widehat{g}_{\bw}(\bw_t,\bv_t;z_i,\bu_{1,t,1})\\
\widehat{g}_{\bv}(\bw_t,\bv_t;z_i,\bu_{2,t,1})
-\widehat{g}_{\bv}(\bw_t^{(i)},\bv_t^{(i)};z_i',\bu_{2,t,1})
\end{pmatrix}
\Bigg\|_2^2
\Bigg]\Bigg]
\leq 16d\ebb[G_S^2].
\label{eq:case1_estimate}
\end{align}

\textbf{Case 2:} $i_t=j\neq i$, which holds with probability
$1-1/n$. Define
\begin{multline*}
h_{\bw,j}^{(i)}(\bu)
=
\big[
f(\bw_t^{(i)}+\mu_1\bu,\bv_t^{(i)};z_j)
-
f(\bw_t^{(i)}-\mu_1\bu,\bv_t^{(i)};z_j)
\big]\\
-
\big[
f(\bw_t+\mu_1\bu,\bv_t;z_j)
-
f(\bw_t-\mu_1\bu,\bv_t;z_j)
\big].
\end{multline*}
Since $h_{\bw,j}^{(i)}(-\bu)=-h_{\bw,j}^{(i)}(\bu)$, we have
$\ebb_{\bu}[h_{\bw,j}^{(i)}(\bu)]=0$. Furthermore,
Lemma~\ref{lem:lip_u_central_local} implies
\[
\|\nabla_{\mathbb{S}}h_{\bw,j}^{(i)}(\bu)\|_2
\leq
2\mu_1L
\Bigg\|
\begin{pmatrix}
\bw_t-\bw_t^{(i)}\\
\bv_t-\bv_t^{(i)}
\end{pmatrix}
\Bigg\|_2.
\]
It then follows from Lemma~\ref{lem:hu} that
\begin{align*}
&\ebb_{\bu}\Big[
\big\|
\widehat{g}_{\bw}(\bw_t^{(i)},\bv_t^{(i)};z_j,\bu)
-
\widehat{g}_{\bw}(\bw_t,\bv_t;z_j,\bu)
\big\|_2^2
\Big]=
\frac{d^2}{4\mu_1^2}
\ebb_{\bu}\big[
h_{\bw,j}^{(i)}(\bu)^2
\big]\\
& \leq
\frac{d^2}{4\mu_1^2(d-1)}
\ebb_{\bu}\big[
\|\nabla_{\mathbb{S}}h_{\bw,j}^{(i)}(\bu)\|_2^2
\big]
\leq
\frac{d^2L^2}{d-1}
\Bigg\|
\begin{pmatrix}
\bw_t-\bw_t^{(i)}\\
\bv_t-\bv_t^{(i)}
\end{pmatrix}
\Bigg\|_2^2
\leq
2dL^2
\Bigg\|
\begin{pmatrix}
\bw_t-\bw_t^{(i)}\\
\bv_t-\bv_t^{(i)}
\end{pmatrix}
\Bigg\|_2^2.
\end{align*}
The same argument applies to the estimator with respect to $\bv$.
Then,
\begin{align}
\ebb_{\bu}\Bigg[
\Bigg\|
\begin{pmatrix}
\widehat{g}_{\bw}(\bw_t^{(i)},\bv_t^{(i)};z_j,\bu_{1,t,1})
-\widehat{g}_{\bw}(\bw_t,\bv_t;z_j,\bu_{1,t,1})\\
\widehat{g}_{\bv}(\bw_t,\bv_t;z_j,\bu_{2,t,1})
-\widehat{g}_{\bv}(\bw_t^{(i)},\bv_t^{(i)};z_j,\bu_{2,t,1})
\end{pmatrix}
\Bigg\|_2^2
\Bigg]
\leq
4L^2d
\Bigg\|
\begin{pmatrix}
\bw_t-\bw_t^{(i)}\\
\bv_t-\bv_t^{(i)}
\end{pmatrix}
\Bigg\|_2^2.
\label{eq:case2_estimate}
\end{align}

Combining the two cases with
Eq.~\eqref{eq:variance_reduce}, followed by taking an average over
$i\in[n]$ and expectation, we obtain
\begin{multline}\label{zero-4-mu_free}
\ebb\Bigg[
\frac1n\sum_{i=1}^n
\Bigg\|
\begin{pmatrix}
g_{\bw}(\bw_t^{(i)},\bv_t^{(i)};z_{i_t}^{(i)})
-g_{\bw}(\bw_t,\bv_t;z_{i_t})\\
g_{\bv}(\bw_t,\bv_t;z_{i_t})
-g_{\bv}(\bw_t^{(i)},\bv_t^{(i)};z_{i_t}^{(i)})
\end{pmatrix}
-
\ebb_{\bu}\Bigg[
\begin{pmatrix}
g_{\bw}(\bw_t^{(i)},\bv_t^{(i)};z_{i_t}^{(i)})
-g_{\bw}(\bw_t,\bv_t;z_{i_t})\\
g_{\bv}(\bw_t,\bv_t;z_{i_t})
-g_{\bv}(\bw_t^{(i)},\bv_t^{(i)};z_{i_t}^{(i)})
\end{pmatrix}
\Bigg]
\Bigg\|_2^2
\Bigg]\\
\leq
\frac{4L^2d}{K}
\ebb\Bigg[
\frac1n\sum_{i=1}^n
\Bigg\|
\begin{pmatrix}
\bw_t-\bw_t^{(i)}\\
\bv_t-\bv_t^{(i)}
\end{pmatrix}
\Bigg\|_2^2
\Bigg]
+
\frac{16d\,\ebb[G_S^2]}{Kn}.
\end{multline}

Furthermore, we have
\begin{align}\label{zero-5}
& \ebb\Bigg[\frac{1}{n}\sum_{i=1}^n
\Bigg\|\ebb_{\bu}\Bigg[\begin{pmatrix}
  g_{\bw}(\bw_t^{(i)},\bv_t^{(i)};z_{i_t}^{(i)})
  -g_{\bw}(\bw_t,\bv_t;z_{i_t})\\
  g_{\bv}(\bw_t,\bv_t;z_{i_t})
  -g_{\bv}(\bw_t^{(i)},\bv_t^{(i)};z^{(i)}_{i_t})
  \end{pmatrix}\Bigg]\Bigg\|_2^2\Bigg]\notag\\
& = \ebb\Bigg[\frac{1}{n}\sum_{i=1}^n
\Bigg\|\begin{pmatrix}
  \nabla_{\bw} \tilde{f}_1(\bw_t^{(i)},\bv_t^{(i)};z_{i_t}^{(i)})
  -\nabla_{\bw} \tilde{f}_1(\bw_t,\bv_t;z_{i_t})\\
  \nabla_{\bv} \tilde{f}_2(\bw_t,\bv_t;z_{i_t})
  -\nabla_{\bv} \tilde{f}_2(\bw_t^{(i)},\bv_t^{(i)};z^{(i)}_{i_t})
  \end{pmatrix}\Bigg\|_2^2\Bigg]\notag\\
& = \frac{1}{n^2}\sum_{i=1}^n \sum_{j\neq i}
\ebb\Bigg[\Bigg\|\begin{pmatrix}
  \nabla_{\bw} \tilde{f}_1(\bw_t^{(i)},\bv_t^{(i)};z_j)
  -\nabla_{\bw} \tilde{f}_1(\bw_t,\bv_t;z_j)\\
  \nabla_{\bv} \tilde{f}_2(\bw_t,\bv_t;z_j)
  -\nabla_{\bv} \tilde{f}_2(\bw_t^{(i)},\bv_t^{(i)};z_j)
  \end{pmatrix}\Bigg\|_2^2\Bigg]\notag\\
&\quad+\frac{1}{n^2}\sum_{i=1}^n
\ebb\Bigg[\Bigg\|\begin{pmatrix}
  \nabla_{\bw} \tilde{f}_1(\bw_t^{(i)},\bv_t^{(i)};z_i')
  -\nabla_{\bw} \tilde{f}_1(\bw_t,\bv_t;z_i)\\
  \nabla_{\bv} \tilde{f}_2(\bw_t,\bv_t;z_i)
  -\nabla_{\bv} \tilde{f}_2(\bw_t^{(i)},\bv_t^{(i)};z_i')
  \end{pmatrix}\Bigg\|_2^2\Bigg]\notag\\
&\leq
L^2\ebb\Bigg[
\frac{1}{n}\sum_{i=1}^n
\Bigg\|\begin{pmatrix}
\bw_t-\bw_t^{(i)}\\
\bv_t-\bv_t^{(i)}
\end{pmatrix}\Bigg\|_2^2
\Bigg]
+\frac{8\ebb[G_S^2]}{n},
\end{align}
where we use the smoothness of $\tilde{f}_1,\tilde{f}_2$ (Lemma \ref{lem:unbiase-zeroth}) for the first term in the identity,
\begin{align*}
\Bigg\|\begin{pmatrix}
  \nabla_{\bw} \tilde{f}_1(\bw_t^{(i)},\bv_t^{(i)};z_j) - \nabla_{\bw} \tilde{f}_1(\bw_t,\bv_t;z_j)\\
  \nabla_{\bv} \tilde{f}_2(\bw_t,\bv_t;z_j) - \nabla_{\bv} \tilde{f}_2(\bw_t^{(i)},\bv_t^{(i)};z_j)
  \end{pmatrix}\Bigg\|_2^2
\leq L^2 \Bigg\|\begin{pmatrix}
                \bw_t-\bw_t^{(i)} \\
                \bv_t-\bv_t^{(i)}
            \end{pmatrix}\Bigg\|_2^2,
\end{align*}
and Lemma \ref{lem:tilde-difference} for the second term in the last identity.

Finally, taking an average over $i\in[n]$ and expectation in
Eq.~\eqref{condition-a-zero}, and then plugging
Eqs.~\eqref{zero-4-mu_free} and~\eqref{zero-5} into it, we obtain
\begin{align*}
&\ebb\Bigg[
\frac1n\sum_{i=1}^n
\Bigg\|
\begin{pmatrix}
g_{\bw}(\bw_t^{(i)},\bv_t^{(i)};z_{i_t}^{(i)})
-g_{\bw}(\bw_t,\bv_t;z_{i_t})\\
g_{\bv}(\bw_t,\bv_t;z_{i_t})
-g_{\bv}(\bw_t^{(i)},\bv_t^{(i)};z_{i_t}^{(i)})
\end{pmatrix}
\Bigg\|_2^2
\Bigg]\\
&\leq
L^2\Big(1+\frac{4d}{K}\Big)
\ebb\Bigg[
\frac1n\sum_{i=1}^n
\Bigg\|
\begin{pmatrix}
\bw_t-\bw_t^{(i)}\\
\bv_t-\bv_t^{(i)}
\end{pmatrix}
\Bigg\|_2^2
\Bigg]
+
\frac{8\ebb[G_S^2]}{n}
\Big(1+\frac{2d}{K}\Big).
\end{align*}
The proof is completed.
\end{proof}

We then apply Lemma \ref{lem:condition-a-zero} and Lemma \ref{lem:tilde-difference} to derive Theorem \ref{cor:zero-stability}. We omit the details for brevity.

We are now ready to control the optimization error for Zeroth-order SGDA, which
requires the following lemmas and we show the proofs here.

\begin{lemma}\label{lem:lip_u}
If $f$ is $G_z$-Lipschitz continuous w.r.t $\bw$, then the function
\begin{align*}
h(\bu) = f(\bw+\mu \bu, \bv;z) - f(\bw-\mu \bu, \bv;z)
\end{align*}
with $\bu \in \mathbb{S}^{d-1}$ is $2\mu G_z$-Lipschitz continuous w.r.t $\bu$.
\end{lemma}
\begin{proof}
Take arbitrary $\bu_1, \bu_2 \in \mathbb{S}^{d-1}$. Decompose the difference:
\begin{align*}
& h(\bu_1) - h(\bu_2)\\
&= \left[f(\bw+\mu\bu_1,\bv;z) - f(\bw-\mu\bu_1,\bv;z)\right] - \left[f(\bw+\mu\bu_2,\bv;z) - f(\bw-\mu\bu_2,\bv;z)\right] \\
&= \left[f(\bw+\mu\bu_1,\bv;z) - f(\bw+\mu\bu_2,\bv;z)\right] + \left[f(\bw-\mu\bu_2,\bv;z) - f(\bw-\mu\bu_1,\bv;z)\right].
\end{align*}
For the first term: by $G_z$-Lipschitz continuity of $f$,
  $$\left\|f(\bw+\mu\bu_1,\bv;z)-f(\bw+\mu\bu_2,\bv;z)\right\| \leq G_z\|\mu\bu_1 - \mu\bu_2\| = \mu G_z\|\bu_1 - \bu_2\|.$$
For the second term: similarly,
  $$\left\|f(\bw-\mu\bu_2,\bv;z)-f(\bw-\mu\bu_1,\bv;z)\right\| \leq G_z\|-\mu\bu_2 + \mu\bu_1\| = \mu G_z\|\bu_1 - \bu_2\|.$$
Combining via triangle inequality gives $|h(\bu_1)-h(\bu_2)| \leq 2\mu G_z\|\bu_1 - \bu_2\|$.
\end{proof}

\begin{lemma}\label{lem:optimization-g-l2-lip}
Let Assumption~\ref{ass:moment} hold. Let $g_{\bw}(\bw_t,\bv_t;z_{i_t})$ and $g_{\bv}(\bw_t,\bv_t;z_{i_t})$ be defined in Eq.~\eqref{g-zero}. Then
\begin{align*}
\ebb_{i_t,\bu}[\|g_{\bw}(\bw_t,\bv_t;z_{i_t})\|_2^2]
&\leq \ebb_{i_t}[\|\nabla_{\bw} \tilde{f}_1(\bw_t,\bv_t;z_{i_t})\|_2^2] + \frac{2dG_S^2}{K},\\
\ebb_{i_t,\bu}[\|g_{\bv}(\bw_t,\bv_t;z_{i_t})\|_2^2]
&\leq \ebb_{i_t}[\|\nabla_{\bv} \tilde{f}_2(\bw_t,\bv_t;z_{i_t})\|_2^2] + \frac{2dG_S^2}{K}.
\end{align*}
\end{lemma}
\begin{proof}
Since $\ebb[X^2]=\ebb[(X-\ebb[X])^2]+(\ebb[X])^2$ for a random variable $X$, we know
\begin{align}\label{g-l2-1-lip}
\ebb_{\bu}[\|g_{\bw}(\bw_t,\bv_t;z_{i_t})\|_2^2]
= \|\ebb_{\bu}[g_{\bw}(\bw_t,\bv_t;z_{i_t})]\|_2^2
+ \ebb_{\bu}[\|g_{\bw}(\bw_t,\bv_t;z_{i_t})-
\ebb_{\bu}[g_{\bw}(\bw_t,\bv_t;z_{i_t})]\|_2^2].
\end{align}
Since an average of $K$ independent random variables divides the variance by a factor of $K$, we know
\begin{multline}\label{eq:average_k}
\ebb_{\bu}[\|g_{\bw}(\bw_t,\bv_t;z_{i_t})-
\ebb_{\bu}[g_{\bw}(\bw_t,\bv_t;z_{i_t})]\|_2^2]\leq\\ \frac{1}{K} \ebb_{\bu}\Big[\Big\|\frac{f(\bw_t+\mu_1\bu_{1,t,1},\bv_t ; z_{i_t})-f(\bw_t-\mu_1\bu_{1,t,1},\bv_t; z_{i_t})}{2\mu_1/d} \bu_{1,t,1}\Big\|_2^2\Big].
\end{multline}
For each $i\in[n]$, define
\[
h_i(\bu)=f(\bw_t+\mu_1\bu,\bv_t;z_i)-f(\bw_t-\mu_1\bu,\bv_t;z_i).
\]
Since $h_i(-\bu)=-h_i(\bu)$, we have $\ebb_{\bu}[h_i(\bu)]=0$. By
Lemma~\ref{lem:hu},
\[
\ebb_{\bu}[h_i(\bu)^2]
\leq
\frac{1}{d-1}\ebb_{\bu}\big[\|\nabla_{\mathbb{S}}h_i(\bu)\|_2^2\big].
\]
By the definition in Lemma~\ref{lem:hu},
\begin{align*}
\nabla_{\mathbb{S}}h_i(\bu)
&=
(\mathbf{I}_d-\bu\bu^\top)\nabla_{\bu} h_i(\bu)\\
&= (\mathbf{I}_d-\bu\bu^\top)\big(\mu_1\nabla_{\bw}f(\bw_t+\mu_1\bu,\bv_t;z_i)
+
\mu_1\nabla_{\bw}f(\bw_t-\mu_1\bu,\bv_t;z_i)\big).
\end{align*}
Note
$\|(\mathbf{I}_d-\bu\bu^\top)\mathbf{y}\|_2\leq \|\mathbf{y}\|_2$ for any $\mathbf{y}\in\mathbb{R}^d$ and $\mathbf{u}\in\mathbb{S}^{d-1}$.
Applying this with $\mathbf{y}=\nabla_{\bu}h_i(\bu)$ and using
$(a+b)^2\leq 2a^2+2b^2$, we obtain
\begin{align*}
\|\nabla_{\mathbb{S}}h_i(\bu)\|_2^2
&\leq
\|\nabla_{\bu}h_i(\bu)\|_2^2\\
&=
\mu_1^2
\big\|
\nabla_{\bw}f(\bw_t+\mu_1\bu,\bv_t;z_i)
+
\nabla_{\bw}f(\bw_t-\mu_1\bu,\bv_t;z_i)
\big\|_2^2\\
&\leq
2\mu_1^2
\Big(
\|\nabla_{\bw}f(\bw_t+\mu_1\bu,\bv_t;z_i)\|_2^2
+
\|\nabla_{\bw}f(\bw_t-\mu_1\bu,\bv_t;z_i)\|_2^2
\Big).
\end{align*}
Therefore,
\begin{align*}
&\ebb_{i_t,\bu}\Big[\Big\|\frac{f(\bw_t+\mu_1\bu_{1,t,1},\bv_t ; z_{i_t})-f(\bw_t-\mu_1\bu_{1,t,1},\bv_t; z_{i_t})}{2\mu_1/d} \bu_{1,t,1}\Big\|_2^2\Big]\\
&=
\frac{d^2}{4\mu_1^2}
\ebb_{i_t,\bu}[h_{i_t}(\bu)^2]
=
\frac{d^2}{4\mu_1^2}
\cdot\frac{1}{n}\sum_{i=1}^n\ebb_{\bu}[h_i(\bu)^2]\leq
\frac{d^2}{4\mu_1^2(d-1)}
\cdot\frac{1}{n}\sum_{i=1}^n\ebb_{\bu}\big[\|\nabla_{\mathbb{S}}h_i(\bu)\|_2^2\big]\\
&\leq
\frac{d^2}{2(d-1)}
\cdot\frac{1}{n}\sum_{i=1}^n\ebb_{\bu}\Big[
\|\nabla_{\bw}f(\bw_t+\mu_1\bu,\bv_t;z_i)\|_2^2
+
\|\nabla_{\bw}f(\bw_t-\mu_1\bu,\bv_t;z_i)\|_2^2
\Big]\\
&\leq
\frac{d^2}{d-1}G_S^2
\leq 2dG_S^2,
\end{align*}
where we use Assumption~\ref{ass:moment} and $d/(d-1)\leq 2$ for $d\geq 2$. Plugging into Eq.~\eqref{eq:average_k}, we get
\begin{align}\label{g-l2-3-lip}
\ebb_{i_t,\bu}[\|g_{\bw}(\bw_t,\bv_t;z_{i_t})-
\ebb_{\bu}[g_{\bw}(\bw_t,\bv_t;z_{i_t})]\|_2^2]
\leq \frac{2dG_S^2}{K}.
\end{align}
Furthermore, we know
\begin{align}\label{g-l2-4-lip}
\|\ebb_{\bu}[g_{\bw}(\bw_t,\bv_t;z_{i_t})]\|_2^2
= \|\nabla_{\bw} \tilde{f}_1(\bw_t,\bv_t;z_{i_t})\|_2^2.
\end{align}
Taking expectation over $i_t$ in Eq.~\eqref{g-l2-1-lip}, and plugging Eq.~\eqref{g-l2-3-lip} and Eq.~\eqref{g-l2-4-lip} into it, we know
\begin{align*}
\ebb_{i_t,\bu}[\|g_{\bw}(\bw_t,\bv_t;z_{i_t})\|_2^2]
\leq \ebb_{i_t}[\|\nabla_{\bw} \tilde{f}_1(\bw_t,\bv_t;z_{i_t})\|_2^2] + \frac{2dG_S^2}{K}.
\end{align*}
The proof for $\ebb_{i_t,\bu}[\|g_{\bv}(\bw_t,\bv_t;z_{i_t})\|_2^2]$ holds similarly.
\end{proof}

The following lemma presents convergence rate analysis for Zeroth-order SGDA.

\begin{proof}[Proof of Lemma~\ref{lem:zero-optimization}]
Note that the projection step is nonexpansive. For any $\bw \in B_R$, we know
\begin{align*}
& \|\bw_{t+1}-\bw\|_2^2
\leq \|\bw_{t}-\eta_t g_{\bw}(\bw_t,\bv_t;z_{i_t})-\bw\|_2^2\\
& = \|\bw_t-\bw\|_2^2 - 2\eta_t \langle \bw_t-\bw,g_{\bw}(\bw_t,\bv_t;z_{i_t})\rangle + \eta_t^2 \|g_{\bw}(\bw_t,\bv_t;z_{i_t})\|_2^2\\
& = \|\bw_t-\bw\|_2^2 + 2\eta_t \langle \bw-\bw_t,\nabla \widetilde{F}_{1,S}(\bw_t,\bv_t) \rangle + 2\eta_t \langle \bw-\bw_t,g_{\bw}(\bw_t,\bv_t;z_{i_t}) - \nabla \widetilde{F}_{1,S}(\bw_t,\bv_t) \rangle \\
& + \eta_t^2 \|g_{\bw}(\bw_t,\bv_t;z_{i_t})\|_2^2\\
& \leq \|\bw_t-\bw\|_2^2+2\eta_t \big(\widetilde{F}_{1,S}(\bw,\bv_t)-\widetilde{F}_{1,S}(\bw_t,\bv_t)\big)  \\
& + 2\eta_t \langle \bw-\bw_t,g_{\bw}(\bw_t,\bv_t;z_{i_t}) - \nabla \widetilde{F}_{1,S}(\bw_t,\bv_t) \rangle
+ \eta_t^2 \|g_{\bw}(\bw_t,\bv_t;z_{i_t})\|_2^2,
\end{align*}
where we have used the convexity of $\widetilde{F}_{1,S}$ established in Lemma \ref{lem:unbiase-zeroth}.
Taking a summation of the above inequality from $t=1$ to $t=T$ ($\bw_1=0$), we derive
\begin{multline*}
\sum_{t=1}^{T}2\eta_t(\widetilde{F}_{1,S}(\bw_t,\bv_t)-\widetilde{F}_{1,S}(\bw,\bv_t))\leq \|\bw\|_2^2+
\sum_{t=1}^{T}\eta_t^2\|g_{\bw}(\bw_t,\bv_t;z_{i_t})\|_2^2\\
+ 2 \sum_{t=1}^T\eta_t \langle \bw-\bw_t,g_{\bw}(\bw_t,\bv_t;z_{i_t}) - \nabla \widetilde{F}_{1,S}(\bw_t,\bv_t) \rangle.
\end{multline*}
Since $\eta_t=\eta$, we further get
\begin{multline*}
\sum_{t=1}^{T}(\widetilde{F}_{1,S}(\bw_t,\bv_t)-\widetilde{F}_{1,S}(\bw,\bv_t)) \leq \frac{1}{2}\eta^{-1}\|\bw\|_2^2+
\frac{1}{2}\eta\sum_{t=1}^{T}\|g_{\bw}(\bw_t,\bv_t;z_{i_t})\|_2^2\\
+ \sum_{t=1}^T \langle \bw_t,\nabla \widetilde{F}_{1,S}(\bw_t,\bv_t)-g_{\bw}(\bw_t,\bv_t;z_{i_t}) \rangle
+ \sum_{t=1}^T \langle \bw,g_{\bw}(\bw_t,\bv_t;z_{i_t}) - \nabla \widetilde{F}_{1,S}(\bw_t,\bv_t) \rangle.
\end{multline*}
It then follows from the concavity of $\widetilde{F}_{1,S}(\bw,\cdot)$ and Schwartz's inequality that
\begin{multline*}
\sum_{t=1}^{T}(\widetilde{F}_{1,S}(\bw_t,\bv_t)-\widetilde{F}_{1,S}(\bw,\bar{\bv}_T))\leq \frac{1}{2}\eta^{-1}\|\bw\|_2^2 +
\frac{1}{2}\eta\sum_{t=1}^{T}\|g_{\bw}(\bw_t,\bv_t;z_{i_t})\|_2^2\\
+ \sum_{t=1}^T \langle \bw_t,\nabla \widetilde{F}_{1,S}(\bw_t,\bv_t)-g_{\bw}(\bw_t,\bv_t;z_{i_t}) \rangle
+ \|\bw\|_2 \Big\|\sum_{t=1}^T \big(g_{\bw}(\bw_t,\bv_t;z_{i_t}) - \nabla \widetilde{F}_{1,S}(\bw_t,\bv_t)\big)\Big\|_2.
\end{multline*}
Since the above inequality holds for all $\bw \in B_R$, we further get
\begin{multline}\label{zero-optimization-1-lip}
\sum_{t=1}^{T}(\widetilde{F}_{1,S}(\bw_t,\bv_t)-\inf_{\bw \in B_R}\widetilde{F}_{1,S}(\bw,\bar{\bv}_T))\leq \frac{1}{2}\eta^{-1}R^2+
\frac{1}{2}\eta\sum_{t=1}^{T}\|g_{\bw}(\bw_t,\bv_t;z_{i_t})\|_2^2+ \\
\sum_{t=1}^T \langle \bw_t,\nabla \widetilde{F}_{1,S}(\bw_t,\bv_t)-g_{\bw}(\bw_t,\bv_t;z_{i_t}) \rangle
+ R \Big\|\sum_{t=1}^T \big(g_{\bw}(\bw_t,\bv_t;z_{i_t}) - \nabla \widetilde{F}_{1,S}(\bw_t,\bv_t)\big)\Big\|_2.
\end{multline}
According to Lemma~\ref{lem:optimization-g-l2-lip} and Assumption~\ref{ass:moment},
\begin{align}\label{zero-optimization-gradient-norm-lip}
\ebb[\|g_{\bw}(\bw_t,\bv_t;z_{i_t})\|_2^2] \leq
\ebb[\|\nabla_{\bw} \tilde{f}_1(\bw_t,\bv_t;z_{i_t})\|_2^2] + \frac{cd\ebb[G_S^2]}{K}.
\end{align}
Then, taking an expectation on both sides of Eq.~\eqref{zero-optimization-1-lip}, noting Eq.~\eqref{zero-optimization-gradient-norm-lip} and
\begin{align}\label{exp-0-lip}
\ebb_{i_t}[\langle \bw_t,\nabla \widetilde{F}_{1,S}(\bw_t,\bv_t)-g_{\bw}(\bw_t,\bv_t;z_{i_t}) \rangle] = 0,
\end{align}
we have
\begin{align}\label{zero-optimization-2-lip}
&\sum_{t=1}^{T}\ebb[\widetilde{F}_{1,S}(\bw_t,\bv_t)-\inf_{\bw \in B_R} \widetilde{F}_{1,S}(\bw,\bar{\bv}_T)]\notag\\
& \leq \frac{R^2}{2\eta}
+ R \ebb\Big[\Big\|\sum_{t=1}^T \big(g_{\bw}(\bw_t,\bv_t;z_{i_t}) - \nabla \widetilde{F}_{1,S}(\bw_t,\bv_t)\big)\Big\|_2\Big]
+
\frac{\eta}{2} \sum_{t=1}^T \big(\ebb[\|\nabla_{\bw} \tilde{f}_1(\bw_t,\bv_t;z_{i_t})\|_2^2] + \frac{2d\ebb[G_S^2]}{K}\big) \notag\\
& \leq \frac{1}{2}\eta^{-1}R^2
+ \frac{d\ebb[G_S^2]}{K} \eta T
+ R \ebb\Big[\Big\|\sum_{t=1}^T \big(g_{\bw}(\bw_t,\bv_t;z_{i_t}) - \nabla \widetilde{F}_{1,S}(\bw_t,\bv_t)\big)\Big\|_2\Big]\notag\\
&+
\frac{1}{2}\eta \sum_{t=1}^T \ebb[\|\nabla_{\bw} \tilde{f}_1(\bw_t,\bv_t;z_{i_t})\|_2^2].
\end{align}
According to Jensen's inequality and $\ebb[(X-\ebb[X])^2] \leq \ebb[X^2]$ (note that $\ebb_{\bu,i_t}[g_{\bw}(\bw_t,\bv_t;z_{i_t})] = \nabla \widetilde{F}_{1,S}(\bw_t,\bv_t)$ by Lemma~\ref{lem:unbiase-zeroth}), we know
\begin{align*}
&\bigg(\ebb\Big[\Big\|\sum_{t=1}^T \big(g_{\bw}(\bw_t,\bv_t;z_{i_t}) - \nabla \widetilde{F}_{1,S}(\bw_t,\bv_t)\big)\Big\|_2\Big]\bigg)^2
\leq \ebb\Big[\Big\|\sum_{t=1}^T \big(g_{\bw}(\bw_t,\bv_t;z_{i_t}) - \nabla \widetilde{F}_{1,S}(\bw_t,\bv_t)\big)\Big\|_2^2\Big]\\
& = \sum_{t=1}^T\ebb\Big[\Big\| g_{\bw}(\bw_t,\bv_t;z_{i_t}) - \nabla \widetilde{F}_{1,S}(\bw_t,\bv_t)\Big\|_2^2\Big]
\leq \sum_{t=1}^T\ebb[\|g_{\bw}(\bw_t,\bv_t;z_{i_t})\|_2^2],
\end{align*}
from which and Eq.~\eqref{zero-optimization-gradient-norm-lip}, we get
\begin{align*}
&R \ebb\Big[\Big\|\sum_{t=1}^T \big(g_{\bw}(\bw_t,\bv_t;z_{i_t}) \!-\! \nabla \widetilde{F}_{1,S}(\bw_t,\bv_t)\big)\Big\|_2\Big]
\!\leq\! \frac{R^2}{4\eta} \!+\! \bigg(\!\ebb\Big[\Big\|\sum_{t=1}^T \big(g_{\bw}(\bw_t,\bv_t;z_{i_t}) \!-\! \nabla \widetilde{F}_{1,S}(\bw_t,\bv_t)\big)\Big\|_2\Big]\!\bigg)^2 \eta\\
& \leq \frac{\eta^{-1}R^2}{4} + \eta \sum_{t=1}^T\ebb[\|g_{\bw}(\bw_t,\bv_t;z_{i_t})\|_2^2]
\leq \frac{\eta^{-1}R^2}{4}+
\eta \sum_{t=1}^T \ebb[\|\nabla_{\bw} \tilde{f}_1(\bw_t,\bv_t;z_{i_t})\|_2^2] + \frac{2d\ebb[G_S^2]}{K}\eta T,
\end{align*}
where we use Cauchy-Schwarz's inequality for the first step.
Plugging the above inequality into Eq.~\eqref{zero-optimization-2-lip}, we have
\begin{align*}
\sum_{t=1}^{T}\ebb[\widetilde{F}_{1,S}(\bw_t,\bv_t)-\inf_{\bw \in B_R}\widetilde{F}_{1,S}(\bw,\bar{\bv}_T)]
\leq & \frac{3R^2}{4\eta}+\frac{3d\ebb[G_S^2]\eta T}{K} + \frac{3}{2}\eta \sum_{t=1}^T \ebb[\|\nabla_{\bw} \tilde{f}_1(\bw_t,\bv_t;z_{i_t})\|_2^2]\\
\leq & \frac{3R^2}{4\eta}+\frac{3d\ebb[G_S^2]}{K}\eta T + \frac{3}{2} \ebb[G_S^2] \eta T,
\end{align*}
where we use Lemma~\ref{lem:unbiase-zeroth} and Assumption~\ref{ass:moment} for the last step.
This yields
\begin{align*}
\frac{1}{T}\sum_{t=1}^{T}\ebb[\widetilde{F}_{1,S}(\bw_t,\bv_t)]-\ebb[\inf_{\bw \in B_R} \widetilde{F}_{1,S}(\bw,\bar{\bv}_T)]
\lesssim \frac{R^2}{\eta T}+ \big(1+\frac{d}{K}\big)\ebb[G_S^2] \eta.
\end{align*}
Moreover, by the mean-value theorem and Cauchy--Schwarz's inequality, we know
\begin{align*}
\big|\widetilde{F}_{1,S}(\bw,\bv) - F_S(\bw,\bv)\big|
&=
\Big|
\frac{1}{n}\sum_{i=1}^n\ebb_{\bu}\big[
f(\bw+\mu_1\bu,\bv;z_i)-f(\bw,\bv;z_i)
\big]
\Big|\\
&\leq
\mu_1\Big(
\frac{1}{n}\sum_{i=1}^n\ebb_{\bu}\big[
\|\nabla_{\bw}f(\bw+\theta_{i}\mu_1\bu,\bv;z_i)\|_2^2
\big]
\Big)^{\frac{1}{2}}
\leq \mu_1 G_S,
\end{align*}
for some $\theta_{i}\in[0,1]$. Taking expectation then gives
\[
\big|\ebb[\widetilde{F}_{1,S}(\bw,\bv) - F_S(\bw,\bv)]\big|
\leq \mu_1\sqrt{\ebb[G_S^2]}.
\]
Then,
\begin{align*}
\frac{1}{T}\sum_{t=1}^{T}\ebb[F_S(\bw_t,\bv_t)]-\ebb[\inf_{\bw \in B_R} F_S(\bw,\bar{\bv}_T)]
\lesssim \frac{R^2}{\eta T}+ \big(1+\frac{d}{K}\big)\ebb[G_S^2] \eta + \mu_1\sqrt{\ebb[G_S^2]}.
\end{align*}
In a similar way, we can show that
\begin{align*}
\ebb[\sup_{\bv \in B_R} F_S(\bar{\bw}_T,\bv)]-\frac{1}{T}\sum_{t=1}^{T}\ebb[F_S(\bw_t,\bv_t)]
\lesssim \frac{R^2}{\eta T}+ \big(1+\frac{d}{K}\big)\ebb[G_S^2] \eta + \mu_2\sqrt{\ebb[G_S^2]}.
\end{align*}
Then,
\begin{align*}
\ebb[\sup_{\bv \in B_R}F_S(\bar{\bw}_T,\bv)-\inf_{\bw \in B_R} F_S(\bw,\bar{\bv}_T)] \lesssim \frac{R^2}{\eta T}+ \big(1+\frac{d}{K}\big)\ebb[G_S^2] \eta + \mu\sqrt{\ebb[G_S^2]}.
\end{align*}
The proof is completed.
\end{proof}

Combining Theorem~\ref{cor:zero-stability} and Lemma~\ref{lem:zero-optimization}, we derive the following theorem.
\begin{proof}[Proof of Theorem~\ref{thm:stab-zero}]
Define
\begin{align*}
\triangle^{w,R}_S({\bw},{\bv})=\sup_{\bv'\in B_R}\ebb\big[F_S({\bw},\bv')\big]-\inf_{\bw'\in B_R}\ebb\big[F_S(\bw',{\bv})\big].
\end{align*}
We have the decomposition
\begin{align}\label{decomp}
\Delta^{w,R}(\bar{\mathbf{w}}_T, \bar{\mathbf{v}}_T)=\Delta^{w,R}(\bar{\mathbf{w}}_T, \bar{\mathbf{v}}_T)-\Delta_S^{w,R}(\bar{\mathbf{w}}_T, \bar{\mathbf{v}}_T)+\Delta_S^{w,R}(\bar{\mathbf{w}}_T, \bar{\mathbf{v}}_T) .
\end{align}
From Theorem~\ref{cor:zero-stability}, we know
\begin{align*}
\epsilon
\lesssim \Big(\frac{\sqrt{\ebb[G_S^2]}}{\sqrt{n}}\Big(1+\frac{d^{\frac{1}{2}}}{K^{\frac{1}{2}}}\Big)\Big) \eta \sqrt{T}
+ \frac{\sqrt{\ebb[G_S^2]}}{n}\eta T
\lesssim
\frac{\sqrt{\ebb[G_S^2]}/L}{\sqrt{n}}
+\frac{\sqrt{\ebb[G_S^2]}}{n}\eta T,
\end{align*}
where we used $L^2\big(1+\frac{2cd}{K}\big)\eta^2 \leq \frac{1}{T}$ in Theorem~\ref{cor:zero-stability}.

Furthermore, Assumption~\ref{ass:moment} implies that the data- and algorithm-dependent constant $G_{S,A}$ in Assumption~\ref{ass:lipschitz} satisfies
$\ebb_{S,A}[G_{S,A}^2]\leq \ebb[G_S^2]$.
According to Theorem~\ref{thm:stab-gen} and Eq.~\eqref{thm:zero-risk-condition}
$
L\epsilon^2
\lesssim
\sqrt{\ebb[G_S^2]}\epsilon$,
we know
\begin{align*}
\triangle^{w,R}(A_{\bw}(S),A_{\bv}(S))
-\triangle^{w,R}_S(A_{\bw}(S),A_{\bv}(S))
\lesssim
\frac{\ebb[G_S^2]/L}{\sqrt{n}}
+\frac{\ebb[G_S^2]}{n}\eta T.
\end{align*}
According to Lemma~\ref{lem:zero-optimization}, we know
\begin{align*}
\Delta_S^{w,R}(\bar{\mathbf{w}}_T, \bar{\mathbf{v}}_T)
&\lesssim
\frac{R^2}{\eta T}
+\Big(1+\frac{d}{K}\Big)\ebb[G_S^2]\eta
+\mu\sqrt{\ebb[G_S^2]}\lesssim
\frac{R^2}{\eta T}
+\frac{\ebb[G_S^2]}{L^2\eta T}
+\mu\sqrt{\ebb[G_S^2]},
\end{align*}
where we use the condition $L^2\big(1+\frac{4d}{K}\big)\eta^2 \leq 1/T$ for the last step.
Combining the above two inequalities with Eq.~\eqref{decomp}, we obtain
\begin{align*}
\Delta^{w,R}(\bar{\mathbf{w}}_T,\bar{\mathbf{v}}_T)
\lesssim&
\frac{\ebb[G_S^2]/L}{\sqrt{n}}
+\frac{\ebb[G_S^2]}{n}\eta T
+\frac{R^2}{\eta T}
+\frac{\ebb[G_S^2]}{L^2\eta T}
+\mu\sqrt{\ebb[G_S^2]}.
\end{align*}
The bound then follows directly.
\end{proof}

\subsubsection{Clipped-SGDA\label{sec:proof-clipped-sgda}}

To prove verify Eq.~\eqref{simple-condition-a} for Clipped-SGDA, we introduce the following lemma which shows that the clipping operator is non-expansive. Actually, one can show that the clipping operator is a projector onto a convex set. We omit the proof for simplicity.

\begin{lemma}\label{lem:clip_diff}
For any $\bv_1,\bv_2\in\wcal$, we have
$
\|\clip(\bv_1,\tau)-\clip(\bv_2,\tau)\|_2 \leq \|\bv_1-\bv_2\|_2.
$
\end{lemma}

Then, we apply Lemma~\ref{lem:clip_diff} to verify Eq.~\eqref{simple-condition-a} for Clipped-SGDA in the following lemma.

\begin{lemma}\label{lem:condition-a-clip}
Let $g(\bw_t,\bv_t;z_{i_t})$ be given in Eq.~\eqref{g-clip}, and Assumptions \ref{ass:smooth}, \ref{ass:moment} hold with $G_S\leq G$ for some $G>0$. Then, Eq. \eqref{simple-condition-a} holds with $A = L^2,B = \frac{8G^2}{n}$.
\end{lemma}

\begin{proof}[Proof of Lemma~\ref{lem:condition-a-clip}]
According to Eq.~\eqref{g-clip} and Lemma \ref{lem:clip_diff}, we know
\begin{align}\label{condition-a-clip}
& \Bigg\|\begin{pmatrix}
g_{\bw}(\bw_t^{(i)},\bv_t^{(i)};z_{i_t}^{(i)}) - g_{\bw}(\bw_t,\bv_t;z_{i_t})\\
g_{\bv}(\bw_t,\bv_t;z_{i_t}) - g_{\bv}(\bw_t^{(i)},\bv_t^{(i)};z^{(i)}_{i_t})
\end{pmatrix}\Bigg\|_2^2 \notag\\
& \leq \|\nabla_{\bw} f (\bw_t^{(i)},\bv_t^{(i)};z_{i_t}^{(i)}) - \nabla_{\bw} f (\bw_t,\bv_t;z_{i_t})\|_2^2 + \|\nabla_{\bv} f (\bw_t,\bv_t;z_{i_t}) - \nabla_{\bv} f (\bw_t^{(i)},\bv_t^{(i)};z_{i_t}^{(i)}) \|_2^2\notag\\
& = \Bigg\|\begin{pmatrix}
\nabla_{\bw} f (\bw_t^{(i)},\bv_t^{(i)};z_{i_t}^{(i)}) - \nabla_{\bw} f (\bw_t,\bv_t;z_{i_t})\\
\nabla_{\bv} f (\bw_t,\bv_t;z_{i_t}) - \nabla_{\bv} f (\bw_t^{(i)},\bv_t^{(i)};z_{i_t}^{(i)})
\end{pmatrix}\Bigg\|_2^2.
\end{align}
It then follows that
\begin{align}
& \ebb\Bigg[\frac{1}{n}\sum_{i=1}^n \Bigg\|\begin{pmatrix}
  g_{\bw}(\bw_t^{(i)},\bv_t^{(i)};z_{i_t}^{(i)}) - g_{\bw}(\bw_t,\bv_t;z_{i_t})\\
  g_{\bv}(\bw_t,\bv_t;z_{i_t}) - g_{\bv}(\bw_t^{(i)},\bv_t^{(i)};z^{(i)}_{i_t})
  \end{pmatrix}\Bigg\|_2^2 \Bigg] \notag\\
& \leq \ebb\Bigg[\frac{1}{n}\sum_{i=1}^n \Bigg\|\begin{pmatrix}
\nabla_{\bw} f (\bw_t^{(i)},\bv_t^{(i)};z_{i_t}^{(i)}) - \nabla_{\bw} f (\bw_t,\bv_t;z_{i_t})\\
\nabla_{\bv} f (\bw_t,\bv_t;z_{i_t}) - \nabla_{\bv} f (\bw_t^{(i)},\bv_t^{(i)};z_{i_t}^{(i)})
\end{pmatrix}\Bigg\|_2^2 \Bigg] \notag\\
& \leq L^2 \ebb\Bigg[\frac{1}{n}\sum_{i=1}^n \Bigg\|\begin{pmatrix}
                \bw_t-\bw_t^{(i)} \\
                \bv_t-\bv_t^{(i)}
            \end{pmatrix}\Bigg\|_2^2\Bigg] + \frac{8\ebb[G_S^2]}{n},
\end{align}
where we use Eq.~\eqref{eq:gradient-bound} for the last step.
The proof is completed.
\end{proof}


The following lemma shows the Lipschitz continuity of the bias of clipping operator.
\begin{lemma}[\label{lem:lipschitz-clip}\citep{zeng2026stochastic}]
For any $\bv_1,\bv_2\in\wcal$, we have
\begin{align}
\|(\clip(\bv_1,\tau)-\clip(\bv_2,\tau))-(\bv_1-\bv_2)\|_2
\leq \|\bv_1-\bv_2\|_2\ibb\big[\|\bv_1\|_2>\tau \text{ or }\|\bv_2\|_2>\tau\big], \label{clip-2}
\end{align}
where $\ibb[\cdot]$ denotes the indicator function, i.e., returning $1$ if the argument holds and $0$ otherwise.
\end{lemma}
We can easily generalize Lemma~\ref{lem:lipschitz-clip} with two variables $\bv_1,\bv_2$ in the following lemma, which shows the Lipschitz continuity of the bias of the clipping operator with four variables $\bv_1,\bv_2,\bv_3,\bv_4$. We omit the details for brevity.
\begin{lemma}\label{lem:lipschitz-clip-2}
For any $\bv_1,\bv_2, \bv_3, \bv_4$, we have
\begin{multline}
\bigg\|\begin{pmatrix}
\clip(\bv_1,\tau_1)-\clip(\bv_2,\tau_1))-(\bv_1-\bv_2)\\
\clip(\bv_3,\tau_2)-\clip(\bv_4,\tau_2))-(\bv_3-\bv_4)
\end{pmatrix}\bigg\|_2 \notag\\
\leq \bigg\|\begin{pmatrix}
\bv_1-\bv_2 \\
\bv_3-\bv_4
\end{pmatrix}\bigg\|_2 \ibb\big[\max\{\|\bv_1\|_2/\tau_1,\|\bv_2\|_2/\tau_1,\|\bv_3\|_2/\tau_2,\|\bv_4\|_2/\tau_2\}>1 \big].
\end{multline}
\end{lemma}

We now prove that the bias of Clipped-SGDA satisfies a generalized Lipschitz-type condition  in Lemma~\ref{lem:bias-clip-st}, which uses Lemma~\ref{lem:lipschitz-clip-2} and Lemma~\ref{lem:clip-ibb}. Lemma~\ref{lem:clip-ibb} helps deal with the indicator function in the proof of Lemma~\ref{lem:bias-clip-st}. Let $S_r$ denote the $r$-th elemtn of $S$.

\begin{lemma}\label{lem:clip-ibb}
Let $\bw,\bw' \in\wcal, \bv,\bv' \in \vcal$ and $r$ follow the uniform distribution over $[n]$. Denote $\theta = (\bw,\bv), \theta'=(\bw',\bv')$. Then
\begin{align*}
\frac{1}{n}\sum_{j:j\neq i} \ibb\big[\max\{\|\nabla_{\bw} f(\theta;\bz_j)\|_2/\tau_1,\|\nabla_{\bw} f(\theta';\bz_j)\|_2/\tau_1,\|\nabla_{\bv} f(\theta';\bz_j)\|_2/\tau_2,\|\nabla_{\bv} f(\theta;\bz_j)\|_2/\tau_2\}>1 \big]\\
\leq
\frac{\ebb_r[\|\nabla_{\bw} f(\theta;S_r)\|_2^2]}{\tau_1^2}
+ \frac{\ebb_r[\|\nabla_{\bw} f(\theta';S_r^{(i)})\|_2^2]}{\tau_1^2}
+ \frac{\ebb_r[\|\nabla_{\bv} f(\theta;S_r)\|_2^2]}{\tau_2^2}
+ \frac{\ebb_r[\|\nabla_{\bv} f(\theta';S_r^{(i)})\|_2^2]}{\tau_2^2}.
\end{align*}
\end{lemma}
\begin{proof}
We know
\begin{align*}
  & \frac{1}{n}\sum_{j:j\neq i} \ibb\big[\max\{\|\nabla_{\bw} f(\theta;\bz_j)\|_2/\tau_1,\|\nabla_{\bw} f(\theta';\bz_j)\|_2/\tau_1,\|\nabla_{\bv} f(\theta';\bz_j)\|_2/\tau_2,\|\nabla_{\bv} f(\theta;\bz_j)\|_2/\tau_2\}>1 \big]\\
  & \leq \ebb_r \big[\ibb\big[\max\{\|\nabla_{\bw} f(\theta;S_r)\|_2/\tau_1,\|\nabla_{\bw} f(\theta';S_r^{(i)})\|_2/\tau_1,\|\nabla_{\bv} f(\theta';S_r^{(i)})\|_2/\tau_2,\|\nabla_{\bv} f(\theta;S_r)\|_2/\tau_2\}>1 \big]\big]\\
  & \leq \mathrm{Pr}\big\{\|\nabla_{\bw} f(\theta;S_r)\|_2>\tau_1\big\} + \mathrm{Pr}\big\{\|\nabla_{\bw} f(\theta';S_r^{(i)})\|_2>\tau_1\big\}
  + \mathrm{Pr}\big\{\|\nabla_{\bv} f(\theta';S_r^{(i)})\|_2>\tau_2\big\}\\
  &+ \mathrm{Pr}\big\{\|\nabla_{\bv} f(\theta;S_r)\|_2>\tau_2\big\}.
  \end{align*}
By Markov's inequality, we know
\begin{align*}
\mathrm{Pr}\big\{\|\nabla_{\bw} f(\theta;S_r)\|_2>\tau_1\big\}
= \mathrm{Pr}\big\{\|\nabla_{\bw} f(\theta;S_r)\|_2^2>\tau_1^2\big\}
\leq \frac{\ebb_r[\|\nabla_{\bw} f(\theta;S_r)\|_2^2]}{\tau_1^2}, 
\end{align*}
and similarly,
\begin{align*}
& \mathrm{Pr}\big\{\|\nabla_{\bw} f(\theta';S_r^{(i)})\|_2>\tau_1\big\} \leq \frac{\ebb_r[\|\nabla_{\bw} f(\theta';S_r^{(i)})\|_2^2]}{\tau_1^2},\\
& \mathrm{Pr}\big\{\|\nabla_{\bv} f(\theta;S_r)\|_2>\tau_2\big\}
\leq \frac{\ebb_r[\|\nabla_{\bv} f(\theta;S_r)\|_2^2]}{\tau_2^2},\\
&\mathrm{Pr}\big\{\|\nabla_{\bv} f(\theta';S_r^{(i)})\|_2>\tau_2\big\} \leq \frac{\ebb_r[\|\nabla_{\bv} f(\theta';S_r^{(i)})\|_2^2]}{\tau_2^2}.
\end{align*}
We complete the proof by combining all these inequalities together.
\end{proof}

\begin{lemma}\label{lem:bias-clip-st}
Let $\bw,\bw' \in\wcal, \bv,\bv' \in \vcal$ and $r$ follow the uniform distribution over $[n]$. Denote $\theta = (\bw,\bv), \theta'=(\bw',\bv')$. Then
  \begin{multline*}
  \Bigg\|\begin{pmatrix}
    \ebb_{r}[\clip(\nabla_{\bw} f(\theta;S_r),\tau_1)]-\nabla_{\bw} F_S(\theta)-\ebb_{r}[\clip(\nabla_{\bw} f(\theta';S^{(i)}_{r}),\tau_1)]+\nabla_{\bw} F_{S^{(i)}}(\theta')\\
    \ebb_{r}[\clip(\nabla_{\bv} f(\theta';S^{(i)}_r),\tau_2)]-\nabla_{\bv} F_{S^{(i)}}(\theta')-\ebb_{r}[\clip(\nabla_{\bv} f(\theta;S_{r}),\tau_2)]+\nabla_{\bv} F_S(\theta)
    \end{pmatrix}\Bigg\|_2 \\
    \leq L \|\theta-\theta'\|_2
    \Bigg(\frac{\ebb_r[\|\nabla_{\bw} f(\theta;S_r)\|_2^2+\|\nabla_{\bw} f(\theta';S_r^{(i)})\|_2^2]}{\tau_1^2}
  + \frac{\ebb_r[\|\nabla_{\bv} f(\theta';S_r^{(i)})\|_2^2+\|\nabla_{\bv} f(\theta;S_r)\|_2^2]}{\tau_2^2}\Bigg)\\
+ \frac{1}{n}\bigg\|\begin{pmatrix}
    \nabla_{\bw} f(\theta;\bz_i)-\nabla_{\bw} f(\theta';\bz_i')\\
    \nabla_{\bv} f(\theta';\bz_i')-\nabla_{\bv} f(\theta;\bz_i)
    \end{pmatrix}\bigg\|_2.
  \end{multline*}
\end{lemma}
\begin{proof}
  Noting that
  \begin{align*}
  \nabla_{\bw} F_S(\theta) = \ebb_r[\nabla_{\bw} f(\theta;S_r)],\qquad \nabla_{\bw} F_{S^{(i)}}(\theta') = \ebb_r[\nabla_{\bw} f(\theta';S^{(i)}_r)], \\
  \nabla_{\bv} F_S(\theta) = \ebb_r[\nabla_{\bv} f(\theta;S_r)], \qquad \nabla_{\bv} F_{S^{(i)}}(\theta') = \ebb_r[\nabla_{\bv} f(\theta';S^{(i)}_r)],
  \end{align*}
  we know (by Jensen's inequality)
  \begin{align}\label{clip-bias-1}
   & \Bigg\|\begin{pmatrix}
    \ebb_{r}[\clip(\nabla_{\bw} f(\theta;S_r),\tau_1)]-\nabla_{\bw} F_S(\theta)-\ebb_{r}[\clip(\nabla_{\bw} f(\theta';S^{(i)}_{r}),\tau_1)]+\nabla_{\bw} F_{S^{(i)}}(\theta')\\
    \ebb_{r}[\clip(\nabla_{\bv} f(\theta';S^{(i)}_r),\tau_2)]-\nabla_{\bv} F_{S^{(i)}}(\theta')-\ebb_{r}[\clip(\nabla_{\bv} f(\theta;S_{r}),\tau_2)]+\nabla_{\bv} F_S(\theta)
    \end{pmatrix}\Bigg\|_2\notag\\
    & \leq \ebb_r \Bigg[\Bigg\|\begin{pmatrix}
    \clip(\nabla_{\bw} f(\theta;S_r),\tau_1)-\nabla_{\bw} f(\theta;S_r)-\clip(\nabla_{\bw} f(\theta';S^{(i)}_{r}),\tau_1)+\nabla_{\bw} f(\theta';S^{(i)}_r)\\
    \clip(\nabla_{\bv} f(\theta';S^{(i)}_r),\tau_2)-\nabla_{\bv} f(\theta';S^{(i)}_r)-\clip(\nabla_{\bv} f(\theta;S_{r}),\tau_2)+\nabla_{\bv} f(\theta;S_r)
    \end{pmatrix}\Bigg\|_2\Bigg]\notag\\
    & \leq \ebb_r \Bigg[\Bigg\|\begin{pmatrix}
    \nabla_{\bw} f(\theta;S_r)-\nabla_{\bw} f(\theta';S^{(i)}_r)\\
    \nabla_{\bv} f(\theta';S^{(i)}_r)-\nabla_{\bv} f(\theta;S_r)
    \end{pmatrix}\Bigg\|_2
   \ibb
    \begin{bmatrix}
    \max\{\|\nabla_{\bw} f(\theta;S_r)\|_2/\tau_1,\|\nabla_{\bw} f(\theta';S_r^{(i)})\|_2/\tau_1,\\
    \|\nabla_{\bv} f(\theta';S_r^{(i)})\|_2/\tau_2,\|\nabla_{\bv} f(\theta;S_r)\|_2/\tau_2\}>1
    \end{bmatrix}
    \Bigg],
  \end{align}
  where we use Lemma \ref{lem:lipschitz-clip-2} for the last step. Since $r$ follows the uniform distribution over $[n]$, we know
  \begin{align*}
  & \ebb_r \Bigg[\Bigg\|\begin{pmatrix}
    \nabla_{\bw} f(\theta;S_r)-\nabla_{\bw} f(\theta';S^{(i)}_r)\\
    \nabla_{\bv} f(\theta';S^{(i)}_r)-\nabla_{\bv} f(\theta;S_r)
    \end{pmatrix}\Bigg\|_2 \ibb
    \begin{bmatrix}
    \max\{\|\nabla_{\bw} f(\theta;S_r)\|_2/\tau_1,\|\nabla_{\bw} f(\theta';S_r^{(i)})\|_2/\tau_1,\\
    \|\nabla_{\bv} f(\theta';S_r^{(i)})\|_2/\tau_2,\|\nabla_{\bv} f(\theta;S_r)\|_2/\tau_2\}>1
    \end{bmatrix}
    \Bigg]\\
  & = \frac{1}{n}\sum_{j:j\neq i}\Bigg\|\begin{pmatrix}
    \nabla_{\bw} f(\theta;\bz_j)-\nabla_{\bw} f(\theta';\bz_j)\\
    \nabla_{\bv} f(\theta';\bz_j)-\nabla_{\bv} f(\theta;\bz_j)
    \end{pmatrix}\Bigg\|_2 \ibb
    \begin{bmatrix}
    \max\{\|\nabla_{\bw} f(\theta;z_j)\|_2/\tau_1,\|\nabla_{\bw} f(\theta';z_j)\|_2/\tau_1,\\
    \|\nabla_{\bv} f(\theta';z_j)\|_2/\tau_2,\|\nabla_{\bv} f(\theta;z_j)\|_2/\tau_2\}>1
    \end{bmatrix}\\
    & + \frac{1}{n}\Bigg\|\begin{pmatrix}
    \nabla_{\bw} f(\theta;\bz_i)-\nabla_{\bw} f(\theta';\bz_i')\\
    \nabla_{\bv} f(\theta';\bz_i')-\nabla_{\bv} f(\theta;\bz_i)
    \end{pmatrix}\Bigg\|_2 \ibb
    \begin{bmatrix}
    \max\{\|\nabla_{\bw} f(\theta;z_i)\|_2/\tau_1,\|\nabla_{\bw} f(\theta';z_i')\|_2/\tau_1,\\
    \|\nabla_{\bv} f(\theta';z_i')\|_2/\tau_2,\|\nabla_{\bv} f(\theta;z_i)\|_2/\tau_2\}>1
    \end{bmatrix}.
  \end{align*}
  By the smoothness of $\theta\mapsto f(\theta;z)$ (Assumption~\ref{ass:smooth}) and Lemma \ref{lem:clip-ibb}, we further get
    \begin{align}\label{clip-bias-2}
  & \ebb_r \Bigg[\Bigg\|\begin{pmatrix}
    \nabla_{\bw} f(\theta;S_r)-\nabla_{\bw} f(\theta';S^{(i)}_r) \notag\\
    \nabla_{\bv} f(\theta';S^{(i)}_r)-\nabla_{\bv} f(\theta;S_r)
    \end{pmatrix}\Bigg\|_2 \ibb
    \begin{bmatrix}
    \max\{\|\nabla_{\bw} f(\theta;S_r)\|_2/\tau_1,\|\nabla_{\bw} f(\theta';S_r^{(i)})\|_2/\tau_1,\\
    \|\nabla_{\bv} f(\theta';S_r^{(i)})\|_2/\tau_2,\|\nabla_{\bv} f(\theta;S_r)\|_2/\tau_2\}>1
    \end{bmatrix}
    \Bigg]\notag\\
  & \leq L \|\theta-\theta'\|_2
    \Bigg(\!\frac{\ebb_r[\|\nabla_{\bw} f(\theta;S_r)\|_2^2\!+\!\|\nabla_{\bw} f(\theta';S_r^{(i)})\|_2^2]}{\tau_1^2}
  + \frac{\ebb_r[\|\nabla_{\bv} f(\theta';S_r^{(i)})\|_2^2\!+\!\|\nabla_{\bv} f(\theta;S_r)\|_2^2]}{\tau_2^2}\!\Bigg)\notag\\
&  + \frac{1}{n}\bigg\|\begin{pmatrix}
    \nabla_{\bw} f(\theta;\bz_i)-\nabla_{\bw} f(\theta';\bz_i')\\
    \nabla_{\bv} f(\theta';\bz_i')-\nabla_{\bv} f(\theta;\bz_i)
    \end{pmatrix}\bigg\|_2,
  \end{align}
  where we simply use the inequality
  \[\ibb \big[
    \max\{\|\nabla_{\bw} f(\theta;z_i)\|_2/\tau_1,\|\nabla_{\bw} f(\theta';z_i')\|_2/\tau_1,
    \|\nabla_{\bv} f(\theta';z_i')\|_2/\tau_2,\|\nabla_{\bv} f(\theta;z_i)\|_2/\tau_2\}>1
    \big] \leq 1.
    \]
  Plugging Eq.~\eqref{clip-bias-2} into Eq.~\eqref{clip-bias-1} completes the proof.
\end{proof}

We now apply the above lemma to show that Eq.~\eqref{simple-condition-c} holds for Clipped-SGDA.

\begin{lemma}\label{lem:condition-b-clip}
Let $g(\bw_t,\bv_t;z_{i_t})$ be defined in Eq.~\eqref{g-clip}. Let Assumptions \ref{ass:smooth}, \ref{ass:moment} hold with $G_S\leq G$ for some $G>0$. Then, Eq. \eqref{simple-condition-c} holds with $\bar{A} = 8L^2G^4(\frac{1}{\tau_1^2}+\frac{1}{\tau_2^2})^2, \bar{B} = \frac{16G^2}{n^2}$.
\end{lemma}

\begin{proof}[Proof of Lemma~\ref{lem:condition-b-clip}]
Denote
$\theta_t=(\bw_t,\bv_t)$ and
$\theta_t^{(i)}=(\bw_t^{(i)},\bv_t^{(i)})$.
From Lemma~\ref{lem:bias-clip-st}, we know
\begin{align}
& \Bigg\|\begin{pmatrix}
        b_{t,\bw} - b_{t,\bw}^{(i)}\\
        b_{t,\bv}^{(i)} - b_{t,\bv}
        \end{pmatrix}\Bigg\|_2 \notag\\
& = \Bigg\|\!\begin{pmatrix}
        \ebb_{i_t}[\clip(\nabla_{\bw} f(\theta_t;z_{i_t}),\tau_1)]
        -\nabla_{\bw}F_S(\theta_t)
        -\ebb_{i_t}[\clip(\nabla_{\bw} f(\theta_t^{(i)};z^{(i)}_{i_t}),\tau_1)]
        +\nabla_{\bw}F_{S^{(i)}}(\theta_t^{(i)})\\
        \ebb_{i_t}[\clip(\nabla_{\bv} f(\theta_t^{(i)};z^{(i)}_{i_t}),\tau_2)]
        -\nabla_{\bv}F_{S^{(i)}}(\theta_t^{(i)})
        -\ebb_{i_t}[\clip(\nabla_{\bv} f(\theta_t;z_{i_t}),\tau_2)]
        +\nabla_{\bv}F_S(\theta_t)
        \end{pmatrix}\Bigg\|_2 \notag\\
& \leq \frac{1}{n}\Bigg\|\begin{pmatrix}
    \nabla_{\bw} f(\theta_t;\bz_i)
    -\nabla_{\bw} f(\theta_t^{(i)};\bz_i')\\
    \nabla_{\bv} f(\theta_t^{(i)};\bz_i')
    -\nabla_{\bv} f(\theta_t;\bz_i)
    \end{pmatrix}\Bigg\|_2
+2LG_S^2\Big(\frac{1}{\tau_1^2}
    +\frac{1}{\tau_2^2}\Big)
    \Bigg\|\begin{pmatrix}
        \bw_t-\bw_t^{(i)}\\
        \bv_t-\bv_t^{(i)}
    \end{pmatrix}\Bigg\|_2,
\end{align}
where we use Assumption~\ref{ass:moment} in the last step.
It then follows from $(a+b)^2\leq2(a^2+b^2)$ that
\begin{align}
&\ebb\bigg[\frac{1}{n}\sum_{i=1}^n
\bigg\|\begin{pmatrix}
        b_{t,\bw} - b_{t,\bw}^{(i)}\\
        b_{t,\bv}^{(i)} - b_{t,\bv}
        \end{pmatrix}\bigg\|_2^2\bigg] \notag\\
& \leq \frac{2}{n^2}\ebb\Bigg[
\frac{1}{n}\sum_{i=1}^n
\Bigg\|\begin{pmatrix}
    \nabla_{\bw} f(\theta_t;\bz_i)
    -\nabla_{\bw} f(\theta_t^{(i)};\bz_i')\\
    \nabla_{\bv} f(\theta_t^{(i)};\bz_i')
    -\nabla_{\bv} f(\theta_t;\bz_i)
    \end{pmatrix}\Bigg\|_2^2\Bigg]
    +8L^2\Big(\frac{1}{\tau_1^2}
    +\frac{1}{\tau_2^2}\Big)^2
    \ebb\Bigg[
    \frac{G_S^4}{n}\sum_{i=1}^n
    \Bigg\|\begin{pmatrix}
        \bw_t-\bw_t^{(i)}\\
        \bv_t-\bv_t^{(i)}
    \end{pmatrix}\Bigg\|_2^2\Bigg] \notag\\
& \leq
    8L^2\Big(\frac{1}{\tau_1^2}
    +\frac{1}{\tau_2^2}\Big)^2
    \ebb\Bigg[
    \frac{G_S^4}{n}\sum_{i=1}^n
    \Bigg\|\begin{pmatrix}
        \bw_t-\bw_t^{(i)}\\
        \bv_t-\bv_t^{(i)}
    \end{pmatrix}\Bigg\|_2^2\Bigg]
    +\frac{16\ebb[G_S^2]}{n^2},
\end{align}
where we use Lemma~\ref{lem:gradient-norm} and
Assumption~\ref{ass:moment} for the last step.
The proof is completed.
\end{proof}

We then apply Lemma~\ref{lem:condition-a-clip} and Lemma~\ref{lem:condition-b-clip} to directly derive Theorem~\ref{cor:clip-stability}. We omit the details for brevity.

The following lemma presents the optimization error bounds for Clipped-SGDA.

\begin{proof}[Proof of Lemma~\ref{lem:opt-clip}]
Let $\bw \in B_R$.
We denote $\hat{g}_{\bw,t}=\clip(\nabla_{\bw} f (\bw_t,\bv_t;z_{i_t}),\tau_1)$.
By the update of Clipped-SGD and note that the projection step is nonexpansive, we know that
\begin{align*}
  \|\bw_{t+1}-\bw\|_2^2 & \leq \big\|\bw_t-\eta_t \hat{g}_{\bw,t}-\bw\big\|_2^2
   = \|\bw_t-\bw\|_2^2+\eta_t^2\|\hat{g}_{\bw,t}\|_2^2+2\eta_t \langle\bw-\bw_t,\hat{g}_{\bw,t}\rangle.
\end{align*}
By $(a+b+c)^2\leq 3a^2+3b^2+3c^2$, we know
\[
\|\hat{g}_{\bw,t}\|_2^2\leq 3\|\hat{g}_{\bw,t}-\ebb_t[\hat{g}_{\bw,t}]\|_2^2+3\|\ebb_t[\hat{g}_{\bw,t}]-\nabla_{\bw} F_S(\bw_t,\bv_t)\|_2^2+3\|\nabla_{\bw} F_S(\bw_t,\bv_t)\|_2^2
\]
and
\begin{multline*}
  \langle\bw-\bw_t,\hat{g}_{\bw,t}\rangle=\langle\bw-\bw_t,\hat{g}_{\bw,t} - \ebb_t[\hat{g}_{\bw,t}]\rangle+\langle\bw-\bw_t,\ebb_t[\hat{g}_{\bw,t}]-\nabla_{\bw} F_S(\bw_t,\bv_t)\rangle\\
  +\langle\bw-\bw_t,\nabla_{\bw} F_S(\bw_t,\bv_t)\rangle.
\end{multline*}
We combine the above three inequalities together
to derive
\begin{multline*}
  \|\bw_{t+1}-\bw\|_2^2 \leq \|\bw_t-\bw\|_2^2+3\eta_t^2\|\hat{g}_{\bw,t}-\ebb_t[\hat{g}_{\bw,t}]\|_2^2+3\eta_t^2\|\ebb_t[\hat{g}_{\bw,t}]-\nabla_{\bw} F_S(\bw_t,\bv_t)\|_2^2 \\
  + 3\eta_t^2 \|\nabla_{\bw} F_S(\bw_t,\bv_t)\|_2^2
  +2\eta_t \langle\bw-\bw_t,\hat{g}_{\bw,t} - \ebb_t[\hat{g}_{\bw,t}]\rangle +2\eta_t\langle\bw-\bw_t,\ebb_t[\hat{g}_{\bw,t}]-\nabla_{\bw} F_S(\bw_t,\bv_t)\rangle\\
  +2\eta_t \langle\bw-\bw_t,\nabla_{\bw} F_S(\bw_t,\bv_t)\rangle.
\end{multline*}
It was shown that~\citep{zhang2020adaptive}
\[
\big\|\ebb_t[\hat{g}_{\bw,t}]-\nabla_{\bw} F_S(\bw_t,\bv_t)\big\|_2\leq G^{2}\tau_1^{-1}.
\]
Then we derive that
\begin{align*}
  & \|\bw_{t+1}-\bw\|_2^2
  \leq \|\bw_t-\bw\|_2^2+3\eta_t^2\|\hat{g}_{\bw,t}-\ebb_t[\hat{g}_{\bw,t}]\|_2^2+3\eta_t^2G^{4}\tau_1^{-2} + 3\eta_t^2 \|\nabla_{\bw} F_S(\bw_t,\bv_t)\|_2^2\\
  & +2\eta_t \langle\bw-\bw_t,\hat{g}_{\bw,t} - \ebb_t[\hat{g}_{\bw,t}]\rangle +2G\tau_1^{-1}\eta_t\|\bw-\bw_t\|_2+2\eta_t \langle\bw-\bw_t,\nabla_{\bw} F_S(\bw_t,\bv_t)\rangle\\
  & \leq \|\bw_t-\bw\|_2^2+3\eta_t^2\|\hat{g}_{\bw,t}-\ebb_t[\hat{g}_{\bw,t}]\|_2^2+3\eta_t^2G^{4}\tau_1^{-2} + 3\eta_t^2 \|\nabla_{\bw} F_S(\bw_t,\bv_t)\|_2^2\\
  & +2\eta_t \langle\bw-\bw_t,\hat{g}_{\bw,t} - \ebb_t[\hat{g}_{\bw,t}]\rangle +2G^2\tau_1^{-1}\eta_t\|\bw-\bw_t\|_2+2\eta_t\big(F_S(\bw,\bv_t)-F_S(\bw_t,\bv_t)\big),
\end{align*}
where we use the convexity of $\bw\mapsto F_S(\bw,\bv)$ for the last step above.
Taking a summation of the above inequality from $t=1$ to $t=T$ ($\bw_1=0$), we derive 
\begin{align*}
  & \|\bw_{T+1}-\bw\|_2^2 + 2\sum_{t=1}^T\eta_t \big(F_S(\bw_t,\bv_t)-F_S(\bw,\bv_t)\big) \\
  & \leq \|\bw_1-\bw\|_2^2+3\sum_{t=1}^T\eta_t^2\|\hat{g}_{\bw,t}-\ebb_t[\hat{g}_{\bw,t}]\|_2^2 +3 G^{4}\tau_1^{-2} \sum_{t=1}^T\eta_t^2
  + 3\sum_{t=1}^T\eta_t^2 \|\nabla_{\bw} F_S(\bw_t,\bv_t)\|_2^2\notag\\
  & +2\sum_{t=1}^T\eta_t \langle\bw-\bw_t,\hat{g}_{\bw,t} - \ebb_t[\hat{g}_{\bw,t}]\rangle +2G^2\tau_1^{-1}\sum_{t=1}^T\eta_t\|\bw_t-\bw\|_2.
\end{align*}
Setting $\eta_t = \eta$, we further get
\begin{align}\label{clip-optimization-0}
  & \|\bw_{T+1}-\bw\|_2^2 + 2\eta \sum_{t=1}^T\big(F_S(\bw_t,\bv_t)-F_S(\bw,\bv_t)\big) \notag\\
  & \leq \|\bw_1-\bw\|_2^2+3\eta^2\sum_{t=1}^T\|\hat{g}_{\bw,t}-\ebb_t[\hat{g}_{\bw,t}]\|_2^2
  +3 G^{4}\tau_1^{-2} \eta^2 T
  + 3\eta^2\sum_{t=1}^T \|\nabla_{\bw} F_S(\bw_t,\bv_t)\|_2^2\notag\\
  & +2\eta \sum_{t=1}^T \langle\bw-\bw_t,\hat{g}_{\bw,t} - \ebb_t[\hat{g}_{\bw,t}]\rangle +2G^2\tau_1^{-1}\eta\sum_{t=1}^T\|\bw_t-\bw\|_2.
\end{align}
It then follows from the concavity of $F_S(\mathbf{w}, \cdot)$ and Schwartz's inequality that
\begin{align*}
  & \|\bw_{T+1}-\bw\|_2^2 + 2\eta \sum_{t=1}^T\big(F_S(\bw_t,\bv_t)-F_S(\bw,\bar{\bv}_T)\big) \\
  & \leq \|\bw_1-\bw\|_2^2+3\eta^2\sum_{t=1}^T\|\hat{g}_{\bw,t}-\ebb_t[\hat{g}_{\bw,t}]\|_2^2+3 G^{4}\tau_1^{-2} \eta^2 T
  + 3\eta^2\sum_{t=1}^T \|\nabla_{\bw} F_S(\bw_t,\bv_t)\|_2^2\\
  & +2\eta \sum_{t=1}^T \langle \bw_t, \ebb_t[\hat{g}_{\bw,t}]-\hat{g}_{\bw,t}\rangle + 2\eta R\Big\|\sum_{t=1}^T \big(\hat{g}_{\bw,t} - \ebb_t[\hat{g}_{\bw,t}]\big)\Big\|_2 +2G^2\tau_1^{-1}\eta\sum_{t=1}^T\|\bw_t-\bw\|_2.
\end{align*}
Define
$
\bw_{S,1} = \arg \min_{\bw \in B_R} F_S(\bw,\bar{\bv}_T),
$
then
\begin{align}
&\|\bw_{T+1}-\bw_{S,1}\|_2^2 + 2\eta \sum_{t=1}^T\big(F_S(\bw_t,\bv_t)-\inf_{\bw \in B_R}F_S(\bw,\bar{\bv}_T)\big)\notag \\
&\leq \|\bw_1-\bw_{S,1}\|_2^2 +3\eta^2\sum_{t=1}^T\|\hat{g}_{\bw,t}-\ebb_t[\hat{g}_{\bw,t}]\|_2^2+3 G^{4}\tau_1^{-2} \eta^2 T
+ 3\eta^2\sum_{t=1}^T \|\nabla_{\bw} F_S(\bw_t,\bv_t)\|_2^2\notag\\
&+2\eta \sum_{t=1}^T \langle \bw_t, \ebb_t[\hat{g}_{\bw,t}] - \hat{g}_{\bw,t}\rangle + 2\eta R\Big\|\sum_{t=1}^T \big(\hat{g}_{\bw,t} - \ebb_t[\hat{g}_{\bw,t}]\big)\Big\|_2 +2G^2\tau_1^{-1}\eta\sum_{t=1}^T\|\bw_t-\bw_{S,1}\|_2.\label{clip-optimization-1}
\end{align}
Note $\ebb\big[\|\nabla_{\bw} F_S(\bw_t,\bv_t)\|_2^2\big] \leq G^2$, $\ebb[\|\hat{g}_{\bw,t}-\ebb_t[\hat{g}_{\bw,t}]\|_2^2]\leq G^2$ and
\begin{align}\label{clip-optimization-exp}
\ebb_t[\langle \bw_t, \ebb_t[\hat{g}_{\bw,t}] - \hat{g}_{\bw,t} \rangle] = 0.
\end{align}
Then taking an expectation to both sides of Eq.~\eqref{clip-optimization-1}, we have
\begin{align}\label{clip-optimization-2}
& \ebb[\|\bw_{T+1}-\bw_{S,1}\|_2^2] + 2\eta \sum_{t=1}^T\ebb\big[F_S(\bw_t,\bv_t)-\inf_{\bw \in B_R}F_S(\bw,\bar{\bv}_T)\big] \notag\\
& \leq \ebb[\|\bw_1-\bw_{S,1}\|_2^2]+3G^2\eta^2T+3 G^{4}\tau_1^{-2} \eta^2 T
+ 3G^2\eta^2 T
+ 2\eta R\ebb\Big[\Big\|\sum_{t=1}^T \big(\hat{g}_{\bw,t} - \ebb_t[\hat{g}_{\bw,t}]\big)\Big\|_2\Big] \notag\\
& +2G^2\tau_1^{-1}\eta\sum_{t=1}^T\ebb[\|\bw_t-\bw_{S,1}\|_2].
\end{align}
According to Jensen's inequality and Eq.~\eqref{clip-optimization-exp}, we know
\begin{align*}
\Big(\ebb_A\Big[\Big\|\sum_{t=1}^T \big(\hat{g}_{\bw,t} - \ebb_t[\hat{g}_{\bw,t}]\big)\Big\|_2\Big]\Big)^2
\leq \ebb_A\Big[\Big\|\sum_{t=1}^T \big(\hat{g}_{\bw,t} - \ebb_t[\hat{g}_{\bw,t}]\big)\Big\|_2^2\Big]
= \sum_{t=1}^T \ebb_A\big[\big\|\hat{g}_{\bw,t} - \ebb_t[\hat{g}_{\bw,t}]\big\|_2^2\big]
\end{align*}
and
\begin{align*}
\ebb\Big[\Big\|\sum_{t=1}^T \big(\hat{g}_{\bw,t} - \ebb_t[\hat{g}_{\bw,t}]\big)\Big\|_2\Big]
& \leq \ebb\big[\big(\sum_{t=1}^T \ebb_A\big[\big\|\hat{g}_{\bw,t} - \ebb_t[\hat{g}_{\bw,t}]\big\|_2^2\big]\big)^{\frac{1}{2}}\big]\\
& \leq \big(\sum_{t=1}^T \ebb[\big\|\hat{g}_{\bw,t} - \ebb_t[\hat{g}_{\bw,t}]\big\|_2^2]\big)^{\frac{1}{2}}
\leq G\sqrt{T}.
\end{align*}
Plugging the above inequality into Eq.~\eqref{clip-optimization-2} and using $\ebb[\|\hat{g}_{\bw,t}-\ebb_t[\hat{g}_{\bw,t}]\|_2^2]\leq G^2$, we have
\begin{multline}\label{clip-optimization-3}
\ebb[\|\bw_{T+1}-\bw_{S,1}\|_2^2] + 2\eta \sum_{t=1}^T\ebb\big[F_S(\bw_t,\bv_t)-\inf_{\bw \in B_R}F_S(\bw,\bar{\bv}_T)\big] \leq \ebb[\|\bw_1-\bw_{S,1}\|_2^2]\\+3\big(2+G^2\tau_1^{-2}\big)G^2\eta^2 T
+ 2 R G\eta\sqrt{T} +2G^2\tau_1^{-1}\eta\sum_{t=1}^T\ebb[\|\bw_t-\bw_{S,1}\|_2].
\end{multline}
Note if $\sum_{t=1}^T\ebb\big[F_S(\bw_t,\bv_t)-\inf_{\bw \in B_R}F_S(\bw,\bar{\bv}_T)\big] \geq 0$, Eq.~\eqref{clip-optimization-2} implies the following inequality for any $t\in[T]$:
\begin{multline*}
\ebb[\|\bw_{t+1}-\bw_{S,1}\|_2^2] \leq \ebb[\|\bw_1-\bw_{S,1}\|_2^2]+3\big(2+\tau_1^{-2}\big)G^2\eta^2 T
+ 2 R G\eta\sqrt{T}\\
+2G^2\tau_1^{-1}\eta\sum_{t=1}^T\ebb[\|\bw_t-\bw_{S,1}\|_2].
\end{multline*}
If we denote $\bar{\triangle}=\big(\max_{t\in[T]}\ebb\big[\|\bw_{t}-\bw_{S,1}\|_2^2\big]\big)^{\frac{1}{2}}$, then it follows that
\[
\bar{\triangle}^2\leq \ebb\big[\|\bw_1-\bw_{S,1}\|_2^2\big]+3\big(2+G^2\tau_1^{-2}\big)G^2\eta^2 T
+ 2 R G\eta\sqrt{T}
+2G^2\tau_1^{-1}\eta T\bar{\triangle}.
\]
Solving this quadratic inequality of $\bar{\triangle}$ implies that
\[
\bar{\triangle}\leq \big(\ebb\big[\|\bw_1-\bw_{S,1}\|_2^2\big]\big)^{\frac{1}{2}}+\sqrt{3}\big(2+G^2\tau_1^{-2}\big)^{\frac{1}{2}} G \eta \sqrt{T}
+ \sqrt{2}R^{\frac{1}{2}}G^{\frac{1}{2}}\eta^{\frac{1}{2}}T^{\frac{1}{4}}
+2G^{2}\tau_1^{-1}\eta T.
\]
We plug this inequality back into Eq.~\eqref{clip-optimization-3} and derive (noting $\tau_1\gtrsim G$)
\begin{multline*}
\eta \sum_{t=1}^T\ebb\big[F_S(\bw_t,\bv_t)-\inf_{\bw \in B_R}F_S(\bw,\bar{\bv}_T)\big] \lesssim \ebb[\|\bw_{S,1}\|_2^2]+G^2\eta^2 T
+ RG\eta\sqrt{T} \\
+ G^2\tau_1^{-1}\eta T \Big(\big(\ebb\big[\|\bw_{S,1}\|_2^2\big]\big)^{\frac{1}{2}}+G\eta \sqrt{T}
+ R^{\frac{1}{2}}G^{\frac{1}{2}}\eta^{\frac{1}{2}}T^{\frac{1}{4}}
+G^2\tau_1^{-1}\eta T\Big).
\end{multline*}
It then follows that
\begin{multline*}
\eta \sum_{t=1}^T\ebb\big[F_S(\bw_t,\bv_t)-\inf_{\bw \in B_R}F_S(\bw,\bar{\bv}_T)\big] \lesssim \ebb[\|\bw_{S,1}\|_2^2]+G^2\eta^2 T
+ R G\eta\sqrt{T} \\
+ G^2\tau^{-1}\eta T\big(\ebb\big[\|\bw_{S,1}\|_2^2\big]\big)^{\frac{1}{2}}+G^3\tau_1^{-1}\eta^2 T^{\frac{3}{2}}
+ G^{\frac{5}{2}}\tau_1^{-1} R^{\frac{1}{2}}\eta^{\frac{3}{2}}T^{\frac{5}{4}}
+G^{4}\tau_1^{-2}\eta^2 T^2.
\end{multline*}
It is clear that
$
2G^2\tau_1^{-1}(\ebb\big[\|\bw_{S,1}\|_2^2\big]\big)^{\frac{1}{2}}\eta T\leq \ebb\big[\|\bw_{S,1}\|_2^2\big]
+G^4\tau_1^{-2}\eta^2 T^2
$
\[
2G^{3}\tau_1^{-1}\eta^2 T^{\frac{3}{2}}
\leq G^2\eta^2 T + G^4\tau_1^{-2}\eta^2 T^2
\]
and
\[
2 G^{\frac{5}{2}}\tau_1^{-1} R^{\frac{1}{2}}\eta^{\frac{3}{2}}T^{\frac{5}{4}}
\leq R G\eta\sqrt{T}
+ G^4\tau_1^{-2}\eta^2 T^2.
\]
Then,
\begin{align*}
\eta \sum_{t=1}^T\ebb\big[F_S(\bw_t,\bv_t)-\inf_{\bw \in B_R}F_S(\bw,\bar{\bv}_T)\big] \lesssim \ebb[\|\bw_{S,1}\|_2^2]+G^2\eta^2 T
+ R G\eta\sqrt{T}
+G^4\tau_1^{-2}\eta^2 T^2,
\end{align*}
which yields
\begin{align*}
\frac{1}{T}\sum_{t=1}^T\ebb\big[F_S(\bw_t,\bv_t)\big]-\ebb\big[\inf_{\bw \in B_R}F_S(\bw,\bar{\bv}_T)\big] \lesssim \frac{R^2}{\eta T}+G^2\eta
+ \frac{G R}{\sqrt{T}}
+G^{4}\tau_1^{-2}\eta T.
\end{align*}
If $\sum_{t=1}^T\ebb\big[F_S(\bw_t,\bv_t)-\inf_{\bw \in B_R}F_S(\bw,\bar{\bv}_T)\big] < 0$, the above inequality also holds.
In a similar way, we can show that
\begin{align*}
\ebb[\sup_{\bv \in B_R}F_S(\bar{\bw}_T,\bv)]-\frac{1}{T}\sum_{t=1}^{T}\ebb[F_S(\bw_t,\bv_t)] \lesssim \frac{R^2}{\eta T}+G^2\eta
+ \frac{G R}{\sqrt{T}}
+G^4\tau_2^{-2}\eta T.
\end{align*}
Combining these two inequalities, we have
\begin{align*}
\ebb\big[\sup_{\bv \in B_R}F_S(\bar{\bw}_T,\bv)-\inf_{\bw \in B_R}F_S(\bw,\bar{\bv}_T)\big] \lesssim \frac{R^2}{\eta T}+G^2\eta
+ \frac{G R}{\sqrt{T}}
+G^4(\tau_1^{-2}+\tau_2^{-2})\eta T.
\end{align*}
The proof is completed by noting $\frac{GR}{\sqrt{T}} \leq \frac{R^2}{\eta T}+G^2 \eta$.
\end{proof}

Combining Theorem~\ref{cor:clip-stability} and Lemma~\ref{lem:opt-clip}, we derive the following theorem.

\begin{proof}[Proof of Theorem~\ref{thm:clip-gen}]
According to Part (b) of Theorem \ref{thm:stab-gen} and Theorem \ref{cor:clip-stability}, we know
\begin{align*}
\triangle^{w,R}(A_{\bw}(S),A_{\bv}(S))-\triangle^{w,R}_S(A_{\bw}(S),A_{\bv}(S))
\lesssim \frac{LG^2\eta^2T^2}{n^2} + \frac{LG^2\eta^2T}{n} + \frac{G^2\eta T}{n} + \frac{G^2 \eta \sqrt{T}}{\sqrt{n}}.
\end{align*}
According to Lemma \ref{lem:opt-clip}, we know
\begin{align*}
\Delta_S^{w,R}(\bar{\mathbf{w}}_T, \bar{\mathbf{v}}_T) \lesssim \frac{R^2}{\eta T}+G^2\eta
+G^4(\tau_1^{-2}+\tau_2^{-2})\eta T.
\end{align*}
The bound then follows directly from Eq.~\eqref{decomp}.
\end{proof}

\end{document}